\documentclass[twoside,11pt]{article}

\usepackage{jmlr2e}

\usepackage{amssymb,amsfonts,amstext,mathrsfs,commath}
\usepackage{latexsym,graphics}
\usepackage{color}

\usepackage{caption}  
\usepackage[font = scriptsize]{caption}
\usepackage{subcaption}
\usepackage{epstopdf}
\usepackage[linesnumbered, boxruled, lined]{algorithm2e}

\usepackage[title,toc,page]{appendix}
\usepackage{cleveref}
\hypersetup{linkcolor =red, citecolor = blue}

\usepackage{bm}

\usepackage{float}

\crefname{figure}{fig.}{figs.}
\Crefname{figure}{Fig.}{Figs.}

\usepackage{bbold}

\allowdisplaybreaks

\numberwithin{equation}{section}
\newcommand{\Z}{\max \left(\thetab_{\max},\deltab_{\max}\right) \max \left(\norm{\thetab}_1,\ \norm{\deltab}_1\right)  }
\newcommand{\Zmin}{\min\{\thetab_{\min}^2,\deltab_{\min}^2\}\min\{\norm{\thetab}^2,\ \norm{\deltab}^2\} }

\newcommand{\limn}{\lim_{n \to \infty}}
\newcommand{\rank}{\text{rank}}
\newcommand{\norml}[1]{\| #1 \|}
\newcommand{\normlarge}[1]{\left\Vert #1\right\Vert}
\newcommand{\normsquare}[1]{\| #1 \|^2}
\newcommand{\normlq}[1]{\| #1 \|_q}
\newcommand{\numleqslant}[1]{\overset{\text{(#1)}}{\leqslant}}
\newcommand{\numequ}[1]{\overset{\text{(#1)}}{=}}
\newcommand{\numgeq}[1]{\overset{\text{(#1)}}{\geqslant}}
\newcommand{\s}{\Psib_{\thetab}\B\Psib_{\deltab}^T}

\newcommand{\err}{\frac{\max \{\thetab_{\max},\deltab_{\max}\} \max \{ \norm{\thetab}_1,\norm{\deltab}_1 \} } {\Zmin}}
\newtheorem{Theorem}{Theorem}
\newtheorem{Proposition}{Proposition}
\newtheorem{Lemma}{Lemma}[section]
\newtheorem{Assumption}{Assumption}
\newcommand\numberthis{\addtocounter{equation}{1}\tag{\theequation}}

\hypersetup{linkcolor =blue}
\DeclareMathOperator*{\argmin}{argmin}

\usepackage{amsfonts} 

\newcommand{\A}{\mathbf{A}}
\newcommand{\B}{\mathbf{B}}
\newcommand{\e}{\mathbf{e}}
\newcommand{\Hb}{\mathbf{H}}
\newcommand{\I}{\mathbf{I}}
\newcommand{\Lb}{\mathbf{L}}
\newcommand{\Ob}{\mathbf{O}}
\newcommand{\R}{\mathbf{R}}
\newcommand{\Rb}{\mathbf{R}}
\newcommand{\Sb}{\mathbf{S}}
\newcommand{\Ub}{\mathbf{U}}
\newcommand{\U}{\mathbf{U}}
\newcommand{\ub}{\mathbf{u}}
\newcommand{\Vb}{\mathbf{V}}
\newcommand{\V}{\mathbf{V}}
\newcommand{\vb}{\mathbf{v}}
\newcommand{\W}{\mathbf{W}}
\newcommand{\X}{\mathbf{X}}
\newcommand{\x}{\mathbf{x}}
\newcommand{\Yb}{\mathbf{Y}}
\newcommand{\Y}{\mathbf{Y}}
\newcommand{\y}{\mathbf{y}}
\newcommand{\M}{\mathbf{M}}
\newcommand{\bv}{\mathbf{v}}
\newcommand{\thetab}{\bm{\theta}}
\newcommand{\deltab}{\bm{\delta}}
\newcommand{\Thetab}{\bm{\Theta}}

\newcommand{\Psib}{\bm{\Psi}}
\newcommand{\Omegab}{\bm{\Omega}}
\newcommand{\Lambdab}{\bm{\Lambda}}

\DeclareMathOperator*{\E}{E}

\usepackage{fancyhdr}
\ShortHeadings{Spectral Algorithms for Community Detection in Directed Networks}{}
\firstpageno{1}

\usepackage{lastpage}
\jmlrheading{21}{2020}{1-\pageref{LastPage}}{8/18; Revised
	3/20}{7/20}{18-581}{Zhe Wang, Yingbin Liang and Pengsheng Ji}
\ShortHeadings{Spectral Algorithms for Community Detection in Directed Networks}{Wang, Liang and Ji}

\begin{document}

\title{Spectral Algorithms for Community Detection in Directed Networks}

\author{\name Zhe Wang \\
       \addr Department of Electrical and Computer Engineering \\
        The Ohio State University \\
        Columbus, OH 43202, USA \\
        \email wang.10982@osu.edu
       \AND
       \name Yingbin Liang  \\
       \addr Department of Electrical and Computer Engineering \\
       The Ohio State University \\ 
       Columbus, OH 43202, USA \\
       \email liang.889@osu.edu
		\AND
		\name Pengsheng Ji (Corresponding author)\\
		\addr Department of Statistics\\ University of Georgia \\ 
		Athens, GA 30602, USA \\
		\email psji@uga.edu
		}
\editor{Francois Caron}

\maketitle

\begin{abstract}%
Community detection in large social networks is affected by degree heterogeneity of nodes.
The D-SCORE algorithm for directed networks was introduced to reduce this effect by taking the element-wise ratios of the singular vectors of the adjacency matrix before clustering. Meaningful results were obtained for the statistician citation network, but rigorous analysis on its performance was missing. First, this paper  establishes theoretical guarantee for this algorithm and its variants for the directed degree-corrected block model (Directed-DCBM). Second, this paper provides significant improvements for the original D-SCORE algorithms  by attaching the nodes outside of the community cores using the information of the original network instead of the singular vectors.
\end{abstract}


\section{Introduction}

Social platforms have become increasingly important in our modern life since they provide fast and easy path to make new friends, maintain relationship and share moments. Due to the highly interactive activities in social platforms (e.g., Facebook, Wechat, Twitter, Line), people have generated a huge amount of data which is highly rich in social information. Various algorithms have been developed to extract useful information from these big social data sets, and community detection or clustering is one of the major tools to uncover the community information from big data.

The basic community detection problem has a simple form: given an $n$-node graph $\mathcal{N = (\mathcal{V}, \mathcal{E})}$ where $\mathcal{V} = \{1,2 \cdots n\}$ is the set of nodes and $\mathcal{E}$ is the set of edges, the goal is to divide $n$ nodes into $K$ disjoint communities. It is believed that nodes within the communities share much more edges than those across communities. In order to formulate the problem more formally and facilitate the design and analysis of algorithms, some network models have been proposed.  As one of the classic models,  the {\em stochastic block model} (SBM) assumes that nodes in the same community have the same statistical edge pattern, i.e., they are stochastically equivalent as pointed out in \cite{Holland1983}. While SBM is useful to capture the community character and easy to analyze, it implies that the distribution of degrees within the community is Poisson, in contrast to the empirical observation that in many natural networks, the degrees follow approximately a power-law distribution \citep{Goldenberg2009}. To overcome this shortcoming, {\em degree corrected} block model (DCBM) was proposed by \cite{Karrer2011}  to characterize the personality of each node  with a heterogeneous parameter. DCBM is more realistic than SBM in terms of the degree distribution, but is usually impossible to fit due to the huge amount of heterogeneous parameters . 

In reality, there exists a lot of \textbf{directed} networks such as citation networks, protein-protein interaction networks, the hyperlink network of websites. Such directed networks are more complex in that there are two types of information involved, namely starting links or receiving links, citing others or being cited, etc, which are not captured by SBM and DCBM. Thus, this paper explores a {\em directed} degree-corrected block model (Directed-DCBM) (see \Cref{model} for more details), which associates different degree parameters with two edge directions for individual nodes in order to model directed networks.

Many community detection algorithms have been proposed in recent years. Among these algorithms, we focus on spectral clustering algorithms for their efficiency and popularity.  In this paper, we provide theoretical analysis of two spectral algorithms for the Directed-DCBM. The first one is D-SCORE algorithm proposed  by \cite{Ji2014} to analyze the statistician citation networks, but no rigorous analysis on its performance was provided. The second one is D-SCORE$_q$ which is a generalization of the row normalization technique. For $q = 2$, it becomes the row normalization technique which is commonly used in spectral clustering algorithms  \citep{Jin2015,Rohe2016} before clustering.

\vspace{-3mm}
\subsection{Contribution}
In theory, this paper provides rigorous analysis of the D-SCORE algorithm for Directed-DCBM. The error bound is in the form of pure heterogeneous parameters, and shows clearly how heterogeneous parameters affect the clustering result and when  consistency can be achieved. This paper also provides unified theoretical analysis of the D-SCORE$_q$ algorithm for the Directed-DCBM. Through the rigorous proof, we show that row normalization for the singular vectors using any $\ell_q$-norm also reduces the effects of heterogeneous parameters and improves the algorithm performance.

The analytical techniques in this paper differ significantly from the previous work in the following aspects. First, the techniques in \cite{Jin2015} for analyzing undirected networks can not be  adapted to directed networks. Instead, we manage to use the Davis-Kahan theorem and take a more direct and general approach, and our techniques are potentially very useful for  general network modeling such multi-layer networks and node-attributed networks. Second,  our way to deal with the asymmetric matrix is different from \cite{Rohe2016}  
who constructed a symmetric matrix by extending the adjacency $\A$ to
$\begin{bmatrix}
0, &\A\\
\A^T, &0
\end{bmatrix}$, whereas we use $\A\A^T$ and $\A^T\A$ that are naturally symmetric matrices and correspond to meaningful networks. Furthermore, our results are directly in the form of heterogeneous parameters which provide explicit insights about the impact of the heterogeneous parameters on the performance of the algorithm, unlike \cite{Rohe2016}. 



Furthermore, we identify possible  issues with the original D-SCORE algorithms for large networks and improve these algorithms using  the intersection-with-attachment technique.  Specifically, we   run the spectral algorithms on the graph core (intersection) and then attach the remaining nodes to the communities, instead of running the spectral algorithms directly on the entire graph as in  \cite{Ji2014}.  The rationale is presented carefully in text and then further demonstrated using  real world data and simulations. 

\vspace{-3mm}
\subsection{Related Work}

We discuss the related work in view of different models as well as algorithms proposed for these models. Due to the extremely intensive studies on community detection, we focus on only algorithms which have theoretical consistency promise and are highly relevant to our study here. There are roughly three kinds of such algorithms that come with theoretical promise, namely the modularity method, spectral clustering, and optimization relaxation.



SBM was introduced by \cite{Holland1983}, and various algorithms have been proposed for solving the community detection problem under SBM. In particular, the modularity method includes profile likelihood modularity \citep{Bickel2009}, Erdos-Renyi modularity \citep{Zhao2012}, etc, and \cite{Zhao2012} provided the consistency proof for these two methods. Spectral clustering mainly has two kinds of methods: spectral clustering with normalized Laplacian matrix \citep{Rohe2011}, and spectral clustering with adjacency matrix \citep{Sussman2012}. In addition, regularization technique has been used to concentrate the eigenvector and improve the algorithm performance, where the details can be found in \cite{Joseph2016}. For the optimization method, objective functions were constructed, which were either inspired by the maximum likelihood estimation or by the insight that there should be more edges inside the community than those outside the community. Solutions to these optimization problems were obtained typically by relaxation, such as SDP relaxation \citep{Amini2014} or convex relaxation ( \citealp{Demaine2003}; \citealp{Chen2012}). It is of general interest to characterize sufficient and necessary conditions that guarantee the consistency of community detection. For example, \cite{Mossel2016,Mossel2013} provided the if and only if conditions for consistent community detection for the case with $K=2$ communities for the planted partition model, which is a special case of SBM. Moreover, \cite{Abbe2015} provided the characterization and new insights for consistent clustering for the case with $K \geqslant 3$.


DCBM was proposed by   \cite{Karrer2011} and various community detection algorithms were studied for DCBM. For modularity methods, \cite{Karrer2011} provided an interpretation of Newman-Girvan modularity method \citep{Newman2004} under DCBM setting and further proposed a profile likelihood modularity method for DCBM. \cite{Zhao2012} provided the consistency proof for these two modularity methods.  Furthermore, \cite{Newman2016} showed that the Newman-Girvan modularity method under DCBM is equivalent with the profile likelihood method in degree-corrected planted partition model with known block parameters.  For spectral clustering methods, \cite{Jing2015} analyzed the performance of spectral clustering and \cite{Gulikers_lelarge_massoulie_2017} proposed a spectral algorithm that does not need the knowledge of the number of communities. In addition, the SCORE algorithm \citep{Jin2015} and the row-normalization technique \citep{Qin2013} were used to alleviate the effect of the heterogeneous parameters. For the optimization methods, \cite{Chen2015} proposed and analyzed a convexified modularity maximization approach under DCBM.


Some {\em directed} network models (where the edges have directions) have been proposed to model directed networks \citep{Wang1987,Reichardt2007,Yang2010} and details can be found in \cite{Malliaros2013}. We mainly focus on directed-DCBM. For such a model, \cite{Ji2014} extended DCBM to directed-DCBM, and adapted the SCORE algorithm designed for DCBM to the D-SCORE algorithm which is applicable for directed-DCBM. \cite{Rohe2016} introduced the stochastic co-block model that combined the idea of DCBM and bi-clustering and developed the spectral co-clustering algorithm called DI-SIM for such a model.

Another important issue of community detection is the estimation of the number $K$ of communities in the graph. Various techniques have been proposed to determine the number of communities in the graph. For example,  \cite{Zhao7321} proposed to extract one community at a time, and then decided whether the reminder of the graph contains multiple communities by comparing the reminder of the graph with the  Erdos-Renyi graph. \cite{Bickel2016} proposed to recursively split the graph into two parts until each part  contains only one community. \cite{chen2017} proposed a network cross-validation approach and \cite{Yang2017}  proposed a likelihood-based method  to determine the number of communities.    More details and other methods can be found in these papers and the references therein.

\section{Network Models} \label{model} 

\subsection{Directed-DCBM} \label{model_notation}

In this section, we introduce the  directed-DCBM. We consider a directed network $\mathcal{N}$, in which there are totally $n$ nodes and we use $\mathcal{V} = \{1, 2, \cdots, n\}$ to denote the set of the indices of these nodes. We assume that the nodes in the network are connected by directional edges. We introduce an $n \times n$ adjacency matrix $\A$ of the network $\mathcal{N}$, and the entries of $\A$ take values either $1$ or $0$. For each entry, $\A({i,j})=1$ if there is a directional edge from node $i$ to node $j$, and $\A({i,j})=0$ otherwise.
 
We assume that the nodes in the network are divided into $K$ disjoint communities, and we use $\mathcal{V}^{(k)} $ for $k=1,\ldots,K$ to represent the set that contains the indices of the nodes in community $k$. Thus,
$ \mathcal{V} =   \mathcal{V}^{(1)} \cup \mathcal{V}^{(2)}  \cdots \cup \mathcal{V}^{(K)}\text{.}$
We let $n_k$ denote the total number of nodes in community $k$, i.e., $n_k =| {\mathcal{V}^{(k)}}|$ for $1 \leqslant k \leqslant K$. Thus, $\sum_{k=1}^K n_k=n$.

We assume that the connectivity behavior of each node is captured by both the common connectivity parameters shared among all nodes in the same community as well as the connectivity parameters of each node. We use a $K\times K$ matrix $\B$ to model the community connectivity behavior. Here, each entry $\B(k,l) $ represents the chance that there exists a directional edge from a node in community $k$ to a node in community $l$, for $k,l=1,\ldots,K$. For each node $i$, we assign two parameters denoted by  $\bm{\thetab}(i)$ and $\deltab(i)$, where $\thetab(i)$ captures how likely node $i$ points edges to other nodes, and $\deltab(i)$ captures how likely node $i$ receives edges from other nodes. Hence, $\thetab(i)$ and $\deltab(i)$ for $i=1,\ldots,n$ model connectivity properties for individual nodes, and are referred to as {\em heterogeneous parameters}.

We model the entries of the adjacency matrix $\A$ as independent Bernoulli random variables, with each entry $\A ({i,j})=1$ having the following probability
\begin{flalign}
P(\A( {i,j}) = 1 ) =  \thetab(i) \B(c_i,c_j) \deltab(j),  \quad \text{for } i,j=1,\ldots,n, \label{bernoulli_random_variable}
\end{flalign}    	
where $c_i$ for $i=1,\ldots,n$ denotes the index of the community that node $i$ belongs to. Note that $\A(i,j)=1$ represents that there exists a directional edge from node $i$ to node $j$. As can be observed from \cref{bernoulli_random_variable}, the probability that there exists such an edge depends on both the community connectivity parameters $\B(c_i,c_j)$ and heterogeneous parameters  $\thetab(i)$ and $\deltab(j)$ of the individual nodes $i$ and $j$.

Since the network contains {\em directional} edges, the directed-DCBM consists of the following three aspects of asymmetry, which distinguishes the directed-DCBM significantly from the typical DCBM. (i) The matrix $\B$ can be asymmetric, i.e., $\B({k,l} )\neq \B({l,k})$, which implies that the connectivity parameter from community $k$ to community $l$ can be different from that from community $l$ to community $k$. (ii) The two heterogeneous parameters for each node can be unequal, i.e., $\thetab(i) \neq \deltab(i)$, which implies that the chance for one node to point edges to other nodes is generally different from that for one node to receive edges from other nodes. (iii) The random adjacency matrix $\A$ is also asymmetric, where $\A(i,j)$ represents the existence of an edge from node $i$ to node $j$, while $\A(j,i)$ represents the existence of an edge from node $j$ to node $i$. And they also take different Bernoulli distribution parameters. As can be seen in \cref{bernoulli_random_variable}, the asymmetries of $\B(c_i,c_j)$  and that of $\thetab(i)$ and $\deltab(i)$ yield asymmetric parameters for $P(\A(i,j) = 1 )$.

Let $\mathbf{\Omega} = \E [\A]$, where $E  [\A]$ is the expectation of the $n \times n$ matrix $\A$. Further let
\begin{flalign}
\mathbf{W} \equiv \A -E  [\A]  =  \A-\mathbf{\Omega}.    \label{W_=A_+_Omega}
\end{flalign}
Note that the entries in matrix $\mathbf{W}$ are independently centered Bernoulli random variables.
\subsection{Notations}


We take the following general notations in this paper. For a vector $\mathbf{v}$ and fixed $q > 0$, $\norml{\bv}_q$ denotes its $\ell_q$-norm. We drop the subscript if $q = 2$. For a matrix $\M$, $\M^T$ denotes the transpose of the matrix $\M$, $\norml{\M}$ denotes the spectral norm, and $\norml{\M}_F$ denotes the Frobenius norm. We let $\norml{\M}_{\min}$ denote the smallest singular value of the matrix $\M$. Let $\sigma_i(\M)$ denote the $i$-th largest singular value of matrix $\M$,  and $\lambda_i(\M)$ denote the $i$-th largest eigenvalue of the matrix $\M$ ordered by the magnitude. In addition, we use  $\M_{\bar{i}}$ to denote the $i$-th row of the matrix $M$ (a bar over the subscript $i$) and $\M({i,j})$ to denote the $(i,j)$th entry of matrix $\M$. For integer $i,j >0$, let $\M_{i \sim j }$ denote the matrix that is formed by  extracting the $i$-th to $j$-th columns of the matrix $\M$.

For two positive sequences $\{a_n\}_{n = 1}^{\infty}$ and $\{b_n\}_{n = 1}^{\infty}$, we say $a_n$ $\asymp b_n$ if there exists a constant $C$ such that $b_n/C  \leqslant a_n \leqslant Cb_n$ for sufficiently large $n$, i.e., $a_n$ and $b_n$ are in the same order. For a set $\mathcal{V}$, $\abs{\mathcal{V}}$ denotes its cardinality.

\subsection{Assumptions}

In this subsection, we describe the assumptions about the matrix $\B$ and the heterogeneous parameters $\thetab (i)$ and $\deltab(i)$ for $i=1,\ldots,n$, which we make throughout this paper. For brevity we drop them in our propositions and lemmas.

\begin{Assumption} \label{Assumption}
	The matrix $\B$ satisfies
	\begin{flalign}
	&0 \leqslant \B(i,j)\leqslant 1  \quad \text{for} \quad  1 \leqslant i,j \leqslant K, \label{eq:4.1}\\
	&\text{$\B\B^T$ and $\B^T\B$ are non-singular, non-negative and irreducible.} \label{eq:4.2}		
	\end{flalign}
\end{Assumption}	
As we observe later, the non-singularity, non-negativity and irreducibility guarantee that the first leading left and right singular vectors (corresponding to the largest singular value) of $B$ are nonzero so that they can ensure the denominator is nonzero in the D-SCORE and D-SCORE$_q$ algorithms.

To describe our assumptions for the heterogeneous parameters, we first define some simplified notations. We collect $\thetab(i)$ for $i=1,\ldots,n$ into a vector denoted by $\thetab$, and collect $\deltab(i)$ for $i=1,\ldots,n$ into a vector denoted by $\deltab$. We define  $n$-dimensional vectors $\thetab^{(k)}$ and $\deltab^{(k)}$ for $1 \leqslant k \leqslant K$ as
\begin{align*}
\thetab^{(k)}(i) = \begin{cases}
\thetab(i)  &\quad \text{if} \quad c_i = k \\
0 &\quad \text{if} \quad c_i \neq k
\end{cases}
\quad \text{ and } \quad
\deltab^{(k)}(i) = \begin{cases}
\deltab(i)  &\quad \text{if} \quad c_i = k \\
0 &\quad \text{if} \quad c_i \neq k
\end{cases},
\end{align*}
where $c_i$ denotes the index of the community that node $i$ belongs to. We further define $\thetab_{\min} \equiv \min_{1 \leqslant i \leqslant n} \thetab(i)$, $\thetab_{\max} \equiv \max_{1 \leqslant i \leqslant n} \thetab(i)$, $\deltab_{\min} \equiv \min_{1 \leqslant i \leqslant n} \deltab(i)$, and $\deltab_{\max} \equiv \max_{1 \leqslant i \leqslant n} \deltab(i)$. We also define the following quantity
\begin{flalign}
Z \equiv \Z, \label{Definition_Z}
\end{flalign}
which appears many times in our analysis.

In this paper, we assume that the heterogeneous parameter vectors $\thetab$ and $\deltab$ can scale with the network size $n$, and hence the asymptotic properties in the following assumptions are all with respect to $n$. For notational simplicity, we do not express these parameters explicitly as a function of $n$.
\begin{Assumption} \label{Assumption_2}
	The heterogeneity parameters $\thetab$ and $\deltab$ satisfy
	\begin{flalign}
	& 0 < \thetab_{\min} \leqslant \thetab_{\max} \leqslant 1, \quad  0 < \deltab_{\min} \leqslant \deltab_{\max} \leqslant 1 ,\label{eq:4.3}\\
	& \norml{\thetab^{(k)}} \asymp \norml{\thetab^{(l)}}, \quad \norml{\deltab^{(k)}} \asymp \norml{\deltab^{(l)}} \quad \text{for } 1 \leqslant k,l \leqslant K ,\label{peq:4.4} \\
	&\lim_{n \to \infty}{\frac{\log(n)Z}{\thetab_{\min}\deltab_{\min}\norm{\thetab}_1\norm{\deltab}_1}}=0.   \label{eq:4.5}
	\end{flalign}
\end{Assumption}
To further explain these assumptions, \cref{peq:4.4} requires that the $\ell_2$-norm of the heterogeneous parameter vectors, i.e., $\norml{\thetab^{(k)}}$, are in the same order across all communities. Intuitively, $\norml{\thetab^{(k)}}$ captures the number of edges that community $k$ points  to other communities in total. Then \cref{peq:4.4} implies that the total number of edges that each community points out are in the same order. To explain \cref{eq:4.5}, $\norm{\thetab}_1$ and $\norm{\deltab}_1$ capture the degrees (i.e., the numbers of edges) that each node respectively receives and points out in total. Then \cref{eq:4.5} essentially requires that the total degree  scales faster than $\log n$.

We next present a few properties that follow directly from \Cref{Assumption_2}. Since $\normsquare{\thetab} = \sum_{k=1}^{K} \normsquare{\thetab^{(k)}}$ and  $\normsquare{\deltab} = \sum_{k=1}^{K} \normsquare{\deltab^{(k)}}$, \cref{peq:4.4} implies
\begin{flalign*}   \label{eq:4.4}
\norml{\thetab^{(i)}} \asymp \norml{\thetab} \text{ and }
\norml{\deltab^{(i)}} \asymp \norml{\deltab} \quad \text{for } 1 \leqslant i , j \leqslant K.  \numberthis
\end{flalign*}
To interpret \cref{eq:4.4},  for all $1 \leqslant i \leqslant K$,  $\norml{\thetab^{(i)}} \asymp \norml{\thetab}$ implies that $\norml{\thetab^{(i)}}$ has the same order as the total degree norm $\norml{\thetab}$. The similar interpretation holds for $\norml{\deltab^{(i)}} \asymp \norml{\deltab} $.

Furthermore, since $\thetab_{\min}\norm{\thetab}_1 \leqslant \norm{\thetab}^2$ and
$\deltab_{\min}\norm{\deltab}_1 \leqslant \norm{\deltab}^2$,  by \cref{eq:4.5} we have
\begin{equation}
\limn \frac{\log(n)Z}{\norml{\thetab}^2 \norml{\deltab}^2} \leqslant \lim_{n \to \infty}{\frac{\log(n)Z}{\thetab_{\min}\deltab_{\min}\norm{\thetab}_1\norm{\deltab}_1}} =  0.
\end{equation}
Since $\frac{\log(n)Z}{\norml{\thetab}^2 \norml{\deltab}^2}  \geqslant 0$ holds for all $n \geqslant 0$, we conclude that
\begin{align}
\limn \frac{\log(n)Z}{\norml{\thetab}^2 \norml{\deltab}^2}  =  0. \label{eq:4.6}
\end{align}

Since the definition of $Z$ suggests that $Z \geqslant \thetab_{\min}\norml{\thetab}_1$ and $Z \geqslant \deltab_{\min}\norml{\deltab}_1$, combining  with \cref{eq:4.5}  we have
\begin{equation}
\limn \frac{\log(n)}{Z}= \limn \frac{\log(n)Z}{Z^2} \leqslant \lim_{n \to \infty}{\frac{\log(n)Z}{\thetab_{\min}\deltab_{\min}\norm{\thetab}_1\norm{\deltab}_1}}  =0.
\end{equation}
Since $ \limn \frac{\log(n)}{Z} \geqslant 0$ holds for all $n > 0$, we conclude that
\begin{align}
\limn \frac{\log(n)}{Z} = 0.  \label{eq:4.7}
\end{align}

\section{Algorithms}
In this section, we describe the two community detection algorithms D-SCORE and D-SCORE$_q$ that we analyze in this paper.  We also provide an improved algorithm, i.e., \Cref{Practical_Algorithm}, which is more suitable to deal with real data.

\begin{algorithm}
	\caption{D-SCORE($\hat{\Ub},\hat{\Vb},K$) } \label{D_SCORE}
	\SetKwInOut{Input}{Input}
	\SetKwInOut{Output}{Output}
	\Input{The number $K$ of communities, the $n \times K$ (unit-norm) leading left and right singular vector matrices of the  adjacency matrix $\A$ denoted by  $\hat{\Ub} = [\hat \Ub_1,  \ldots, \hat \Ub_K]$	and
		$\hat{\Vb} = [\hat \Vb_1,  \ldots, \hat \Vb_K]$.	\\}
	
	Fix a threshold $T_n=\log n$ (used to avoid zero denominator), define the $n \times (K-1)$ ratio matrices $\Rb_{\hat \Ub} $ and $\Rb_{\hat \Vb} $, such that for $1 \leqslant i \leqslant n, 1 \leqslant k \leqslant (K-1)$,
	\begin{equation} \label{ratio_matrix}
	\Rb_{\hat \Ub} (i,k) =  \begin{cases}
	T_n \quad &\text{if} \quad   \frac{\hat{\Ub}_{k+1}(i)}{\hat{\Ub}_1(i)}> T_n \\
	\frac{\hat{\Ub}_{k+1}(i)}{\hat{\Ub}_1(i)} \quad &\text{if}  \quad \abs{\frac{\hat{\Ub}_{k+1}(i)}{\hat{\Ub}_1(i)}} \leqslant T_n \\
	-T_n \quad &\text{if} \quad \frac{\hat{\Ub}_{k+1}(i)}{\hat{\Ub}_1(i)} < -T_n
	\end{cases}  ,
	\Rb_{\hat \Vb} (i,k) =  \begin{cases}
	T_n \quad &\text{if} \quad   \frac{\hat{\Vb}_{k+1}(i)}{\hat{\Vb}_1(i)}> T_n \\
	\frac{\hat{\Vb}_{k+1}(i)}{\hat{\Vb}_1(i)} \quad &\text{if}  \quad \abs{\frac{\hat{\Vb}_{k+1}(i)}{\hat{\Vb}_1(i)}} \leqslant T_n \\
	-T_n \quad &\text{if} \quad \frac{\hat{\Vb}_{k+1}(i)}{\hat{\Vb}_1(i)} < -T_n
	\end{cases}
	\end{equation} \\
	
	Put $\Rb_{\hat \Ub} $ and $\Rb_{\hat \Vb} $ together to form an $n \times (2K - 2)$ ratio matrix $\hat \Rb  $, i.e., $\hat \Rb   = [\Rb_{\hat \Ub} ,\Rb_{\hat \Vb} ]$. Then  run $k$-means on $\hat \Rb$, i.e., find the solution to the following optimization problem:
	\begin{flalign*}
	\M^*=\argmin \limits_{\M \in \mathcal{\M}_{n,2k-2,K}} \norm{\M- \hat \Rb}_F^2 ,	
	\end{flalign*}
	where $\mathcal{M}_{n,2K-2,K}$ denotes the set of $n\times{(2K-2)}$ matrices with only $K$ different rows.
	
	Use $M^*$ to assign membership.	
	
	\Output{The community labels of the nodes.}
\end{algorithm}

\begin{algorithm}
	\caption{{D-SCORE}$_q$($\hat{\Ub},\hat{\Vb},K$) } \label{Row_Normalization}
	\SetKwInOut{Input}{Input}
	\SetKwInOut{Output}{Output}
	\Input{The number $K$ of communities, the $n \times K$ (unit-norm) leading left and right singular vector matrices of the adjacency matrix $\A$ denoted by  $\hat{\Ub} = [\hat \Ub_1,  \ldots, \hat \Ub_K]$	and	$\hat{\Vb} = [\hat \Vb_1,  \ldots, \hat \Vb_K]$.	\\}

	Fix a threshold $T_n=\log n$ (used to avoid zero denominator), define two $n \times K$ ratio matrices $\Rb_{\hat \Ub} $ and $\Rb_{\hat \Vb} $, such that for $1 \leqslant i \leqslant n, 1 \leqslant k \leqslant K$,
	\begin{equation} \label{row_normalization_ratio_matrix}
	\Rb_{\hat \Ub} (i,k) =  \begin{cases}
	T_n \quad &\text{if} \quad   \frac{\hat{\Ub}_{k}(i)}{\normlq{\hat{\Ub}_{\bar{i}}}}> T_n \\
	\frac{\hat{\Ub}_{k}(i)}{\normlq{\hat{\Ub}_{\bar{i}}}} \quad &\text{if}  \quad \abs{\frac{\hat{\Ub}_{k}(i)}{\normlq{\hat{\Ub}_{\bar{i}}}}} \leqslant T_n\\
	-T_n \quad &\text{if} \quad \frac{\hat{\Ub}_{k}(i)}{\normlq{\hat{\Ub}_{\bar{i}}}} < -T_n
	\end{cases} ,
	\Rb_{\hat \Vb} (i,k) =  \begin{cases}
	T_n \quad &\text{if} \quad   \frac{\hat{\Vb}_{k}(i)}{\normlq{\hat{\Vb}_{\bar{i}}}}> T_n \\
	\frac{\hat{\Vb}_{k}(i)}{\normlq{\hat{\Vb}_{\bar{i}}}} \quad &\text{if}  \quad \abs{\frac{\hat{\Vb}_{k}(i)}{\normlq{\hat{\Vb}_{\bar{i}}}}} \leqslant T_n \\
	-T_n \quad &\text{if} \quad \frac{\hat{\Vb}_{k}(i)}{\normlq{\hat{\Vb}_{\bar{i}}}} < -T_n
	\end{cases}
	\end{equation}
	
	Put $\Rb_{\hat \Ub} $ and $\Rb_{\hat \Vb} $ together to form an $n \times 2K$ ratio matrix $\hat \Rb  $, i.e., $\hat \Rb   = [\Rb_{\hat \Ub}, \Rb_{\hat \Vb} ]$. Then run $k$-means on $\hat \Rb$, i.e., find the solution to the following optimization problem:
	\begin{flalign*}
	\M^*=\argmin \limits_{\M \in \mathcal{M}_{n,2K,K}} \norm{\M- \hat \Rb }_F^2,	
	\end{flalign*}
	where $\mathcal{M}_{n,2k,K}$ denotes the set of $n\times 2K$ matrices with $K$ different rows.
	
	Use $\M^*$ to assign membership.	
	
	\Output{The community labels of the nodes.}
\end{algorithm}

D-SCORE (see \Cref{D_SCORE}) was proposed in \cite{Ji2014} for directed-DCBM, as an adapted version of SCORE proposed in \cite{Jin2015} for community detection for DCBM with undirected edges. SCORE is a type of spectral clustering algorithm and can deal with the model with  nodes having heterogeneous parameters to capture their individual connectivity behavior. The central idea of SCORE is to first collect the first $K$ leading eigenvectors of the adjacency matrix into a new matrix, and then divide each row of such a matrix by its first entry. The effect of  heterogeneous parameters can be reduced dramatically, and hence the standard clustering approaches can be applied. SCORE handles network models with undirected edges, but cannot handle networks  with directed edges.

D-SCORE adapts SCORE to network models with {\em directed edges}, where the adjacency matrix is usually {\em asymmetric}. Thus D-SCORE uses the left and right singular vectors for spectral clustering as opposed to SCORE that uses eigenvectors due to the {\em symmetry} of the adjacency matrix.  More specifically, D-SCORE  first collects the first $K$ leading left and right singular vectors into two matrices, and then divides each row of these two matrices by its first entry. In this way, the effect caused by the heterogeneous parameters can also be eliminated. D-SCORE  then combines these two matrices together and applies standard approaches for clustering. D-SCORE was shown to have good empirical performance when it was applied to analyze data of a co-authorship and a citation network for statisticians in  \cite{Ji2014}. However,  the performance guarantee for D-SCORE was not established. In \Cref{performance_guarantee}, we provide such performance analysis.

We then propose an alternative algorithm, i.e., D-SCORE$_q$  (see \Cref{Row_Normalization}), for directed-DCBM, which is an adapted version of the SCORE$_q$ algorithm proposed in \cite{Jin2015} for community detection for DCBM with undirected edges. SCORE$_q$ differs from SCORE in that SCORE$_q$  divides each row of the matrix by the $\ell_q$ norm rather than the first entry of the corresponding row in SCORE to eliminate the effect caused by the  heterogeneous parameters. Note that both SCORE$_q$ and SCORE are designed for networks with {\em undirected edges}. D-SCORE$_q$  differs from D-SCORE in the same way as SCORE$_q$ differs from SCORE, i.e., D-SCORE$_q$ divides each row of the matrix of singular vectors by the $\ell_q$ norm of the corresponding row. Both D-SCORE and D-SCORE$_q$  are designed for networks with {\em directed edges}. In  \Cref{performance_guarantee}, we  provide the performance guarantee for D-SCORE$_q$ for any integer $q > 0$.



\begin{algorithm}
	\caption{Improved {D-SCORE}$_q$$(K,A$) using intersection-with-attachment}  \label{Practical_Algorithm}
	\SetKwInOut{Input}{Input}
	\SetKwInOut{Output}{Output}
	\Input{The number $K$ of communities and the   adjacency matrix $A$.}
	
	Compute the $K$ largest (unit-norm) leading left and right singular vectors of the adjacency matrix $A$ to form two  $n \times K$ singular vector matrices denoted by ${U} = [\ U_1,  \ldots, U_K]$	and  ${V} = [ V_1,  \ldots, V_K]$. Denote the set of the nodes by $S$.	\\
	
	Extract the largest  connected components of  matrices $AA^T$ and $A^TA$, and denote   $S_l$ and $S_r$ respectively as the sets of nodes in the two  connected components.
	
	Select the rows of $U$ and $V$  corresponding to  $S_l \cap S_r$ to form two $|S_l \cap S_r| \times K$ matrices $\hat{U} = [\hat U_1, \hat U_2, \ldots, \hat U_K]$ and  $\hat{V} = [\hat V_1, \hat V_2, \ldots, \hat V_K]$.

	Run D-SCORE($\hat{U}, \hat{V}, K$) or D-SCOREq($\hat{U}, \hat{V}, K$) to assign the community labels to the nodes in $S_l \cap S_r$.
	
	Attach these nodes outside $S_l \cap S_r$, i.e., $i \in S \backslash  (S_l \cap S_r)$, by the following optimization step.
	\begin{flalign}
	c_i = \max_{ c \in {1, \cdots, K} } \sum_{j = 1}^{n} \big( A_{ij} +  A_{ji}  \big)  1_{\{c_j\}}(c),
	\end{flalign}
	where $ 1_{\{c_j\}}(\cdot)$ equals one if $c = c_j$ and equals zero otherwise.
	
	\Output{Community labels of the nodes.}
\end{algorithm}

We further propose an algorithm based on the intersection graph with attachment (see \Cref{Practical_Algorithm}) to improve the performance of D-SCORE and D-SCORE$_q$. In order for D-SCORE and D-SCORE$_q$ to perform well, it requires that the weighted graphs defined by $\A^T\A$ and $\A\A^T$  are both connected.  
This connectivity requirement on $\A^T\A$ and $\A\A^T$
can be violated  in real data with large networks.  When this happens to either matrix,  its leading eigenvector is  0 in theory for all nodes outside of the giant component, but the extremely small numbers (computational errors for 0)  appear as the denominators for  D-SCORE$_q$ and D-SCORE, causing misclustering errors on these nodes.

To fix this issue,  \Cref{Practical_Algorithm} is introduced  to first extract the intersection of the sets of the nodes respectively corresponding to the largest  connected components of $\A^T\A$ and $\A\A^T$ (see step $2$ in \Cref{Practical_Algorithm}). Such an intersection set can be interpreted as the core of the graph. And then we apply D-SCORE$_q$ or D-SCORE over this intersection set (see steps $3$ and $4$ in \Cref{Practical_Algorithm}) to assign community labels to nodes in the intersection set. We then assign each node outside the intersection set to the community, to which the node has the most edge connections (including received and pointed out edges). This step, i.e., step $5$ in \Cref{Practical_Algorithm}, is referred to as the attachment step. As demonstrated by our experiments in \Cref{sec:experiments} and \Cref{experiments_on_Synthetic_data}, the experiments show that  the {\em intersection-with-attachment} technique can greatly improve performance of all the original D-SCORE algorithms. 

The intuition behind \Cref{Practical_Algorithm} is that nodes outside the intersection set is kind of noise nodes with less information since they do not have a strong connection with the graph, we extract the core of the graph by ignoring the noise nodes, and then attach them with the core graph. This observation can be seen clearly in \cref{political_blog_1_WG,political_blog_1_INT,political_blog_2_WG,political_blog_2_INT},   nodes in the intersection (the core) have a clear community structure while   nodes outside the intersection is kind of mingling with each other. Ignoring   noise nodes in the first step gives a clear picture for the underlying community structure, and thus improves the performance of proposed algorithms.   

Furthermore, for the robustness consideration, we can replace the $k$-means step in D-SCORE and D-SCORE$_q$ with $k$-medoids \citep{PARK20093336} or other approaches for clustering, which are more robust to outliers.

\section{Main Results}  \label{performance_guarantee}
In this section, we establish the performance guarantee for D-SCORE and  D-SCORE$_q$ in \Cref{performance_guarantee_for_DSCORE} and \Cref{performance_guarantee_for_DSCOREq}, respectively.
\subsection{Performance Guarantee for D-SCORE}  \label{performance_guarantee_for_DSCORE}


As a road map to prove the performance guarantee for D-SCORE, we first analyze the property of the matrix that consists of singular vectors of the expected adjacency matrix $\Omegab$ in \Cref{lemma5.1}, and then bound the distance between this matrix and its random version that consists of the singular vectors of the random adjacency matrix $\A$ in \Cref{distance between singular vector V}. Furthermore, we  prove that the ratio matrix generated by the expected adjacency matrix $\Omegab$ has a desired property for spectral clustering in \Cref{k mean gap}, and then bound the distance between such a ratio matrix and its random version generated by the random adjacency matrix $\A$ in \Cref{R gap}.  After that we bound the distance between $M^*$ and the ratio matrix generated by the singular vectors of  the expected adjacency matrix $\Omegab$ in \Cref{M* R gap}. Combining all these five propositions together, we establish our main result in \Cref{mis clustering nodes}. All the proofs are provided in \Cref{app:d-score}.

First, we analyze the singular vector matrix of the expected matrix $\mathbf{\Omega}$ of the random adjacency matrix $\A$, which captures the key information for clustering. We also anticipate that the property of $\mathbf{\Omega}$ should well approximate that of $\A$, which we study next. We first define $\Sb \equiv \bm{\Psi}_{\thetab}\B\bm{\Psi}_{\deltab}^T$, where the matrix $\B$ captures the connectivity parameters among communities (see \cref{bernoulli_random_variable}), and $\Psib_{\thetab}$, $\Psib_{\deltab}$ are the $K \times K$ diagonal matrices such that for $ 1 \leqslant i \leqslant K$,
\begin{flalign} \label{definition_diagonal_Psi}
\Psi_{\thetab}(i,i) = \frac{\norml{\thetab^{(i)}}}{\norml{\thetab}} \quad \text{and} \quad \Psi_{\deltab}(i,i) = \frac{\norml{\deltab^{(i)}}}{\norml{\thetab}}.
\end{flalign} 
Hence, $\Psib_{\thetab}$, $\Psib_{\deltab}$ capture the total heterogeneity of each community. 

The following proposition provides the singular vector decomposition of $\bm{\Omega}$.
\begin{Proposition}                    \label{lemma5.1}
	Let $\bm{\Omega}=\Ub\bm{\Lambda} \Vb^T$ denote the compact singular value decomposition of $\bm{\Omega}$. Then, the singular values of $\bm{\Omega}$ are given by
	\begin{flalign}
	\sigma_i(\Omegab) &= \begin{cases}
	\norml{\thetab}\norml{\deltab} \sigma_i (\Sb) &\quad\text{if } 1 \leqslant i \leqslant K, \\
	0 &\quad\text{if } i > K,
	\end{cases} \label{eq:5.1}
	\end{flalign}
	where $ \Sb \equiv \bm{\Psi}_{\thetab}\B\bm{\Psi}_{\deltab}^T$.  Let $\Sb=\Yb \bm{\Lambda_s}\Hb^T$ denote the singular value decomposition of $\Sb$. The singular vectors of $\Omegab$ in row's form are given by
	\begin{flalign}
	\Vb_{\bar{i}}=\frac {\deltab (i)}{\norm {\deltab^{(c_i)}}} \Hb_{\bar{c_i}}    \quad \text{and }   \quad
	\Ub_{\bar{i}}=\frac {\thetab (i)}{\norm {\thetab^{(c_i)}}} \Yb_{\bar{c_i}}     \quad \text{for }    1 \leqslant i \leqslant n, \label{row represent U}
	\end{flalign}
	and in column's form are given by
	\begin{flalign}
	\Vb_i&= \sum_{k=1}^{K} \frac{\deltab^{(k)}}{\norml{\deltab^{(k)}}}\Hb_i(k)  \quad \text{for }    1 \leqslant i \leqslant K  ,     \label{eq:5.2}\\
	\Ub_i&= \sum_{k=1}^{K} \frac{\thetab^{(k)}}{\norml{\thetab^{(k)}}}\Yb_i(k)    \quad \text{for }    1 \leqslant i \leqslant K.         \label{eq:5.3}
	\end{flalign}
	Furthermore,
	\begin{flalign*} \label{bound_VU}
	\norml{\Vb_{\bar{i}}} \asymp {\frac{\deltab(i)}{\norml{\deltab}}} \quad  \text{and} \quad
	\norml{\Ub_{\bar{i}}} \asymp \frac{\thetab(i)}{\norml{\thetab}}, \quad \text{for } 1 \leqslant i \leqslant n  . \numberthis
	\end{flalign*} 	
\end{Proposition}
\begin{proof}
	The proof can be found in \Cref{proof_of_proposition_1}.
\end{proof}

We note that  \cref{eq:5.1} implies that $\Omegab$ has only $K$ non-zero singular values due to the fact that there are in total $K$ disjoint communities. Thus, the {\em compact} singular value decomposition of $\Omegab$ is written in the form of an $n \times K$ left singular matrix $\Ub$,  an $n \times K$ right singular matrix $\Vb$, and a $K \times K$ diagonal matrix $\bm{\Lambda}$.

To further explain the result of \Cref{lemma5.1}, consider nodes $i,j$  and suppose they are in the same community, i.e., $c_i = c_j = k$. Then by \cref{row represent U}, the corresponding rows of nodes $i$  and $j$ in the matrix $\Vb$ are given by  $\Vb_{\bar{i}}=\frac {\deltab (i)}{\norm {\deltab^{(k)}}} \Hb_{\bar{k}}$ and $\Vb_{\bar{j}}=\frac {\deltab (j)}{\norm {\deltab^{(k)}}} \Hb_{\bar{k}}$, respectively.  These two row vectors differ only by the  individual node parameters $\deltab(i)$ and $\deltab(j)$. In fact, the step $(3.1)$ in the \Cref{D_SCORE} exactly  eliminates these heterogeneous parameters to make the corresponding vectors become the same if nodes are in the same community. On the other hand, if nodes $i,j$ are in the different communities, i.e., $c_i \neq c_j$, their corresponding row vectors $\Vb_{\bar{i}}$ and $\Vb_{\bar{j}}$ are very different. The same argument is applicable to the row vectors in the  left singular vector matrix $\Ub$. This observation intuitively justifies why the singular vector matrices can be used for recovering the community labels of the nodes.

Next, we bound the distance between the singular vectors  of the random adjacency matrix $\A$ and those of $\Omegab$.  The central idea of the proof is the proper application of Davis-Kahan inequality.
\begin{Proposition}                     \label{distance between singular vector V}
	Let	the first K leading left and right singular vectors of $\A$ be denoted by $\hat{\Vb}_1 \cdots \hat{\Vb}_K$ and $\hat{\Ub}_1 \cdots \hat{\Ub}_K$, and the first K leading left and right singular vectors of $\Omegab$ be denoted by $\Vb_1 \cdots \Vb_K$ and $\Ub_1 \cdots \Ub_K$. Then there exist two constants $C_V$ and $C_U$ with absolute value $1$ and two orthogonal $(K-1)\times(K-1)$ matrices $\Ob_\Vb$ and $\Ob_\Ub$, such that for $n$ large enough, with probability at least $1 - O(n^{-4})$, the following bounds hold
	\begin{flalign*}
	\norml{\hat{\Vb_1} -  \Vb_1C_V}_F \leqslant C \frac{\sqrt{\log (n) Z}}{\norml{\thetab} \norml{\deltab}} , \qquad &
	\norml{\hat{\Vb}_{2 \sim K} - \Vb_{2 \sim K}\Ob_\Vb}_F \leqslant C \frac{\sqrt{\log (n) Z}}{\norml{\thetab} \norml{\deltab}} ,\\
	\norml{\hat{\Ub}_1 -  \Ub_1C_U}_F \leqslant C \frac{\sqrt{\log (n) Z}}{\norml{\thetab} \norml{\deltab}} , \qquad &
	\norml{\hat{\Ub}_{2 \sim K} - \Ub_{2 \sim K}\Ob_\Ub}_F \leqslant C \frac{\sqrt{\log (n) Z}}{\norml{\thetab} \norml{\deltab}}, 
	\end{flalign*}
	where $Z$ is defined in \cref{Definition_Z}.
\end{Proposition}
\begin{proof}
	The proof can be found in \Cref{proof_of_proposition_2}.
\end{proof}
With \Cref{distance between singular vector V}, we are ready to explain further the idea of eliminating the effect caused by heterogeneous parameters from the singular vectors in \Cref{D_SCORE}. {\em The central idea is to divide each row of the singular vector matrix by its first entry.} To this end, for $i = 1, \cdots n$, we define ratio matrices $\Rb_\Vb$ and $\Rb_\Ub$ as
\begin{flalign} \label{definition_R_V_U}
(\Rb_\Vb)_{\bar{i}} = \frac{ (\Vb_{2 \sim K}\Ob_\Vb)_{\bar{i}}}{C_V\Vb_1(i)} \quad \text{and} \quad (\Rb_\Ub)_{\bar{i}} = \frac{ (\Ub_{2 \sim K}\Ob_\Ub)_{\bar{i}}}{C_U\Ub_1(i)}.
\end{flalign}
Namely we divide each row of the matrix $\Vb$ by its first entry and then collect the $2$nd to $K$th columns to form the ratio matrix $\Rb_\Vb$. The matrix $\Rb_\Ub$ is similar.
Note that
\begin{flalign*} \label{meaning_of_DSCORE}
(\Rb_\Vb)_{\bar{i}} &= \frac{(\Vb_{2 \sim K}O_V)_{\bar{i}}}{C_V \Vb_1(i)} = \frac{(\Vb_{2 \sim K})_{\bar{i}}\Ob_\Vb}{C_V \Vb_1(i)} \numequ{i}   \frac{\frac {\deltab (i)}{\norm {\deltab^{(c_i)}}} (\Hb_{2 \sim K})_{\bar{c_i}}\Ob_\Vb}{\frac {\deltab (i)}{\norm {\deltab^{(c_i)}}} C_V \Hb_{1}(c_i)} = \frac{( \Hb_{2 \sim K}O_V)_{\bar{c_i}}}{C_V \Hb_{1}(c_i)}, \numberthis
\end{flalign*}
where (i) follows from \cref{row represent U}.

Comparing \cref{meaning_of_DSCORE} with ${\Vb}_{\bar{i}} =\frac {\deltab (i)}{\norm {\deltab^{(c_i)}}} \Hb_{\bar{c_i}}$ in \cref{row represent U}, we observe that the ratio matrix $\Rb_{\Vb}$ in \cref{meaning_of_DSCORE} does not contain the heterogeneous parameters, and the corresponding row  of each node $i$ in $\Rb_{\Vb}$, i.e., $(\Rb_{\Vb})_{\bar{i}}$, is determined only by $c_i$, which denotes the community that node $i$ belongs to. This implies that if the nodes are in the same community, then their corresponding  rows in $\Rb_\Vb$ are the same. The same argument is also applicable to the ratio matrix $\Rb_{\Ub}$. This explains the importance of the ratio step in \Cref{D_SCORE}. Our next result formally legitimates the ratio matrix $\Rb \equiv [\Rb_\Vb, \Rb_\Ub]$ for clustering.

\begin{Proposition}               \label{k mean gap}
	For the ratio matrix  $\Rb = [\Rb_\Vb, \Rb_\Ub]$ generated by the singular vectors of the matrix $\Omegab$, and for  $ 1 \leqslant i \leqslant n$ and  $1 \leqslant j \leqslant n$, the following inequalities hold:
	\begin{flalign*}
	\norml{\Rb_{\bar{i}} - \Rb_{\bar{j}} } &\geqslant 2 \quad \text{if} \quad c_i \neq c_j, \quad \text{and} \quad
	\norml{\Rb_{\bar{i}} - \Rb_{\bar{j}} } = 0 \quad \text{if} \quad  c_i =c_j.
	\end{flalign*}
\end{Proposition}
\begin{proof}
	The proof can be found in \Cref{proof_of_proposition_3}.
\end{proof}
\Cref{k mean gap}  states that  if the nodes are in the same community, then their corresponding rows in $\Rb \equiv [\Rb_\Vb, \Rb_\Ub]$ are the same. Otherwise if they are in different communities, their corresponding rows are sufficiently different. \Cref{k mean gap} also implies that the ratio matrix $\Rb$ has exactly $K$ different rows due to the fact that there are only $K$ communities in the graph. Thus, the ratio matrix $\Rb$ has the desirable  properties for spectral clustering.

We then generate another ratio matrix $\hat{\Rb}=[\Rb_{\hat{\Vb}}, \Rb_{\hat{\Ub}}]$, where $\Rb_{\hat{\Vb}}$ and $\Rb_{\hat{\Ub}}$ are generated from  $\hat{\Vb}$ and $\hat{\Ub}$ in the way similar to the generation of $\Rb_\Vb$ and $\Rb_\Ub$ from $\Vb$ and $\Ub$. The exact  definitions of $\Rb_{\hat{\Vb}}$ and $\Rb_{\hat{\Ub}}$ are in \cref{ratio_matrix}. Note that, $\hat{\Rb}$ is the ratio matrix generated from the random adjacency matrix $\A$, whereas $\Rb$ is the ratio matrix generated from the expected matrix of $\A$, i.e., the $\Omegab$.

To bound the distance between the ratio matrices $\Rb$ and $\hat{\Rb}$, define a quantity $err_n$,
\begin{flalign} \label{def_err_n}
err_n \equiv	\err ,
\end{flalign}
which characterizes the effect of heterogeneous parameters on the difference between $\Rb$ and $\hat{\Rb}$ as shown in \Cref{R gap}.
\begin{Proposition}                                                 \label{R gap}
	For $\Rb = [\Rb_{\Vb},\Rb_{\Ub}], \hat{\Rb}=[\Rb_{\hat{\Vb}}, \Rb_{\hat{\Ub}}]$, and n large enough, with probability at least $1 - O(n^{-4})$,  we have
	\begin{flalign}
	\norml{\hat{\Rb}  - \Rb}_F^2 \leqslant    C T_n^2 \log (n) err_n.
	\end{flalign}	
\end{Proposition}
\begin{proof}
	The proof can be found in \Cref{proof_of_proposition_4}.
\end{proof}
We then analyze the matrix $\M^*$ which is defined  as the output matrix of step $2$  in \Cref{D_SCORE}. In fact, $\M^*$  is the  matrix with exactly $K$ different rows and nearest to the ratio matrix $\hat{\Rb}$ in term of Frobenius norm. In the following proposition, we bound the distance of $\M^*$ and the ratio matrix $\Rb$, so that the properties of $\Rb$ in \Cref{k mean gap} can serve as a good approximation of the properties of $\M^*$. The proof of \Cref{M* R gap} is based on \Cref{R gap} and the definition of $\M^*$.
\begin{Proposition}     \label{M* R gap}
	For $n$ large enough, with probability at least $1 - O(n^{-4})$, we have
	\begin{flalign*}
	\norml{\M^{*} - \Rb}_F^2 \leqslant T_n^2 \log(n) err_n.
	\end{flalign*}
\end{Proposition}
\begin{proof}
	The proof can be found in  \Cref{prove_of_lemma_3_1}.
\end{proof}
In order to present our main theorem for the  D-SCORE algorithm, we first define the following notation for convenience. Let  $\mathcal{V}$ denote the set of all the nodes in the graph and let $\mathcal{W}$ be the set of nodes that are correctly clustered by the D-SCORE algorithm. Then by definition, $\mathcal{V} \backslash \mathcal{W}$ is the set of incorrectly clustered nodes, i.e., the nodes which are misclustered by the algorithm. Recall that $n_i$ denotes the number of nodes in community $i$, for $i = 1, \cdots, K$.  The following theorem establishes the bound on the number of  misclustered notes for D-SCORE.

\begin{Theorem}[Convergence of D-SCORE]   \label{mis clustering nodes}
	Consider directed-DCBM, for which  \Cref{Assumption} and \Cref{Assumption_2} hold.  Suppose $ |\mathcal{V} \backslash \mathcal{W}|  < {\min{\{n_1, n_2 \cdots n_K }\}} $.	Let
	$\mathcal{W} \equiv \{1 \leqslant i \leqslant n:\norml{M_{\bar{i}}^* - R_{\bar{i}}}  \leqslant \frac{1}{2} \}.$ Then nodes in the set $\mathcal{W}$ are correctly clustered by  the D-SCORE algorithm. Furthermore,  for n large enough, with probability at least $1 - o(n^{-4})$,
	\begin{flalign}
	|\mathcal{V} \backslash \mathcal{W}| \leqslant CT_n^2\log(n) err_n. 	
	\end{flalign}
\end{Theorem}
\begin{proof}
	The proof can be found in \Cref{prove_of_theorem_1}.
\end{proof}

We note that the assumption $ |\mathcal{V} \backslash \mathcal{W}|  < {\min{\{n_1, n_2 \cdots n_K }\}} $ in \Cref{mis clustering nodes} guarantees that D-SCORE clusters at least one node in each community correctly. A Similar assumption was also made in \cite{Jin2015} to show the performance guarantee for SCORE algorithm.

To further understand \Cref{mis clustering nodes}, we consider a simple situation, in which the heterogeneous parameters $\thetab$ and $\deltab$ are bounded by constants, i.e., $0 < \alpha \leqslant \thetab, \deltab \leqslant \beta \leqslant 1$. (Note that the special case of the stochastic block model \cite{Holland1983} has $\thetab$ and $\deltab $ to be constant.) In such a case, $err_n \leqslant  \frac{\beta^2}{\alpha^4} $, i.e., it is bounded by a constant. Hence, the error bound of  \Cref{mis clustering nodes}  is  in the order of $O(T_n^2\log(n) )$. Typically, we take $T_n = \log(n)$, and then the misclustering rate satisfies
\begin{flalign*}
\limn \frac{	|\mathcal{V} \backslash \mathcal{W}|}{n} \leqslant \limn \frac{ C\log^3(n)}{n} = 0.
\end{flalign*}

\subsection{Performance Guarantee for D-SCOREq}\label{performance_guarantee_for_DSCOREq}

The general idea of the analysis of D-SCORE$_q$ is similar to that of D-SCORE with some technical differences. Hence, here we directly present the main theorem for D-SCORE$_q$ below and relegate the technical proof to Appendix \ref{app:d-scoreq}.

With a little abuse of notations, we reuse $\Rb, \Rb_\Vb, \Rb_\Ub$ and $\hat{\Rb}, \Rb_{\hat{\Vb}}, \Rb_{\hat{\Ub}}$ for D-SCORE$_q$, which have slightly different meaning as those for D-SCORE  as we explain below. The matrices  $\Rb_{\hat{\Vb}}$ and $ \Rb_{\hat{\Ub}}$ are defined in \cref{row_normalization_ratio_matrix}, and $\Rb_\Vb$ and $\Rb_\Ub$ are defined as
\begin{flalign} \label{expection_matrix_row_ratio}
(\Rb_\Vb)_{\bar{i}} = \frac{ (\Vb\Ob_\Vb)_{\bar{i}}}{\normlq{(\Vb\Ob_\Vb)_{\bar{i}}}} \quad \text{and} \quad (\Rb_\Ub)_{\bar{i}} = \frac{ (\Ub\Ob_\Ub)_{\bar{i}}}{\normlq{(\Ub\Ob_\Ub)_{\bar{i}}}} ,
\end{flalign} for $i = 1, \cdots n$. Thus, we have
\begin{flalign*} \label{meaning_of_Row_Normalization}
(\Rb_\Vb)_{\bar{i}} &= \frac{ (\Vb\Ob_\Vb)_{\bar{i}}}{\normlq{(\Vb\Ob_\Vb)_{\bar{i}}}}  =  \frac{\Vb_{\bar{i}}\Ob_\Vb}{\normlq{\Vb_{\bar{i}}\Ob_\Vb}} \numequ{i}   \frac{\frac {\deltab (i)}{\norm {\deltab^{(c_i)}}} \Hb_{\bar{c_i}}\Ob_\Vb}{\frac {\deltab (i)}{\norm {\deltab^{(c_i)}}} \normlq{\Hb_{\bar{c_i}}\Ob_\Vb}} = \frac{ \Hb_{\bar{c_i}}\Ob_\Vb}{\normlq{\Hb_{\bar{c_i}}\Ob_\Vb}}, \numberthis
\end{flalign*}
where (i) follows from \cref{row represent U}.

Comparing \cref{meaning_of_Row_Normalization} with $\Vb_{\bar{i}}=\frac {\deltab (i)}{\norm {\deltab^{(c_i)}}} \Hb_{\bar{c_i}}$ in \cref{row represent U}, we observe that the ratio matrix $\Rb_\Vb$ in \cref{meaning_of_Row_Normalization} does not contain  factor $\frac {\deltab (i)}{\norm {\deltab^{(c_i)}}}$ of the heterogeneous parameters, and the corresponding row of each node $i$ in $\Rb_\Vb$, i.e., $(\Rb_{\Vb})_{\bar{i}}$ ,  is determined only by $c_i$, which denotes the community that node $i$ belongs to. This implies that if these nodes are in the same community, and then their corresponding  rows in $\Rb_\Vb$ are the same. The same argument is also applicable to the  matrix $\Rb_{\Ub}$. This explains the importance of the ratio step in \Cref{Row_Normalization} and also explains why D-SCORE$q$ is as powerful as D-SCORE.


We are now ready to present the main theorem for the D-SCORE$_q$ algorithm as follows.
\begin{Theorem}[Convergence of D-SCORE$_q$]   \label{rown_mis clustering nodes}
	Consider the  directed-DCBM under  \Cref{Assumption} and \Cref{Assumption_2}. Suppose   $ |\mathcal{V} \backslash \mathcal{W}|  < {\min{\{n_1, n_2 \cdots n_K }\}} $. Let
	$\mathcal{W} \equiv \{1 \leqslant i \leqslant n:\norml{\M_{\bar{i}}^* - \R_{\bar{i}}}  \leqslant \frac{C}{2} \}.$ Then there exists a constant $C$, such that nodes in the set $\mathcal{W}$ are correctly clustered by  the D-SCORE$_q$ algorithm. Furthermore,  for n large enough, with probability at least $1 - o(n^{-4})$,
	\begin{flalign}
	|\mathcal{V} \backslash \mathcal{W}| \leqslant CT_n^2\log(n) err_n. 	
	\end{flalign}
\end{Theorem}
\begin{proof}
See Appendix \ref{app:d-scoreq}.
\end{proof}

\section{Experiments}\label{sec:experiments}
In this section, we conduct experimental studies to compare the performance of six spectral clustering algorithms, namely, D-SCORE, D-SCORE$_q$, rD-SCORE, rD-SCORE$_q$, oPCA, rPCA, and two likelihood algorithms APL \citep{amini2013} and BCPL \citep{Bickel21068}. We compare these eight algorithms on the web blogs data and the experiments on simulated data.

\subsection{Algorithms}\label{sec:sixalg}
Among the algorithms that we compare in the experiments, D-SCORE and D-SCORE$_q$ correspond  to \Cref{D_SCORE} and \Cref{Row_Normalization} in this paper. The algorithm oPCA (see  \Cref{oPCA}) is the original spectral clustering method, which collects the singular vectors of the adjacency matrix into one matrix and runs $K$-means on such a matrix. Furthermore,  for these algorithms, instead of directly dealing with adjacency matrix $\A$, a pre-processing step called {\em regularized graph Laplacian} \citep{Rohe2016,Joseph2016}  (see \Cref{Regularized_graph_Laplacian}) can be added to regularize the adjacency matrix $\A$. Hence, correspondingly, rPCA first regularizes the adjacency matrix $\A$ to generate a regularized graph Laplacian $\Lb$ (as in \Cref{Regularized_graph_Laplacian}), and then applies oPCA to $\Lb$. Similarly, rD-SCORE  first generates a regularized graph Laplacian $\Lb$ and then applies D-SCORE (\Cref{D_SCORE}) to $\Lb$. The rD-SCORE$_q$  follows the similar regularization procedure of rD-SCORE, but applies D-SCORE$_q$ (\Cref{Row_Normalization}) to $\Lb$ instead of D-SCORE.  Specially for $q=2$, rD-SCORE$_2$ is almost the same as the DI-SIM algorithm in \cite{Rohe2016}. The only difference lies in that \cite{Rohe2016} provided a bi-clustering structure, whereas rD-SCORE$_2$ provides a single cluster structure for nodes. Similarly, rD-SCORE$_q$ can be seen as an extension of the DI-SIM algorithm from the $\ell_2$-norm to the $\ell_q$-norm for any positive integer $q$.

\begin{algorithm}
	\caption{oPCA} \label{oPCA}
	\SetKwInOut{Input}{Input}
	\SetKwInOut{Output}{Output}
	\Input{The number  $K$ of communities and the adjacency matrix $A$.}
	Obtain the first $K$ leading left and right singular vector matrices $V$ and $U$ of $A$.
	
	Put $V$ and $U$ together to  form a matrix $R = [V,U]$ , and apply the K-means method to $R$.
	
	\Output{The community labels of the nodes in the adjacency matrix $A$.}
\end{algorithm}

\begin{algorithm}
	\caption{Regularized graph Laplacian} 	\label{Regularized_graph_Laplacian}
	\SetKwInOut{Input}{Input}
	\SetKwInOut{Output}{Output}
	\Input{The adjacency matrix $A$.}
	Calculate the diagonal matrix $O^{\tau}$, $P^{\tau} \in R^{n \times n}$, where $O^{\tau}(i,i) =\tau + \sum_{j=1}^{n} A(i,j) $ and $P^{\tau}(i,i) =\tau + \sum_{j=1}^{n} A(j,i) $. The regularization parameter $\tau$ is usually set as the average degree $\tau = \sum_{i,j = 1}^{n} A(i,j)/n$.
	
	Let $L = (O^{\tau})^{-1/2} A (P^{\tau})^{-1/2 }$.
	
	\Output{The regularized graph Laplacian matrix $L$.}
\end{algorithm}

\subsection{Applications to Real Data Sets}
\subsubsection{Applications to Political Blogs Data}
In this subsection, we apply the  above mentioned eight algorithms to the web blogs data introduced in \cite{Adamic2005}.  The blogs data was collected at 2004 presidential election. Such political blogs data can be represented by a directed graph, in which each node in the graph corresponds to a web blog labelled either as liberal or conservative.  An directed edge from node $i$ to node $j$ indicates that there is a hypelink from blog $i$ to blog $j$.  Clearly, such a political blog graph is {\em directed}. The fact that there is a hyperlink from blog $i$ to $j$ does not imply there is also a hypelink from blog $j$  to $i$. Hence, the adjacency matrix of the political blogs data is an asymmetric matrix.

In our experiment, we first extract the largest  component of the graph, which contains $1222$ nodes, and denote it by an asymmetric directed adjacency matrix $\A$. Then, we extract the largest  components of $\A^T\A$ and $\A\A^T$, and use $\mathcal{S}_r$ and $\mathcal{S}_l$ to denote the node sets of these two largest connected components, respectively. We define the intersection set $\mathcal{S} \equiv \mathcal{S}_r \cap  \mathcal{S}_l$, which contains $823$ nodes.

We run all of the six spectral algorithms in the following two different approaches. In the first approach, we run these six algorithms on the entire graph that contains $1222$ nodes. 
In the second approach, we first run the  six algorithms on the intersection set $\mathcal{S}$, and then we use the attachment technique to attach nodes outside $\mathcal{S}$ to clusters (as described in \Cref{Practical_Algorithm}). We repeat each algorithm on each setting 500 times and take the mean of the total number of misclustered nodes. 
Since APL and BCPL are designed for undirected network,  we first build a symmetric adjacency matrix based on the asymmetric one, and then apply APL and BCPL to the symmetric one. Since symmetric network does not have intersection approach, we count the misclustered node of APL and BCPL in intersection approach directly form their result in entire graph approach while limited the node only in the intersection set calculated in DSCORE algorithms.



\begin{table}[H]
	\captionsetup{font=normalsize}
	\centering
	\begin{tabular}{llll}
		&Entire Graph (1222)  &Int. with Attach. (1222)   &Intersection (823)  \\  	\hline 
	oPCA	&434 &300   &217 \\  	\hline 
	rPCA	&414  &246  &190\\  	\hline 
	BCPL	&379  &379  &236 \\  	\hline 
	APL	&61  &61  &28 \\  	\hline 
	DSCORE	&142  &60   &22 \\  	\hline 
	rDSCORE	&139  &60  &22 \\  	\hline 
	DSCORE2	&141  &61   &23 \\  	\hline 
	rDSCORE2	&142  &60   &26 \\  	\hline 
	\end{tabular}
	\caption{Misclustered nodes in political blog data.}  \label{web_blogs_data}
\end{table}



The experiment results are shown in \Cref{web_blogs_data}.
 It can be observed from the table that D-SCORE and D-SCORE$_q$ almost have the same performance, and rD-SCORE  and rD-SCORE$_q$ almost have the same performances.  Furthermore, D-SCORE and D-SCORE$_2$ perform better than oPCA, which implies that the ratio step to remove the heterogeneous parameters helps greatly to improve the clustering accuracy. The same occurs in the comparison of the algorithms with the regularized graph Laplacian. Moreover, APL almost performs the same as DSCORE type algorithms while BCPL doesn't. We will show in the stimulation section that the performance of APL is easily affected by the community structure, and cause its performance unstable.

 Next, by comparing the first and second columns  in \Cref{web_blogs_data}, we observe  that for all algorithms, it is much better to run the algorithms on the intersection set and then attach the outside nodes than  directly  running the algorithm on the entire graph. Especially, the intersection-with-attachment  technique introduced in Algorithm 3 has improved all the original D-SCORE algorithms for the entire graph. 
 
 To explain this improvement further, we plot  the vectors that the algorithms (i.e., oPCA, D-SCORE, D-SCORE$_q$) use in the clustering in \Cref{intersection_original,intersection_DSCORE,intersection_DSCORE2}. Note that in these algorithms, before the $k-$means step, each node corresponds to one row of a matrix. We thus use these row vectors as the coordinate of the nodes and plot them in the figures. The \Cref{political_blog_O_WG,political_blog_1_WG,political_blog_2_WG} include the nodes in the entire graph. The \Cref{political_blog_O_INT,political_blog_1_INT,political_blog_2_INT} include the nodes in the intersection set. We use red triangles  and yellow squares  to represent nodes in the liberal and conservative communities, respectively.  Note that extreme coordinates in  
\Cref{political_blog_1_WG,political_blog_2_WG} are already thresholded  and form the imaginary borders for better presentation;  these extreme coordinates are the effect of having extremely small numbers (computational errors for 0) as the denominators when D-SCORE and D-SCORE$_q$ are used directly on the entire graph, as explained in Section 3. 

\begin{figure}
	\centering
	\begin{subfigure}{0.46\linewidth}
		\captionsetup{font=scriptsize}
		\includegraphics[height=5cm, width=1\linewidth]{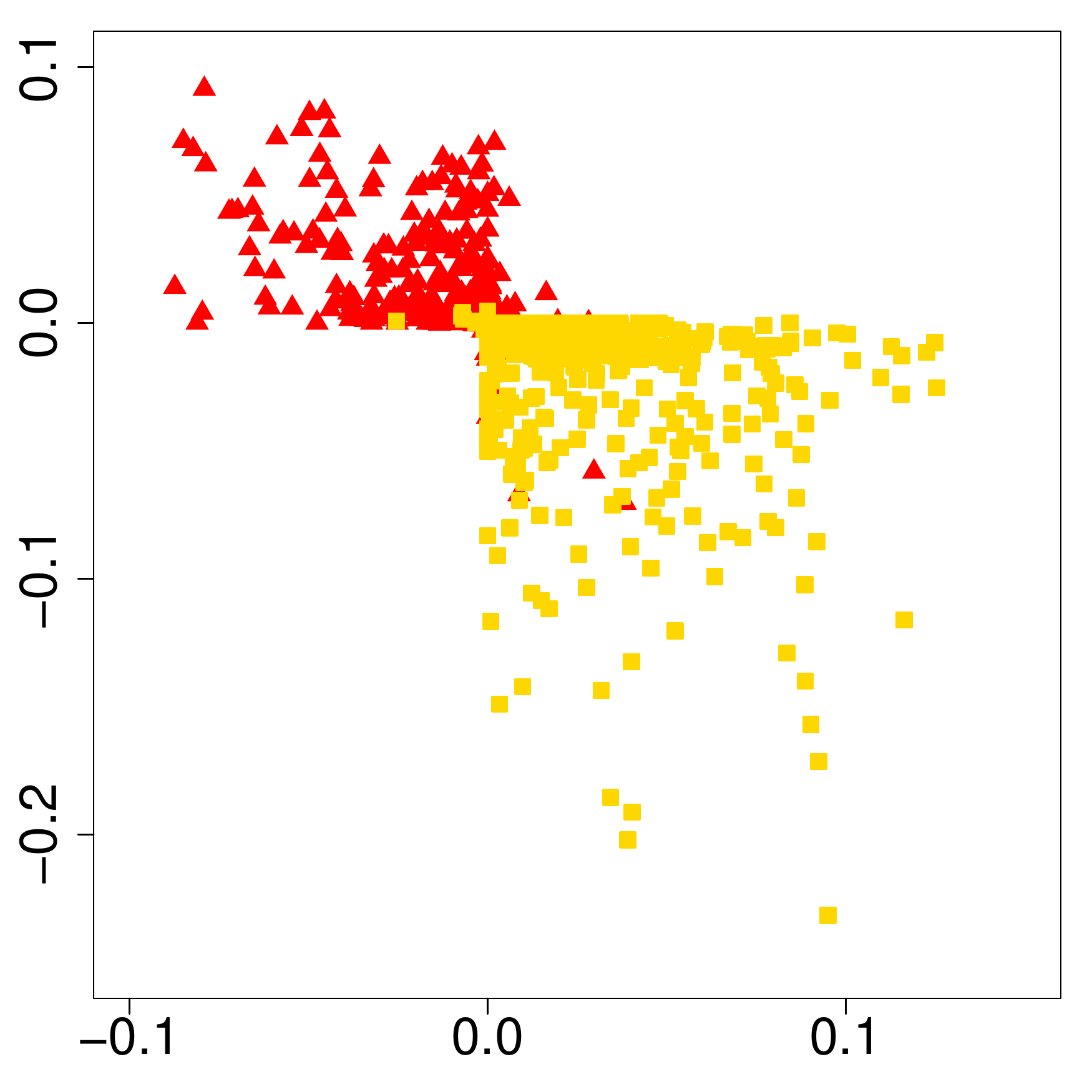}
		\caption{Entire Graph} \label{political_blog_O_WG}
	\end{subfigure}
	\quad
	\begin{subfigure}{0.46\linewidth}
		\captionsetup{font=scriptsize}
		\includegraphics[height=5cm, width=1\linewidth]{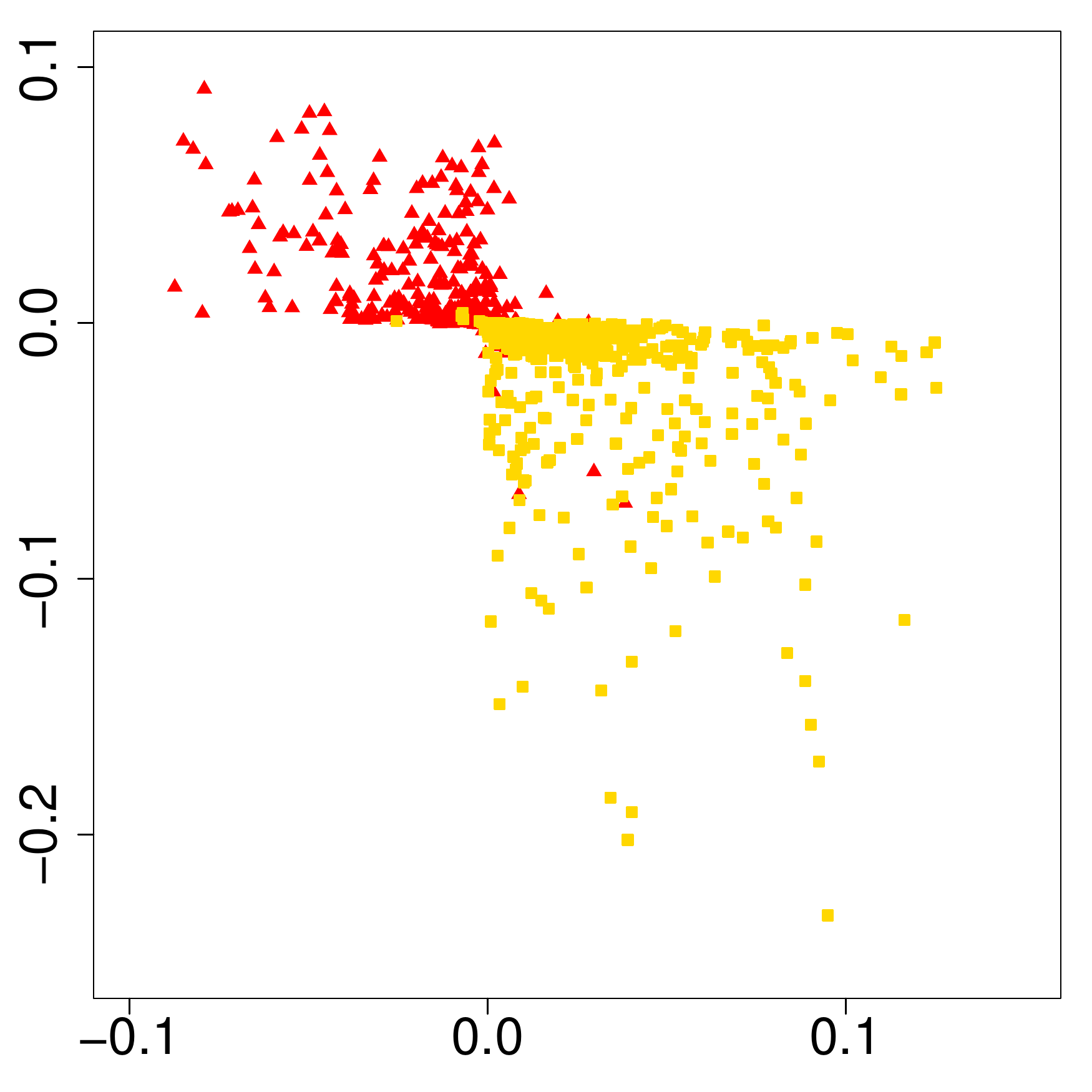}
		\caption{Intersection} \label{political_blog_O_INT}
	\end{subfigure}
	\caption{Comparison of the clustering vectors in the entire graph and in the intersection graph of original spectral clustering. The $x$-axis is the  second leading left singular vector, and the $y$-axis is the  second leading right singular vector. \label{intersection_original}}
\end{figure}

%



\begin{figure}
	\centering
	\begin{subfigure}{0.46\textwidth}
		\captionsetup{font=scriptsize}
		\includegraphics[height=5cm, width=1\linewidth]{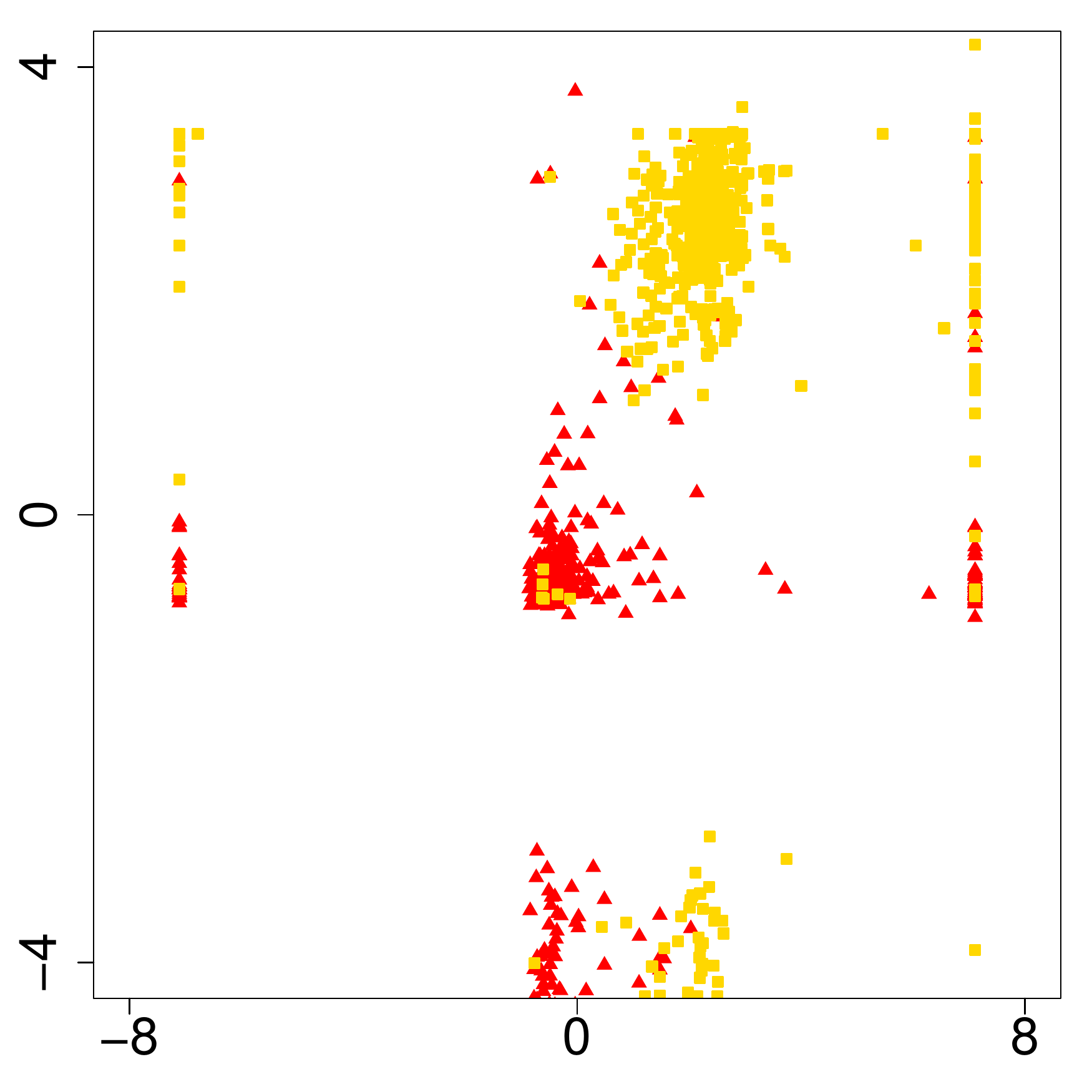}
		\caption{Entire Graph}\label{political_blog_1_WG}
	\end{subfigure}
	\quad
	\begin{subfigure}{0.46\textwidth}
		\captionsetup{font=scriptsize}
		\includegraphics[height=5cm, width=1\linewidth]{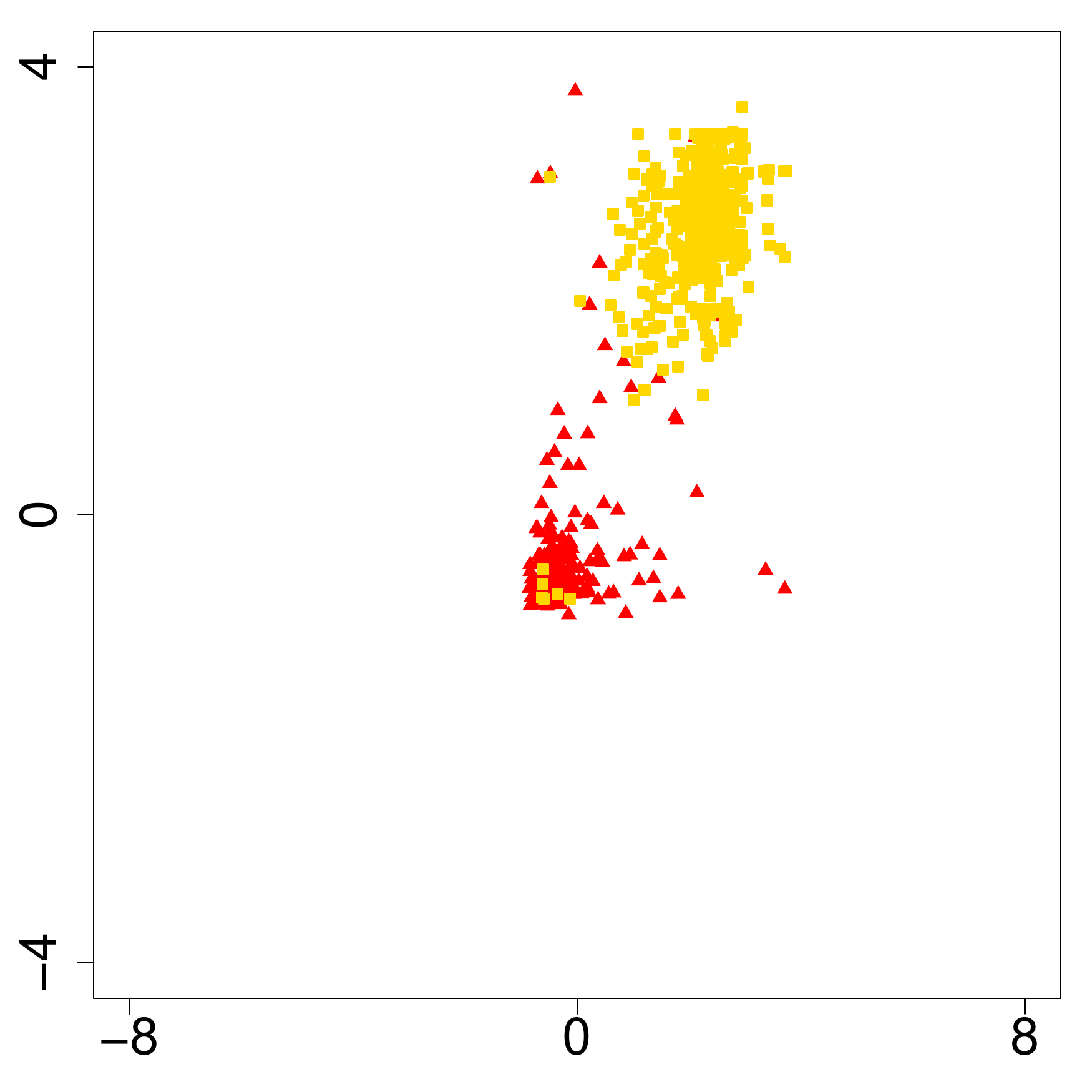}
		\caption{Intersection} \label{political_blog_1_INT}
	\end{subfigure}
	\caption{Comparison of the clustering vectors of entire graph and the intersection in D-SCORE. The $x$-axis is the left ratio vector, and the $y$-axis is the right ratio vector. }\label{intersection_DSCORE}
\end{figure}

\begin{figure}	\label{initutively_graph}
	\centering
	\begin{subfigure}{0.46\textwidth}
		\captionsetup{font=scriptsize}
		\includegraphics[height=5 cm, width=1\linewidth]{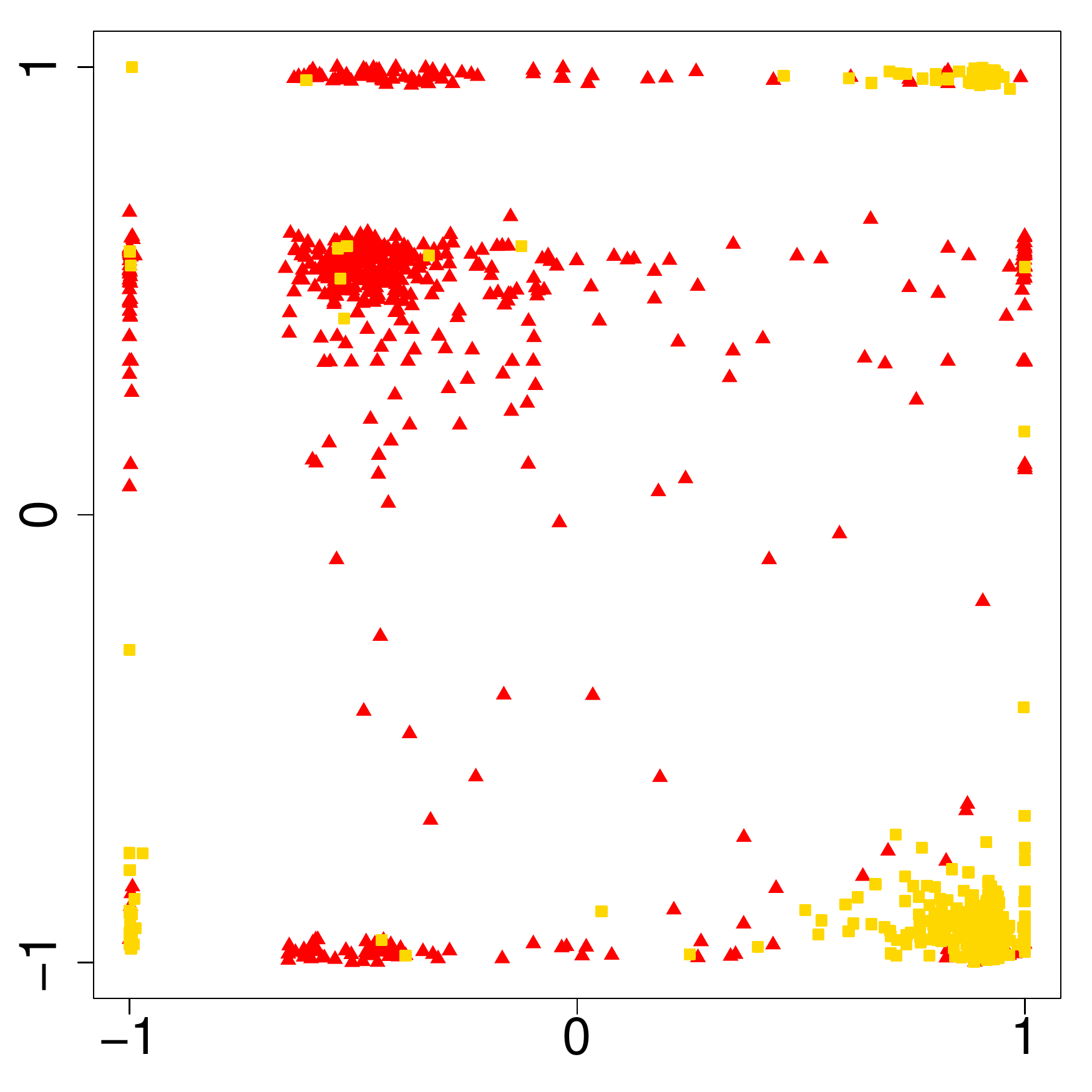}
		\caption{Entire Graph} \label{political_blog_2_WG}
	\end{subfigure}
	\quad
	\begin{subfigure}{0.46\textwidth}
		\captionsetup{font=scriptsize}
		\includegraphics[height=5 cm, width=1\linewidth]{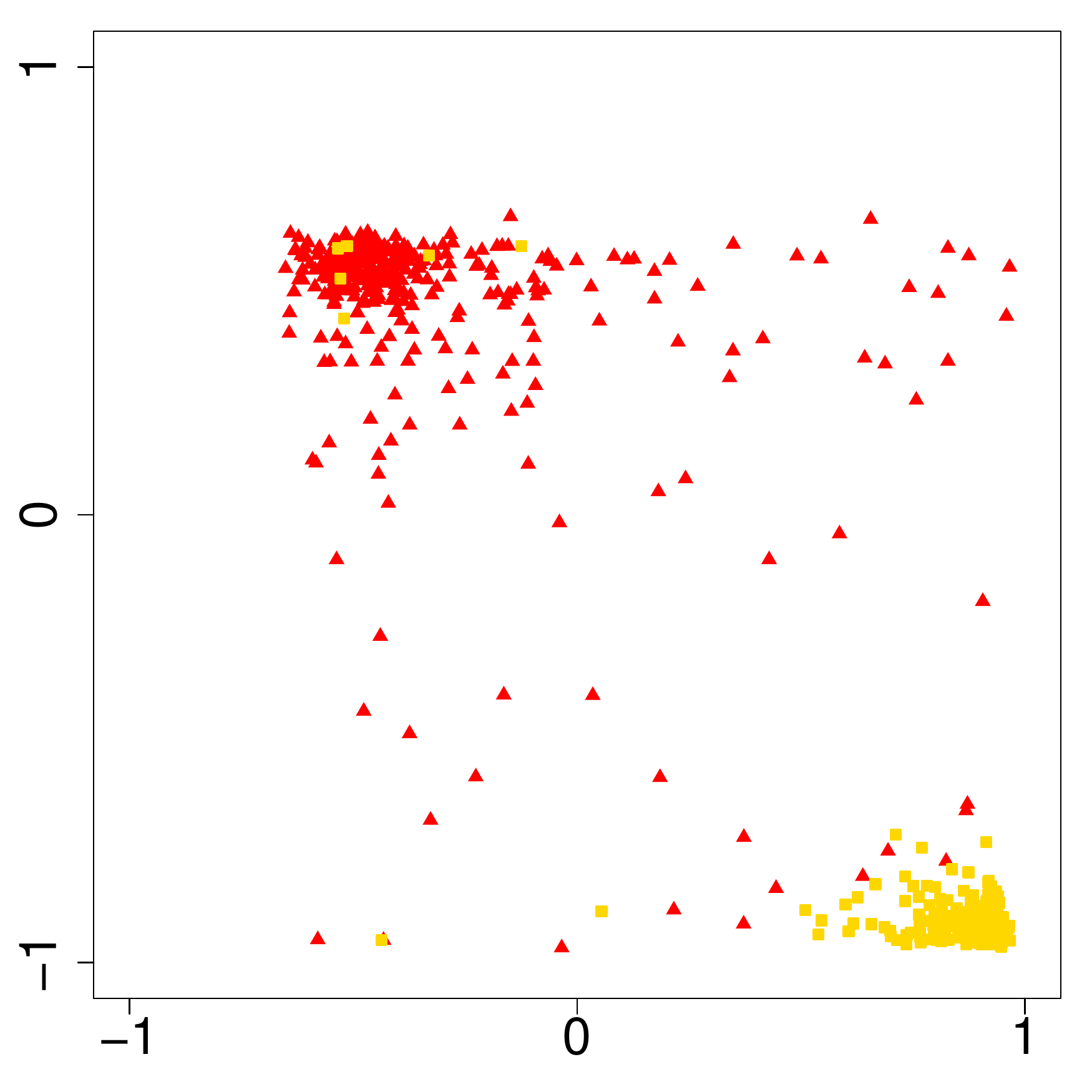}
		\caption{Intersection} \label{political_blog_2_INT}
	\end{subfigure}
	\caption{Comparison of the  singular vectors of entire graph and the intersection in D-SCORE$_q$. The $x$-axis is the second left ratio vector, and the $y$-axis is the second right ratio vector.}\label{intersection_DSCORE2}
\end{figure}

First, we compare \Cref{intersection_original} (which applies the original spectral clustering) with \Cref{intersection_DSCORE,intersection_DSCORE2} (which apply the D-SCORE and D-SCORE$_2$, respectively). It is clear that nodes in \Cref{intersection_DSCORE,intersection_DSCORE2} are much more separable than nodes in \Cref{intersection_original} due to the ratio step in D-SCORE and D-SCORE$_2$. Furthermore, We observe that the intersection graph (\Cref{political_blog_1_INT,political_blog_2_INT}) extracts the center of the entire graph and deletes nodes near the border in \Cref{political_blog_1_WG,political_blog_2_WG}, which act as noise and   mislead the clustering result. 
The intersection-with-attachment technique works by taking the clustering results for the intersection (shown here) and attaching the noise nodes to these clusters using the links in the original network $A$  (not shown here), and hence yields better performance.



\subsubsection{Applications to  Email-Eu-Core Network}
In this subsection, we apply the  above mentioned eight algorithms to the email-Eu-core network  introduced in \cite{snapnets}.  The email data was collected from a large European research institution, and a directed edge from node $i$ to node $j$ indicates that  person $i$ has sent at least one email to person $j$. Clearly, the email-Eu-core network is also a directed network. There are many communities in this network, but we extract the top $4$ largest communities which contains $297$ nodes as  the entire graph and $252$ nodes in intersection graph. We repeat the experiment $500$ times and show the mean error in \cref{email-EuALL_data}. The experimental observation is similar with that of the political blog data, and thus we omit it for brevity. 
%

\begin{table}[H]
	\captionsetup{font=normalsize}
	\centering
	\begin{tabular}{llll}
		&Entire Graph (297)  &Int. with Attach. (297)   &Intersection (252)  \\  	\hline 
		oPCA	&107  &78   &72 \\  	\hline 
		rPCA	&89  &57   &53\\  	\hline 
		BCPL	&23 &23  &18 \\  	\hline 
		APL	&17  &17   &12 \\  	\hline 
		DSCORE	&23  &7 &6 \\  	\hline 
		rDSCORE	&25   &7   &6 \\  	\hline 
		DSCORE2	&15  &4  &3 \\  	\hline 
		rDSCORE2	&16 &4   &3 \\  	\hline 
	\end{tabular}
	\caption{Misclustered nodes in email-Eu-core network. }  \label{email-EuALL_data}
\end{table}

\subsection{Simulations} \label{experiments_on_Synthetic_data}
In this  section, we compare the eight algorithms described in \Cref{sec:sixalg} through a series of simulations. In the experiments, we first generate an adjacency matrix $\A_{0}$ by Directed-DCBM, and then extract the largest  connected component $\A$ of $\A_0$ with the node set of $\A$ denoted by $\mathcal{S}_{0}$. We also extract the largest  connected components of $\A^T\A$ and $\A\A^T$, and denote the node sets as $\mathcal{S}_1$ and $\mathcal{S}_2$, respectively. Let $\mathcal{S}= \mathcal{S}_1 \cap  \mathcal{S}_2$. We also apply the six spectral algorithms in two approaches: (i) the entire graph approach, where we run the six algorithms over the set $\mathcal{S}_{0}$; 
and (ii) intersection-with-attachment approach, where we run the six algorithms over the intersection set $\mathcal{S}$, and then use the attachment technique to cluster nodes outside the intersection set. The usage of APL and BCPL in simulation is the same as that in real data experiment. Since the symmetric adjacency matrix does not have intersection set issue, we directly plot the result of APL and BCPL in entire graph approach in the intersection with attachment approach for comparison.


\subsubsection{Block Model with Symmetric  Structure} In this experiment, we generate the data by DCBM by setting the heterogeneous parameters $\thetab$ such that $P(\thetab(i) = 0.5) = 0.01$, $P(\thetab(i) = 0.1) = 0.05$ and $P(\thetab(i) = 0.6) = 0.4$. We set $\deltab(i) = \thetab(i)$ for all $i \in \{1, \cdots n \}$. Also, we set the block matrix
$\B = \begin{bmatrix}
1, &0.4\\
0.4, &1
\end{bmatrix}$, which is symmetric. Let  $K = 2$. Then, we uniformly randomly assign community labels to nodes and let the total number  $n$ of nodes go from $800$ to $1200$ with the step size $50$. For each $n$, we repeat the experiment $500$ times and \Cref{experiment_BM} plots the average of the misclustered rate.
\begin{figure}[h]
	\centering
	\begin{subfigure}{0.45\textwidth} 
		\captionsetup{font=small}
		\includegraphics[height=5cm,width=\linewidth]{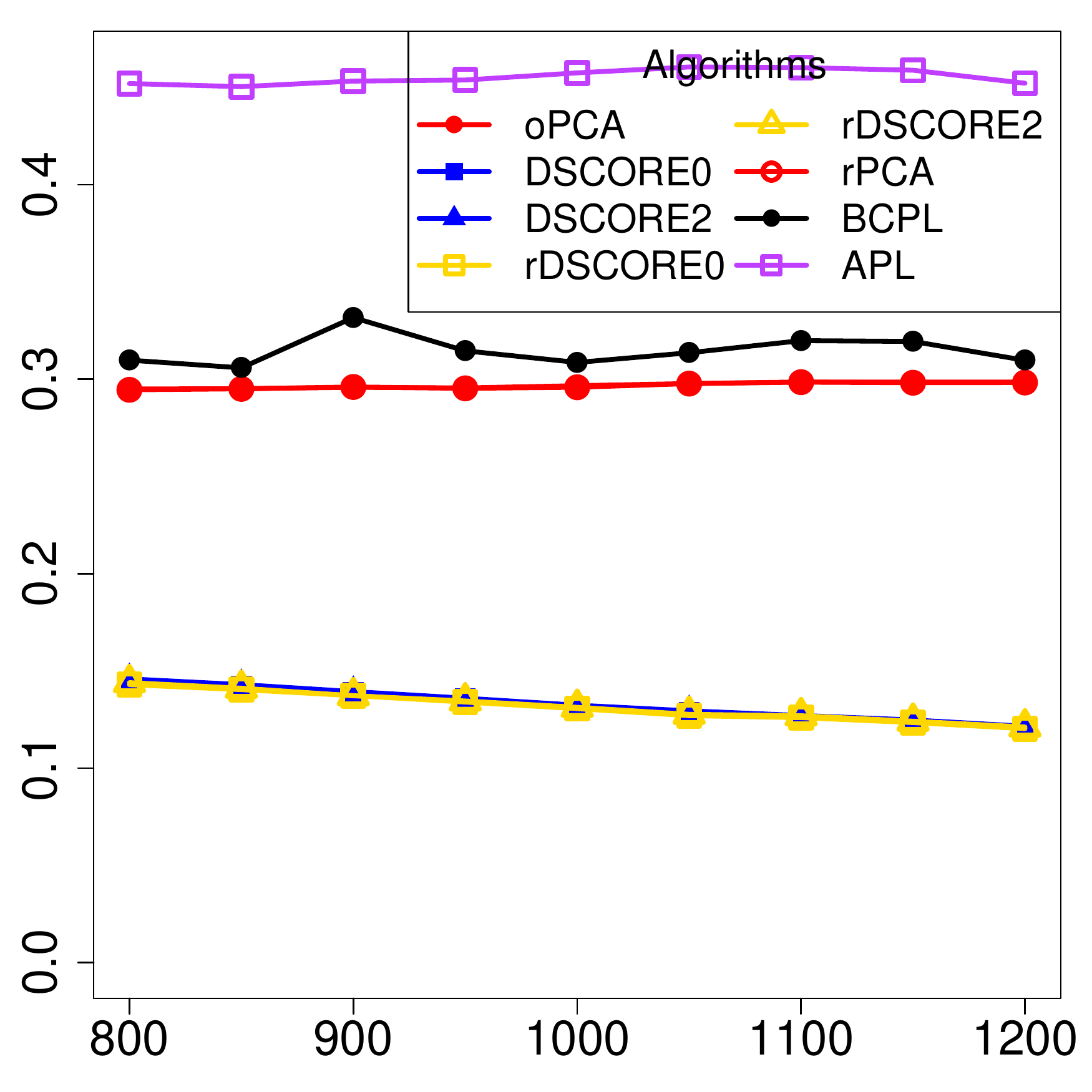}
		\caption{Entire graph} \label{experiment_BM_0}
	\end{subfigure}
	\begin{subfigure}{0.45\textwidth}
		\captionsetup{font=small}
		\includegraphics[height=5cm,width=\linewidth]{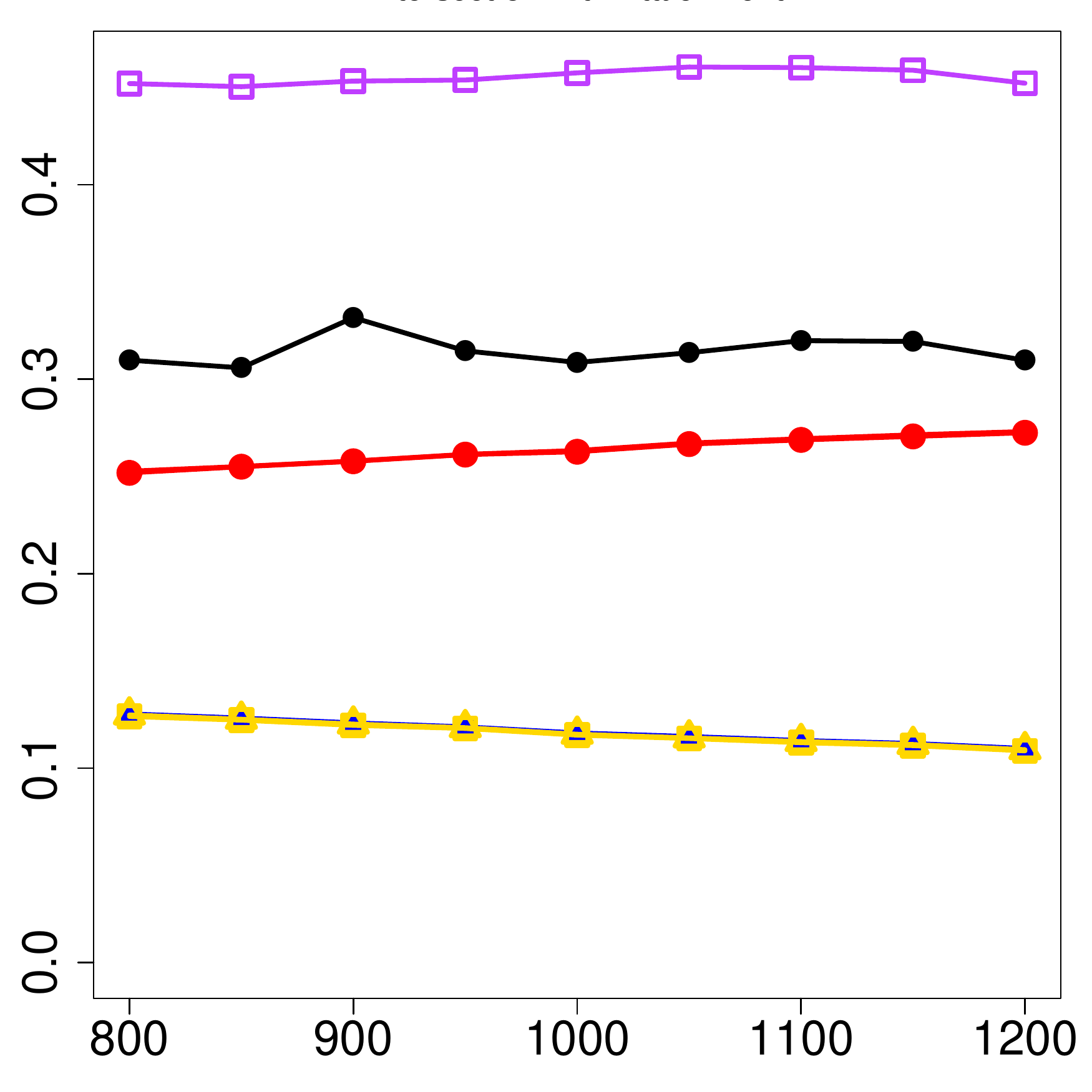}
		\caption{Intersection with attachment} \label{experiment_BM_1}
	\end{subfigure}
	\caption{Comparison of the  misclustering rate under SBM with symmetric structure. The  horizontal axis is the number of nodes in the entire graph, and the vertical  axis is the minsclutering rate.
	} \label{experiment_BM}
\end{figure}

It can be observed that although the model is symmetric,  D-SCORE and D-SCORE$_q$   still perform better than oPCA, APL and BCPL, and the performance is similar with its corresponding pre-precessing version.  Also, by comparing \Cref{experiment_BM_0,experiment_BM_1}, we observe that the intersection-with-attachment technique  improves all variants of the D-SCORE algorithms.


\subsubsection{DCBM with Symmetric and Dense Structure} In this experiment, we set the block matrix
$\B = \begin{bmatrix}
1, &0.4\\
0.4, &1
\end{bmatrix}$
with  two  communities.  We randomly choose the heterogeneous parameter $\thetab$ for nodes with  $P(\thetab(i) = 0.5) = 0.05$, $P(\thetab(i) = 0.1) = 0.05$ and $P(\thetab(i) = 0.6) = 0.4$. We set $\deltab(i) = \thetab(i)$ for all $i \in \{1, \cdots n \}$. Other  parameters are chosen to the same as the previous experiment. 

\begin{figure} [h]
	\centering
	\begin{subfigure}{0.45\textwidth}
		\captionsetup{font=small}
		\includegraphics[height=5cm, width=\linewidth]{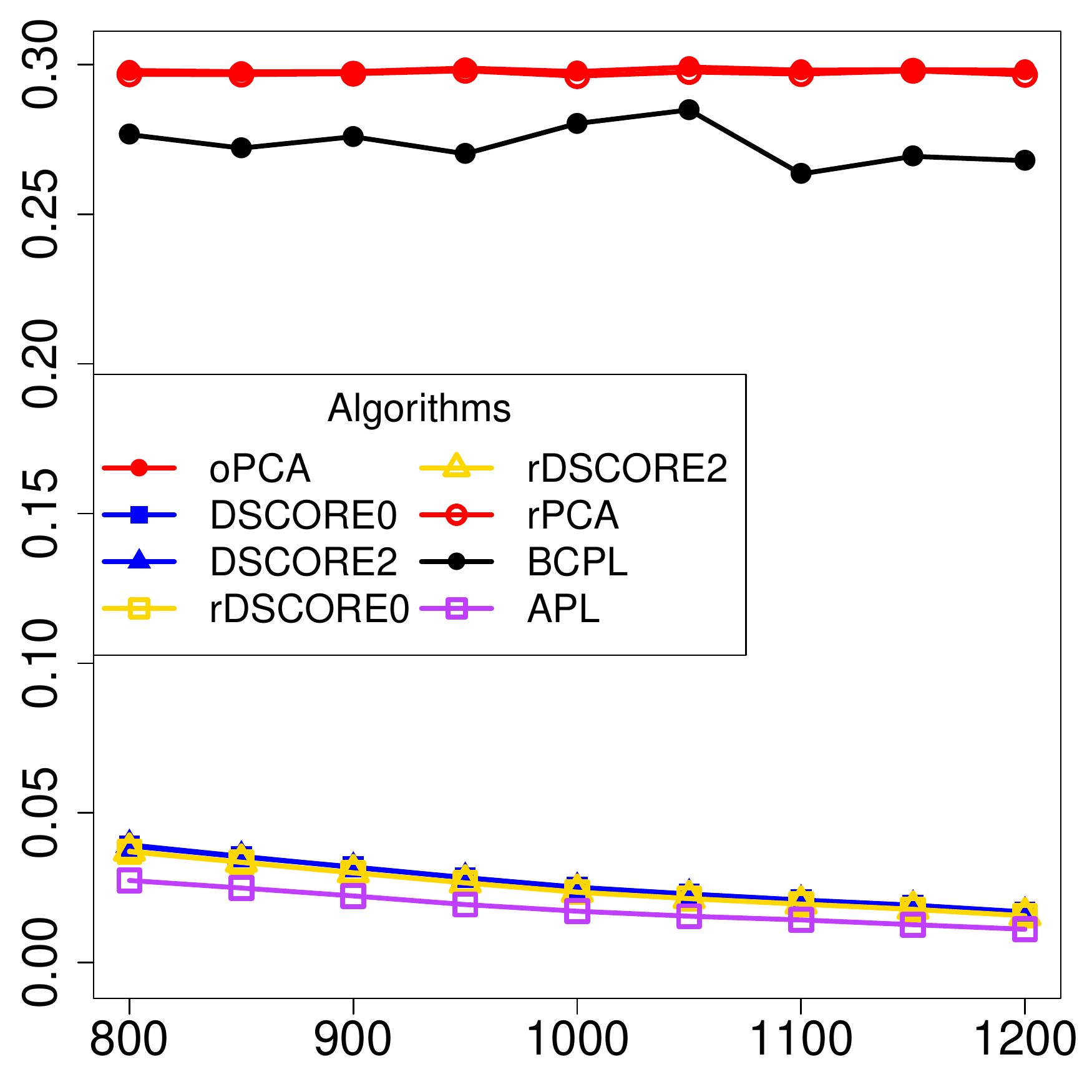}
		\caption{Entire graph} \label{expreiment_2_a}
	\end{subfigure}
	\begin{subfigure}{0.45\textwidth}
		\captionsetup{font=small}
		\includegraphics[height=5cm, width=\linewidth]{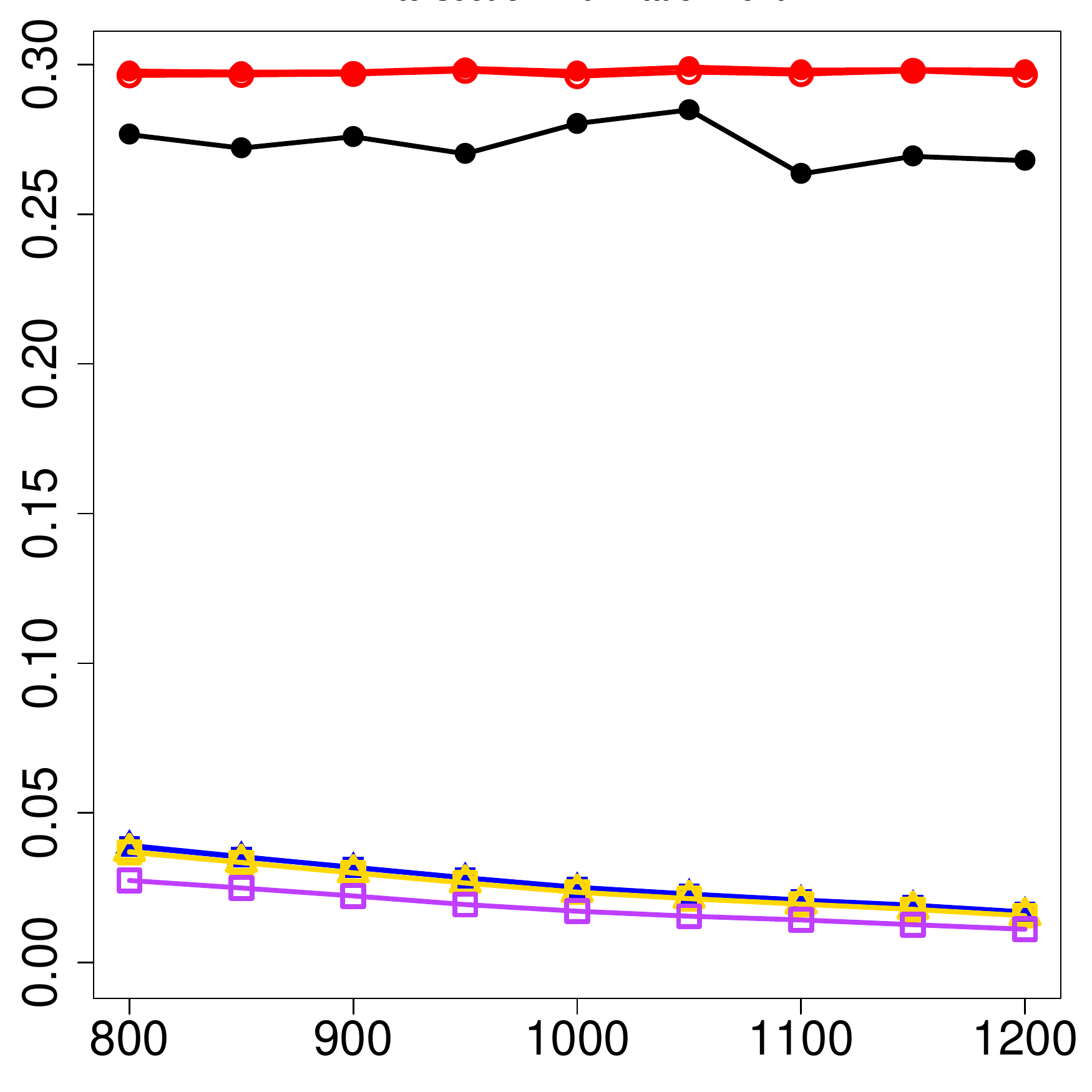}
		\caption{Intersection with attachment} \label{expreiment_2_c}
	\end{subfigure}
	\caption{Comparison of the  misclustering rate under DCBM with symmetric structure. The vertical axis is the number of nodes in the entire graph, and the horizontal axis is the minsclutering rate.} \label{experiment_2_DCBM}
\end{figure}

 The mean of misclustering rate  is plotted in \Cref{experiment_2_DCBM}. It can be observed   that  DSCORE, DSCORE$_q$, rDSCORE and rDSCORE$_q$ have almost the same performance and perform much better than oPCA and rPCA. This implies   that the ratio technique in these algorithms greatly helps to improve the clustering accuracy. The performance of BCPL is better than oPCA and rPCA while worse than the proposed algorithms. What surprises us is that APL performs pretty well in this setting.


\subsubsection{DCBM with Asymmetric and Sparse Structure} \label{sparse_graph}
In this experiment, we set the block matrix
$\B = \begin{bmatrix}
1, &0.4\\
0.5, &1
\end{bmatrix}$
, the number of communities $K = 2$, and the heterogeneous parameter $\thetab$ such that $P(\thetab(i) = 0.5) = 0.01$, $P(\thetab(i) = 0.1) = 0.01$ and $P(\thetab(i) = 0.6) = 0.4$. In this experiment, we randomly pick  $\deltab$ in the same way as $\thetab$ instead of setting $\thetab(i) = \deltab(i)$, which increases the asymmetric structure of the model. Other parameters are chosen to the same as the previous experiment.  The mean of the misclustering rate  is plotted in \Cref{experiment_3_DCBM}.

\begin{figure}[H]
	\centering
	\begin{subfigure}{0.45\textwidth}
		\captionsetup{font=small}
		\includegraphics[height=5cm, width=\linewidth]{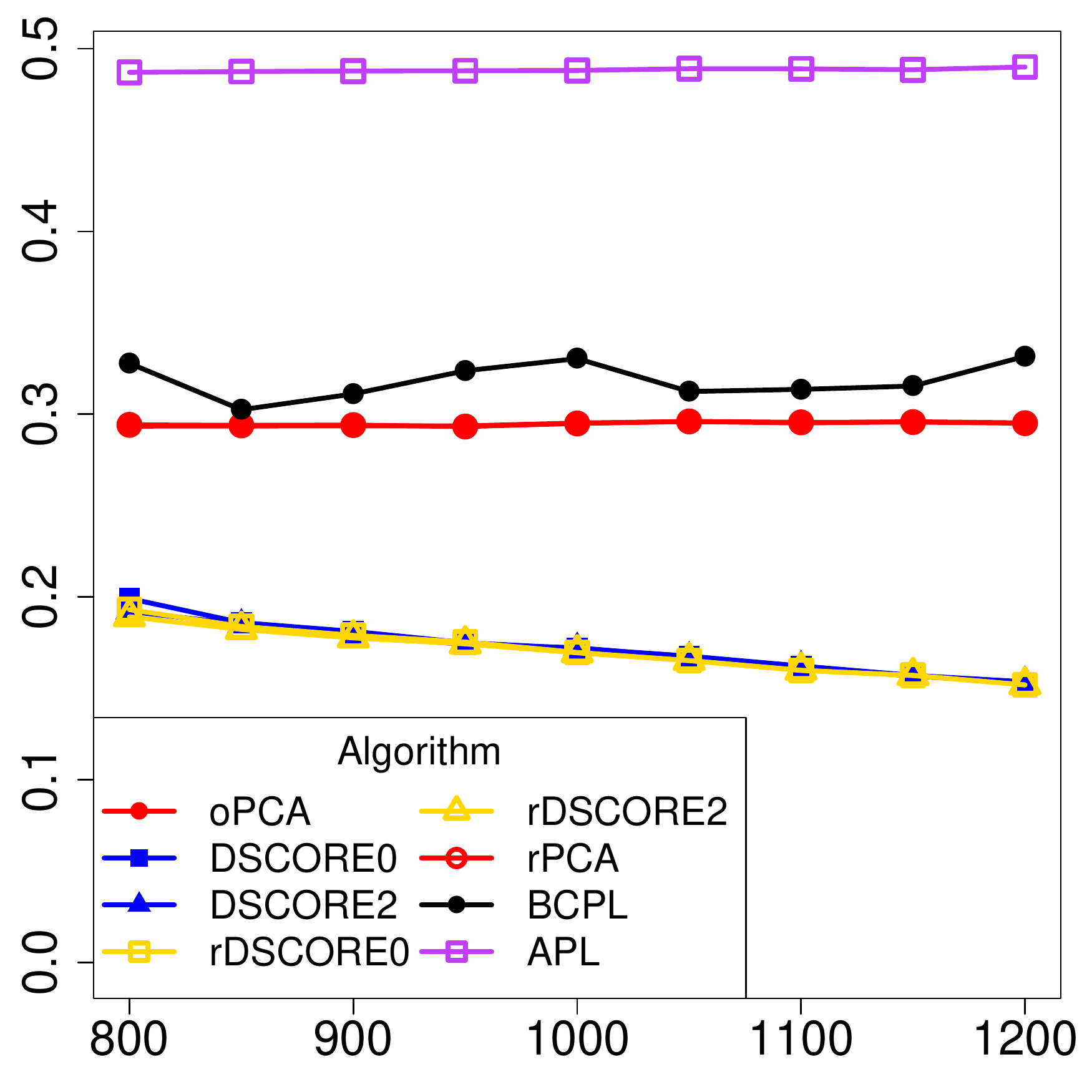}
		\caption{Entire graph} \label{expreiment_3_a}
	\end{subfigure}
	\begin{subfigure}{0.45\textwidth}
		\captionsetup{font=small}
		\includegraphics[height=5cm, width=\linewidth]{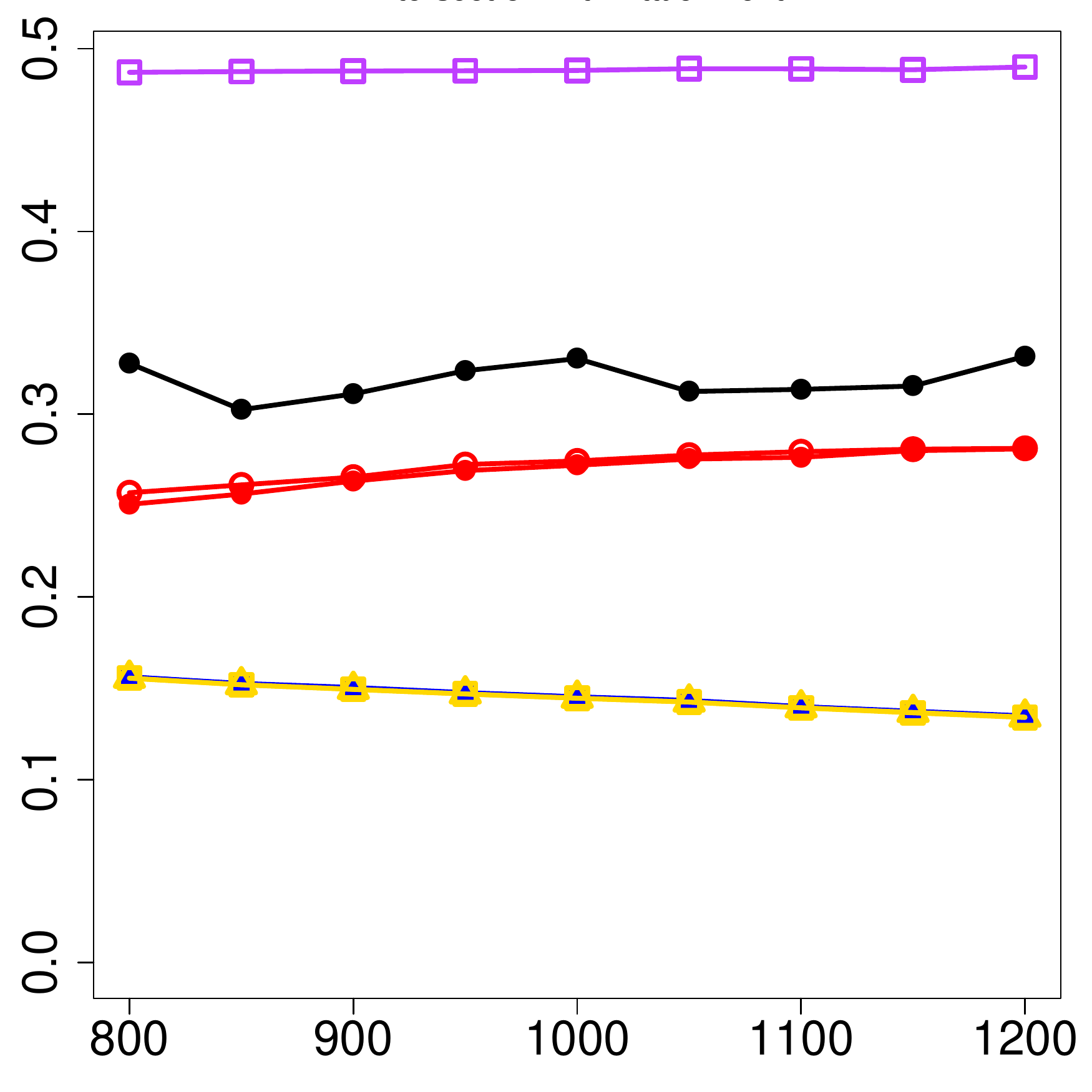}
		\caption{Intersection with attachment} \label{expreiment_3_c}
	\end{subfigure}
	\caption{Comparison of the  misclustering rate under DCBM with asymmetric and sparse structure. The horizontal axis is the number of nodes in the entire graph, and the vertical axis is the minsclutering rate.} \label{experiment_3_DCBM}
\end{figure}

We observe form \Cref{expreiment_3_c} that DSCORE, DSCORE$_q$, rDSCORE and rDSCORE$_q$ perform the same and are better than oPCA, rPCA, APL and BCPL, which implies that the ratio technique greatly helps. Also, by comparing \Cref{expreiment_3_a,expreiment_3_c}, we observe that the intersection-with-attachment approach performs better than the entire graph approach.

\subsubsection{DCBM with Asymmetric and  Dense Structure} \label{dense_graph}
In this experiment, we set the block matrix
$\B = \begin{bmatrix}
1, &0.4\\
0.5, &1
\end{bmatrix}$
, the number of communities $K = 2$, and the heterogeneous parameter $\thetab$ such that $P(\thetab(i) = 0.5) = 0.05$, $P(\thetab(i) = 0.1) = 0.01$ and $P(\thetab(i) = 0.6) = 0.4$. The parameter $\deltab$ is randomly picked in the same way as $\thetab$. Other  parameters are chosen to the same as the previous experiment.  The mean of the misclustering rate  is plotted in \Cref{experiment_4_DCBM}.

\begin{figure}[h]
	\centering
	\begin{subfigure}{0.45\textwidth}
		\captionsetup{font=small}
		\includegraphics[height=5cm, width=\linewidth]{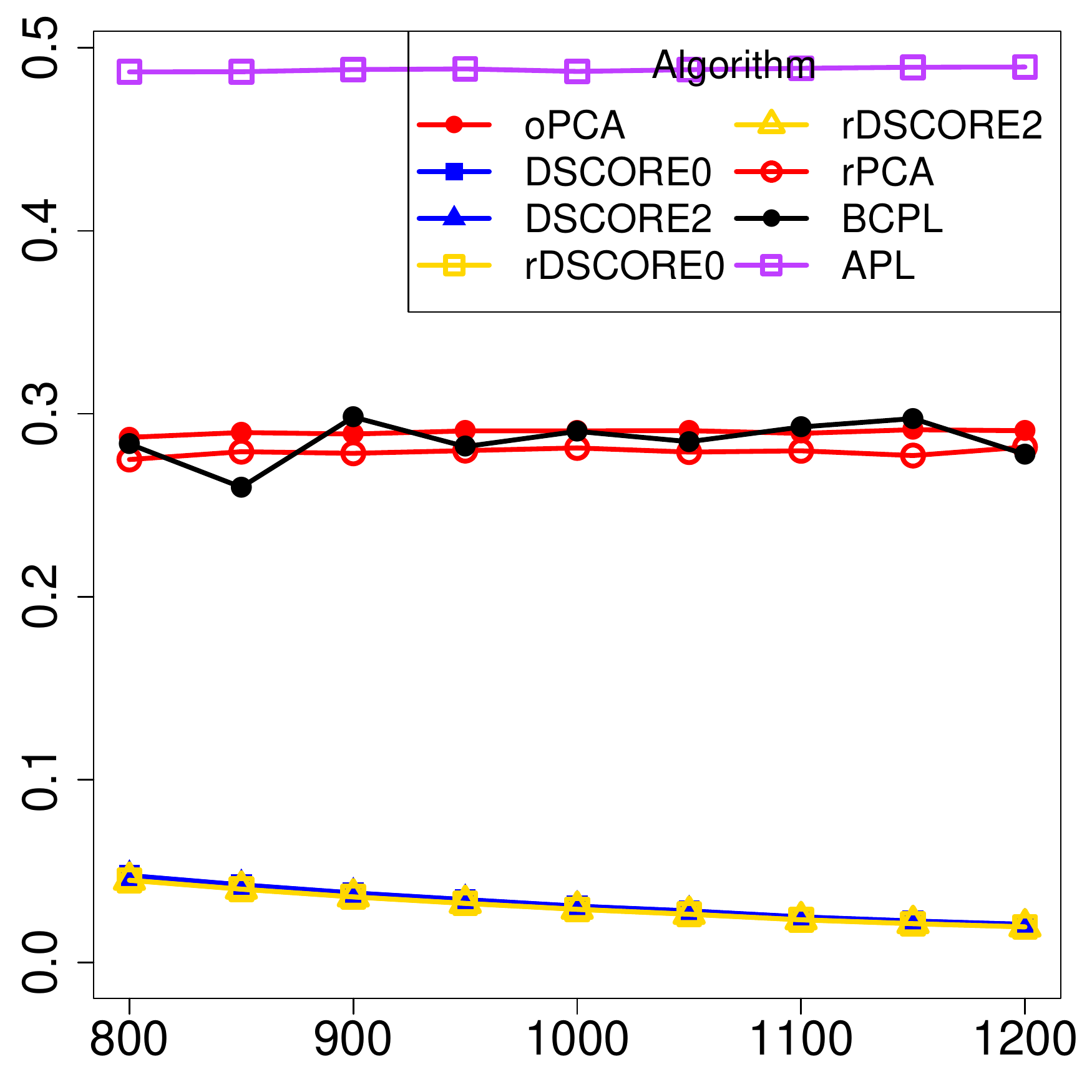}
		\caption{Entire graph} \label{expreiment_4_a}
	\end{subfigure}
	\begin{subfigure}{0.45\textwidth}
		\captionsetup{font=small}
		\includegraphics[height=5cm, width=\linewidth]{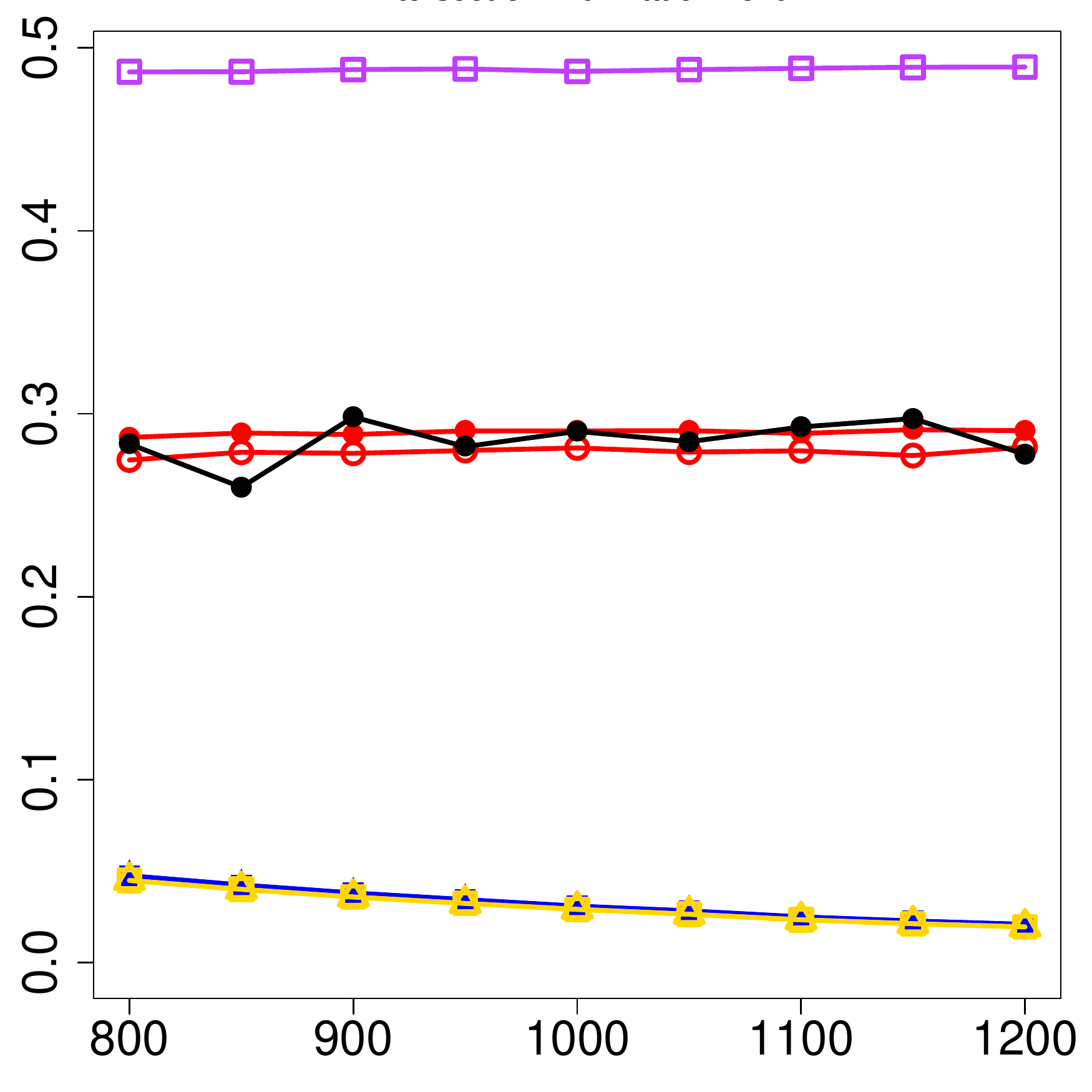}
		\caption{Intersection with attachment} \label{expreiment_4_c}
	\end{subfigure}
	\caption{Comparison of the  misclustering rate under DCBM with asymmetric and dense structure. The horizontal axis is the number of nodes in the entire graph, and the vertical axis is the minsclutering rate.} \label{experiment_4_DCBM}
\end{figure}

Here, our setting of parameters makes the graph denser than that in the previous experiment (\Cref{sparse_graph}). It can be seen that 
the performance of the entire graph is almost the same as that of the intersection with attachment. This suggests that the intersection-with-attachment technique is more efficient for sparse networks.  This should not be surprising because, for dense networks,  the nodes are more connected and noise nodes  that have low degrees and need the attachment step are reduced.

\vspace{-3mm}
\section{Conclusion}
In this paper, we provided  theoretical guarantee and experimental results for two spectral clustering algorithms for networks with directed edges. In theory,  we established the performance guarantee for  D-SCORE and D-SCORE$_q$ under Direct-DCBM. We also  conducted extensive experiments to demonstrate the advantage  of the improved D-SCORE algorithms over the original version and  the competitive algorithms.  As an extension, since the translation of network structures into Euclidean coordinates using D-SCORE and SCORE can be easily extended to multi-layer networks and node-attributed networks, the theory presented in this paper can be potentially extended to those more general scenarios.


\acks{Z. Wang and Y. Liang would like to thank the partial support of the U.S. National Science Foundation under the grants ECCS-1818904 and CCF-1801855. The authors appreciate the valuable discussion with Jiashun Jin at Carnegie Mellon University.}



\vspace{1cm}
\noindent\textbf{\Large Appendices}
\appendix

\section{Proof of \Cref{mis clustering nodes} (Convergence of D-SCORE) }\label{app:d-score}

We first provide the proofs for Propositions \ref{lemma5.1}-\ref{M* R gap}, and then combine all these properties together to prove \Cref{mis clustering nodes}. Note that all the propositions and lemmas that we show below  need \Cref{Assumption} and \Cref{Assumption_2} to hold.
\subsection{Proof of \Cref{lemma5.1}} \label{proof_of_proposition_1}
\begin{proof}
	We first let $\Thetab_{\theta}$ and  $\Thetab_{\delta}$ denote the $n \times K$ matrices such that for $ 1 \leqslant i \leqslant n$ and $1 \leqslant k \leqslant K$,
	\begin{flalign*}
	\Thetab_{\theta}(i,k) = \begin{cases}
	\frac{\thetab(i)}{\norml{\thetab^{(k)}}}  &\quad \text{if} \quad c_i = k \\
	0 &\quad \text{if} \quad c_i \neq k
	\end{cases}
	\quad \text{ and } \quad
	\Thetab_{\delta}(i,k) = \begin{cases}
	\frac{\deltab(i)}{\norml{\thetab^{(k)}}}  &\quad \text{if} \quad c_i = k \\
	0 &\quad \text{if} \quad c_i \neq k
	\end{cases}.
	\end{flalign*}
	The matrix   $\Thetab_{\theta}$ serves as a membership matrix with each row, say, the $i$th row, containing only one nonzero entry, whose column index corresponds to the community that node $i$ belongs to.
	
	Then by the above definitions of $\Thetab_{\thetab}, \Thetab_{\deltab}$ and the definitions of $\Psib_{\thetab}, \Psib_{\delta}$ (see \cref{definition_diagonal_Psi}), we can express the expectation matrix $\Omegab = (\Thetab_{\thetab} \norml{\thetab} \Psib_{\thetab}) \B ( \Thetab_{\deltab} \norml{\deltab} \Psib_{\deltab}  )^T$. Denoting $\Sb \equiv \Psib_{\thetab} \B \Psib_{\deltab}^T$, we obtain
	\begin{flalign}
	\Omegab = \norml{\thetab} \norml{\deltab} \Thetab_{\thetab} \Sb \Thetab_{\deltab}^T . \label{Omega_matrix}
	\end{flalign}

	Since the diagonal matrices $\Psib_{\thetab}$ and $\Psib_{\deltab}$ are of full rank, $\rank (\Sb) = \rank(\Psib_{\thetab}\B\Psib_{\deltab}^T) = \rank(\B ) = K$. Thus, the $K \times K$ matrix $\Sb$ is also of full rank and has only non-zero singular values. Then, we denote the  SVD of the matrix $\Sb$ as
	\begin{flalign}
	\Sb=\Yb \bm{\Lambda_S} \Hb^T \label{svd_S},
	\end{flalign}
	where $\bm{\Lambda_S}$ is a $K \times K$ non-zero diagonal matrix with the singular values arranged in a decreasing order, and $\Hb$ and $\Yb$ are $K \times K$ orthogonal matrices.
	
	We substitute \cref{svd_S} into \cref{Omega_matrix} and obtain
	\begin{flalign}
	\Omegab = \norml{\thetab} \norml{\deltab}  (\Thetab_{\thetab} \Yb ) \bm{\Lambda_S}   (\Hb \Thetab_{\deltab})^T. \label{svd_omega}
	\end{flalign}
	
	By the definitions of $\Thetab_{\thetab}$ and $\Thetab_{\deltab}$, $\Thetab_{\thetab}^T \Thetab_{\thetab} = \I$ and $\Thetab_{\deltab}^T \Thetab_{\deltab} = \I$. Thus,
	\begin{flalign*}
	(\Thetab_{\thetab} \Yb )^T \Thetab_{\thetab} \Yb &= \Yb^T \Thetab_{\thetab}^T \Thetab_{\thetab} \Yb = \I, \\
	(\Thetab_{\deltab} \Hb )^T \Thetab_{\deltab} \Hb	&= \Hb^T \Thetab_{\deltab}^T \Thetab_{\deltab} \Hb = \I  .             \numberthis \label{svd_proof_equal_identity}
	\end{flalign*}
	
	By \cref{svd_proof_equal_identity}, we observe that $\Thetab_{\thetab} \Yb $ and $\Thetab_{\deltab} \Hb $ have orthogonal columns. Thus, \cref{svd_omega}  is the compact SVD of the matrix $\Omegab$. Denoting the compact SVD of $\Omegab$ as $\Omegab = \Ub \bm{\Lambda} \Vb^T$,  we have
	\begin{flalign}
	\V &= \bm{\Theta_{\delta}} \Hb, \label{lemma_1_row_form_V} \\
	\U &= \bm{\Theta_{\theta}} \Y, \label{lemma_1_row_form_U}\\
	\bm{\Lambda} &= \norml{\thetab} \norml{\deltab} \bm{\Lambda_S} \label{lemma_1_not_zero}
	\end{flalign}
	where $\bm{\Lambda}$ is a $K \times K$ non-zero diagonal matrix, and $\V$ and $\U$ are $n \times K$ matrices with orthogonal columns.
	
	By \cref{lemma_1_not_zero}, $\Omegab$ has only $K$ non-zero singular values because $\Lambdab_S$ is a $K \times K$ non-zero diagonal matrix, i.e., $\sigma_i(\Omegab) =  \norml{\thetab} \norml{\deltab} \sigma_i(\Sb)$ for $i \leqslant K$ and  $\sigma_i (\Omegab) = 0$ for $i > K$. Therefore, \cref{row represent U} follows from the forms of individual rows of  \cref{lemma_1_row_form_V,lemma_1_row_form_U}.
	
	Since $\Hb$ is an orthogonal matrix, \cref{bound_VU} follows because  $\norml{\bm{\V_{\bar{i}}}} = \normlarge{\frac {\deltab (i)}{\norm {\bm{\deltab^{(c_i)}}}} \bm{\Hb_{\bar{c_i}}}}  = \frac {\deltab (i)}{\norm {\bm{\deltab^{(c_i)}}}}$. By \cref{eq:4.4}, $\norml{ \bm{V_{\bar{i}}}} \asymp {\frac{\deltab(i)}{\bm{\norml{\deltab}}}}$. Following the arguments similar to the above, we have $\norml{\bm{\U_{\bar{i}}}} \asymp \frac{\thetab(i)}{\norml{\thetab}}$.
\end{proof}
\subsection{Proof of \Cref{distance between singular vector V}} \label{proof_of_proposition_2}
In \Cref{distance between singular vector V}, we bound the distance between the singular vectors of $\Omegab$ and those of $\A$. In order to bound such distance, we first show a few lemmas, including \Cref{eigenvalue relationship} that establishes the eigenvalues of $\Omegab$ to be at the level of $\norml{\thetab}^2 \norml{\deltab}^2$,  \Cref{spectral norm gap} that bounds the distance between the random adjacency matrix $\A$ and its expected version $\Omegab$, and \Cref{constant lemma} that lower bounds  $\lambda_{1}(\Sb^T\Sb) -\lambda_2(\Sb^T\Sb)$ away from zero. Combining all these lemmas, we apply Davis-Kahan Theorem (\Cref{Davis_Kahan_theorem}) to establish \Cref{distance between singular vector V}.

Now, we formally state the lemmas mentioned above and relegate their proofs to \Cref{proof_of_lemma_of_DSCORE}.

\begin{Lemma}                     \label{eigenvalue relationship}
	Under Directed-DCBM, for $ 1 \leqslant i \leqslant K$, we obtain
	\begin{equation}
	\lambda_i(\Omegab^T \Omegab) \asymp \norml{\thetab}^2 \norml{\deltab}^2.
	\end{equation}
\end{Lemma}
\begin{proof}
	The proof can be found in \Cref{proof_eigenvalue relationship}.
\end{proof}
\begin{Lemma}                                 \label{spectral norm gap}
	For sufficiently large $n$, with probability at least $1 - o(n^{-4})$,
	\begin{equation}
	\norml{\A - \Omegab} \leqslant 6 \sqrt{\log(n) Z}.
	\end{equation}
\end{Lemma}
\begin{proof}
	The proof can be found in \Cref{proof_spectral norm gap}.
\end{proof}
\begin{Lemma}                              \label{constant lemma}
	With  $\Sb=\Y \bm{\Lambda_S} \Hb^T$, for $i = 1, \cdots, K$, we have
	\begin{flalign}
	0 < C \leqslant \Hb_1(i)  \leqslant 1  \quad &\text{and} \quad     	0 < C \leqslant \Y_1(i)  \leqslant 1, \label{constant H}\\
	\lambda_{1}(\Sb^T\Sb)-\lambda_2(\Sb^T\Sb) &\geqslant C,    \label{constant gap of s} \\
	\V_1(i) > 0,\U_1(i) > 0   &\quad \text{for} \quad  1 \leqslant i \leqslant n  \label{big_0}.
	\end{flalign}
\end{Lemma}
\begin{proof}
	The proof can be found in \Cref{proof_constant lemma}.
\end{proof}
From \cref{big_0}, we observe that the singular vector corresponding to the largest singular value of $\Omegab$ has all positive entries. Thus, we can use it as the denominator to generate ratio matrix.

The following lemma is a variant of Davis-Kahan theorem.
\begin{Lemma}[\cite{Yu2015}, Theorem 2] \label{Davis_Kahan_theorem}
	Let $\A,\hat{\A} \in R^{n \times n}$ be symmetric, with eigenvalues $\lambda_1 \geqslant \cdots \geqslant    \lambda_n$ and  $\hat{\lambda}_1 \geqslant \cdots \geqslant    \hat{\lambda}_n$ and corresponding eigenvectors $\vb_1,\cdots, \vb_n$ and $\hat{\vb}_1,  \cdots, \hat{\vb}_n$, respectively. Fix
	$1 \leqslant r \leqslant s \leqslant n$ and assume that $\min(\lambda_{r-1} - \lambda_r, \lambda_s - \lambda_{s+1}) \geqslant 0$, where we define $\lambda_0 = \infty$ and $\lambda_{n+1} = -\infty$. Let $k = s - r+1$, $\V = (\vb_r, \vb_{(r+1)}, \cdots, \vb_s) \in R^{n \times k}$  and $\hat{\V} = (\hat{\vb}_r, \hat{\vb}_{(r+1)}, \cdots, \hat{\vb}_s) \in R^{n \times k}$. Then there exists an orthogonal matrix $O \in R^{k \times k}$ such that
	\begin{flalign}
	\norml{\V\bm{O} - \hat{\V}} \leqslant \frac{2^{\frac{3}{2}} k^{\frac{1}{2}} \norml{\A - \hat{\A}}}{\min(\lambda_{r-1} - \lambda_r, \lambda_s - \lambda_{s+1})}.
	\end{flalign}
\end{Lemma}

Now, we are ready to prove \Cref{distance between singular vector V}.	
\begin{proof}[Proof of \Cref{distance between singular vector V}]
	First, we derive
	\begin{flalign*}
	\norml{\X^T\X -\Omegab^T \Omegab} & \leqslant \norml{\X^T\X -\X^T\Omegab} + \norml{\X^T\Omegab -\Omegab^T\Omegab} \\
	& \leqslant \norml{\X}\norml{\X-\Omegab} + \norml{\X- \Omegab} \norml{\Omegab} \\
	& \leqslant \norml{\X - \Omegab} (\norml{\X} + \norml{\Omegab}) \\
	& \numleqslant{i} C \sqrt{\log(n) Z} (2 \norml {\Omegab}  + 6 \sqrt{\log(n) Z}) \\
	& \numleqslant{ii} C_1 \sqrt{\log(n) Z}\norml{\thetab} \norml{\deltab}   + C_2 \log (n) Z, 	\numberthis \label{XX-OO}
	\end{flalign*}
	where (i) follows from \Cref{spectral norm gap}, which shows that $\norml{\X-\Omegab} \leqslant 6 \sqrt{\log(n) Z}$, and hence we have $ \norml{\X} \leqslant \norml{\Omegab} +  6 \sqrt{\log(n) Z}$,
	and (ii) follows from  \Cref{eigenvalue relationship}, which implies  $ \norml{\Omegab} = \sqrt{\lambda_1 (\Omegab^T \Omegab )}\asymp \norml{\thetab} \norml{\deltab}$.
	
	Applying \Cref{Davis_Kahan_theorem} (Davis-Kahan theorem), we obtain
	\begin{flalign*}
	\norml{\hat{\V_1} -  \V_1C_V}_F &\leqslant \frac{C \norml{\X^T\X -\Omegab^T \Omegab}}{\lambda_1 (\Omegab^T \Omegab) - \lambda_2(\Omegab^T \Omegab)}     \\
	&\numleqslant{i} \frac{ C_1 \sqrt{\log(n) Z}\norml{\thetab} \norml{\deltab}   + C_2 \log (n) Z }{\lambda_1 (\Omegab^T \Omegab) - \lambda_2(\Omegab^T \Omegab)} \\
	&\numleqslant{ii} \frac{ C_1 \sqrt{\log(n) Z}\norml{\thetab} \norml{\deltab}   + C_2 \log (n) Z }{C \norml{\thetab}^2\norml{\deltab}^2} \\
	&\leqslant C_1 \frac{\sqrt{\log(n) Z}}{\norml{\thetab} \norml{\deltab}} + C_2  \left(  \frac{\sqrt{ \log(n) Z}}{\norml{\thetab} \norml{\deltab}} \right)^2  \\
	&\numleqslant{iii} C_1 \frac{\sqrt{\log(n) Z}}{\norml{\thetab} \norml{\deltab}},
	\end{flalign*}
	where (i) follows from \cref{XX-OO},  (ii)  follows from \cref{eq:5.1} and \cref{constant gap of s}, which implies that
	\begin{align} \label{prop_2_eigen_gap}
	\lambda_1 (\Omegab^T \Omegab) - \lambda_2(\Omegab^T \Omegab) = \norml{\thetab}^2\norml{\deltab}^2(\lambda_1 (\Sb^T \Sb) - \lambda_2(\Sb^T \Sb)) \geqslant C \norml{\thetab}^2\norml{\deltab}^2,
	\end{align}
	and (iii) follows from \cref{eq:4.6}, which gives $\limn \frac{\sqrt{\log(n)Z}}{\norml{\thetab} \norml{\deltab}} =  0$, and thus on the right hand side of the inequality, the first term   dominates the second term for large $n$.
	
	Similarly, we apply \Cref{Davis_Kahan_theorem} to bound the singular vectors corresponding to the $2$nd to $K$th largest singular values, and have
	\begin{flalign*}
	\norml{\hat{\V}_{2 \sim K} - \V_{2 \sim K}\mathbf{O_V}}_F &\leqslant \frac{C \norml{\X^T\X -\Omegab^T \Omegab}} {\min (\lambda_1 (\Omegab^T \Omegab) - \lambda_2(\Omegab^T \Omegab) , \lambda_K (\Omegab^T \Omegab) - \lambda_{(K+1)}(\Omegab^T \Omegab))}\\
	&\numleqslant{i} \frac{C \norml{\X^T\X -\Omegab^T \Omegab}} {\min (\lambda_1 (\Omegab^T \Omegab) - \lambda_2(\Omegab^T \Omegab) , \lambda_K (\Omegab^T \Omegab))}	 \\
	&\numleqslant{ii} \frac{ C_1 \sqrt{\log(n) Z}\norml{\thetab} \norml{\deltab}   + C_2 \log (n) Z }{C \norml{\thetab}^2\norml{\deltab}^2} \\
	&\leqslant C_1 \frac{\sqrt{\log(n) Z}}{\norml{\thetab} \norml{\deltab}} + C_2  \left(  \frac{\sqrt{ \log(n) Z}}{\norml{\thetab} \norml{\deltab}} \right)^2 \\
	&\numleqslant{iii} C \frac{\sqrt{\log(n) Z}}{\norml{\thetab} \norml{\deltab}} ,
	\end{flalign*}
	where (i) follows from \Cref{lemma5.1}, which implies  $\lambda_{(K+1)}(\Omegab^T \Omegab) = 0$, (ii)  follows from  \Cref{eigenvalue relationship}, \cref{prop_2_eigen_gap,XX-OO}, and (iii) follows from \cref{eq:4.6}, and as we argued above, the first term dominates the second term for large $n$.
	
	Following the proof procedure similar to the above arguments, we can obtain that $\norml{\hat{\U}_1 -  \U_1C_U}  \leqslant C \frac{\sqrt{\log (n) Z}}{\norml{\thetab} \norml{\deltab}}$ and $ \norml{\hat{\U}_{2 \sim K} - \U_{2 \sim K}\mathbf{O_U}}_F \leqslant C \frac{\sqrt{\log (n) Z}}{\norml{\thetab} \norml{\deltab}}$.
\end{proof}

\subsection{Proof of \Cref{k mean gap}} \label{proof_of_proposition_3}
\begin{proof}
	In the following, we deal with row vectors, and the row $\ell_2$-norm. Take two nodes $i$ and $j$ from the graph. Then, by definition, we have
	\begin{flalign*}
	\norml{\R_{\bar{i}} - \R_{\bar{j}} }^2 = \norml{(\R_{\V})_{\bar{i}} - (\R_{\V})_{\bar{j}} }^2 + \norml{(\R_{\U})_{\bar{i}} - (\R_{\U})_{\bar{j}} }^2 .
	\end{flalign*}
	Thus, to prove \Cref{k mean gap}, it is sufficient to show $\norml{(\R_{\V})_{\bar{i}} - (\R_{\V})_{\bar{j}} }^2 \geqslant 2$ and $\norml{(\R_{\U})_{\bar{i}} - (\R_{\U})_{\bar{j}} }^2 \geqslant 2$  for $c_i \neq c_j$, , and $\norml{(\R_{\V})_{\bar{i}} - (\R_{\V})_{\bar{j}} }^2 = 0$ and $\norml{(\R_{\U})_{\bar{i}} - (\R_{\U})_{\bar{j}} }^2 = 0 $ for $c_i = c_j$.
	
	We  first show that these hold for $\norml{(\R_{\V})_{\bar{i}} - (\R_{\V})_{\bar{j}} }^2$. We derive the follow equations.
	\begin{flalign*}
	\norml{(\R_{\V})_{\bar{i}} - (\R_{\V})_{\bar{j}} }^2 &\overset{\text{(i)}}{=}  \normlarge{ \frac{(\mathbf{V}_{2 \sim K} \Ob_\V)_{\bar{i}}}{ C_V \V_{1}(i)} - \frac{(\mathbf{V}_{2 \sim K} \Ob_\V)_{\bar{j}}}{ C_V \V_{1}(j)}}^2 \\
	&\overset{\text{(ii)}}{=}   \normlarge{ \frac{(\mathbf{V}_{2 \sim K})_{\bar{i}} }{  \V_{1}(i)} - \frac{(\mathbf{V}_{2 \sim K})_{\bar{j}} }{  \V_{1}(j)}}^2 \\
	&\overset{\text{(iii)}}{=}   \normlarge{ \frac{(\mathbf{V}_{2 \sim K})_{\bar{i}} }{  \V_{1}(i)} - \frac{(\mathbf{V}_{2 \sim K})_{\bar{j}}}{  \V_{1}(j)}}^2
	+ \normlarge{\frac{\V_1(i)}{\V_1(i)} - \frac{\V_1(j)}{\V_1(j)}}^2  \\
	&=  \normlarge{\frac{\V_{\bar{i}}}{\V_1(i)} - \frac{\V_{\bar{j}}}{\V_1(j)}}^2 \\
	&\numequ{iv}  \normlarge{\frac{\frac {\deltab (i)}{\norm {\deltab^{(c_i)}}} \Hb_{\bar{c_i}} }{\frac {\deltab (i)}{\norm {\deltab^{(c_i)}}} \Hb_{\bar{c_i}}(1) } - \frac{\frac {\deltab (j)}{\norm {\deltab^{(c_j)}}} \Hb_{\bar{c_i}} }{\frac {\deltab (j)}{\norm {\deltab^{(c_j)}}} \Hb_{\bar{c_j}}(1) }}^2   \\ 
	&= \normlarge{\frac{\Hb_{\bar{c_i}}}{\Hb_{\bar{c_i}}(1)} - \frac{\Hb_{\bar{c_j}}}{\Hb_{\bar{c_j}}(1)}}^2, \numberthis \label{VCI-VCJ}
	\end{flalign*}
	where (i) follows from the definition of $(R_{V})_{\bar{i}}$ ( see \cref{definition_R_V_U}),
	(ii) follows from \Cref{distance between singular vector V}, which gives $|C_V| = |C_U| = 1$, and $\Ob_\V$ and $O_U$ are orthogonal matrices,
	(iii) follows because the second term equals $0$,
	and (iv) follows from \cref{row represent U}, which shows $\V_{\bar{i}}=\frac {\deltab (i)}{\norm {\deltab^{(c_i)}}} \Hb_{\bar{c_i}}$.
	
	Thus, \cref{{VCI-VCJ}} implies that $\norml{(\R_{\V})_{\bar{i}} - (\R_{\V})_{\bar{j}} }^2=0$ if $c_i = c_j$. Otherwise, if $c_i \neq c_j$, we have
	\begin{flalign*} \label{row_different}
	\norml{(\R_{\V})_{\bar{i}} - (\R_{\V})_{\bar{j}} }^2 &\numequ{i} \normlarge{\frac{\Hb_{\bar{c_i}}}{\Hb_{\bar{c_i}}(1)} - \frac{\Hb_{\bar{c_j}}}{\Hb_{\bar{c_j}}(1)}}^2 \\
	&= \normlarge{\frac{\Hb_{\bar{c_i}}}{\Hb_{\bar{c_i}}(1)}}^2 + \normlarge{\frac{\Hb_{\bar{c_j}}}{\Hb_{\bar{c_j}}(1)}}^2 - 2 \left\langle\frac{\Hb_{\bar{c_i}}}{\Hb_{\bar{c_i}}(1)},\frac{\Hb_{\bar{c_j}}}{\Hb_{\bar{c_j}}(1)}\right\rangle  \\
	&\numequ{ii} \frac{1}{|\Hb_{\bar{c_i}}(1)|} + \frac{1}{|\Hb_{\bar{c_j}}(1)|} \\
	&\overset{\text{(iii)}}{\geqslant} 1 +1 \\
	&= 2,        \numberthis
	\end{flalign*}
	where (i) follows from \cref{VCI-VCJ},
	(ii) follows from  \cref{row represent U}, which shows that $\Hb$ is an orthogonal matrix, and thus $\norml{\Hb_{c_i}} = 1$, and the rows of $\Hb$ are also orthogonal to each other, i.e., $\langle \Hb_{\bar{i}},\Hb_{\bar{j}} \rangle = 0$,  for $i \neq j$ , (iii) follows from   \cref{constant H},  which shows $ \Hb_{\bar{c_i}}(1) \leqslant 1$.
	
	The  inequality $\norml{(\R_{\U})_{\bar{i}} - (\R_{\U})_{\bar{j}} }^2 \geqslant 2$ for $c_i \neq c_j $, and otherwise equals $0$ can be shown in similar way, which completes the proof of \Cref{k mean gap}.
\end{proof}

\subsection{Proof of \Cref{R gap}} \label{proof_of_proposition_4}

To prove \Cref{R gap}, we first establish \Cref{ill behaviour} to bound the  number of  ill-behavior nodes and a technical inequality in  \Cref{SCORE_inequality}.

First, for a constant $ 0 < C < 1 $, we define
\begin{flalign*}
\hat{S}_V  \equiv \left(
1 \leqslant i \leqslant n;\  \abs { \frac{\hat{\V}_1(i)}{C_V \V_1(i)} - 1} \leqslant C\right)  \text{ and }
\hat{S}_U  \equiv \left(
1 \leqslant i \leqslant n;\  \abs { \frac{\hat{\U}_1(i)}{C_U \U_1(i)} - 1} \leqslant C \right). \numberthis \label{ill-node}
\end{flalign*}

Then, we bound the number of nodes that outside $\hat{S}_V$ and $ \hat{S}_U $ in \Cref{ill behaviour}.
\begin{Lemma}                  \label{ill behaviour}
	For nodes in $ \hat{S}_V$ or $ \hat{S}_U$, the following equations hold
	\begin{flalign}
	\abs{\hat{\V}_1(i)} \asymp \abs{C_V \V_1(i)} \asymp  {\frac {\deltab (i)}{\norm {\deltab}}} \quad \text{for} \quad i \in \hat{S}_V ,\\
	\abs{\hat{\U}_1(i)} \asymp \abs{C_U \U_1(i)} \asymp  {\frac {\thetab (i)}{\norm {\thetab}}} \quad \text{for} \quad  i \in \hat{S}_U.
	\end{flalign}
	Furthermore, with probability at least $1 - O(n^{-4})$, the cardinality of $\mathcal{V}\backslash \hat{S}_V$ and $\mathcal{V}\backslash \hat{S}_U$ satisfy
	\begin{flalign}
	\abs{\mathcal{V} \backslash \hat{S}_V} \leqslant \frac{C \log (n) Z}{\norml{\thetab}^2 \delta_{\min}^2}  \quad   \text{ and } \quad
	\abs{\mathcal{V} \backslash \hat{S}_U} \leqslant \frac{C \log (n) Z}{\norml{\deltab}^2 \theta_{\min}^2}.
	\end{flalign}
\end{Lemma}
\begin{proof}
	The proof can be found in \Cref{proof_ill behaviour}.
\end{proof}
Then, we provide a technical inequality in \Cref{SCORE_inequality}.
\begin{Lemma}  \label{SCORE_inequality}
	For $\vb ,\ub \in \mathbf{R}^n$, $a, b \in \mathbf{R}, a >0, b>0$, the following inequality holds,
	\begin{flalign*}
	\left\Vert \frac{\vb}{a} - \frac{\ub}{b}\right\Vert^2 \leqslant 2  \left(\frac{1}{a^2} \norml{\vb - \ub}^2 + \frac{(b - a)^2}{(ab)^2}\norml{\ub}^2 \right).
	\end{flalign*}
\end{Lemma}
\begin{proof}
	The proof can be found in \Cref{proof_SCORE_inequality}.
\end{proof}
Now we are ready to proof the proposition.
\begin{proof}[Proof of \Cref{R gap}]
	Note that
	\begin{flalign*}
	\norml{\R^* - \R}_F^2 = \norml{\R_{\hat{\V}}^* - \R_{\V}}_F^2 + \norml{\R_{\hat{\U}}^* - \R_{\U}}_F^2.
	\end{flalign*}
	It is sufficient to prove $\norml{\R_{\hat{\V}}^* - \R_{\V}}_F^2\leqslant    C\frac{ T_n^2 \log (n) Z}{\deltab_{min}^2 \norml{\thetab}^2 } $ and  $\norml{R_{\hat{U}}^* - R_{U}}_F^2\leqslant    C\frac{ T_n^2 \log (n) Z}{\thetab_{min}^2 \norml{\deltab}^2 }$, and then combining these two inequalities, we establish the proposition.  We first prove  $\norml{\R_{\hat{\V}}^* - \R_{\V}}_F^2\leqslant    C\frac{ T_n^2 \log (n) Z}{\deltab_{min}^2 \norml{\thetab}^2 } $,  and the   latter one can be shown similarly.
	
	First we show  $\normsquare{(\mathbf{V}_{2 \sim K}\Ob_\V)_{\bar{i}}}  \leqslant C \frac{\deltab^2 (i) }{\norm {\deltab}^2}$.   Note that
	\begin{flalign*}     \label{bound_for_V2KOV}
	\norml{(\mathbf{V}_{2 \sim K}\Ob_\V)_{\bar{i}}}^2 &= \norml{(\mathbf{V}_{2 \sim K})_{\bar{i}}\Ob_\V}^2 = \norml{(\mathbf{V}_{2 \sim K})_{\bar{i}}}^2  \leqslant \norml{\V_{\bar{i}}}^2  \numleqslant{i}  C\frac {\deltab^2  (i)}{\norm {\deltab}^2},  \numberthis
	\end{flalign*}
	where (i) follows from \cref{bound_VU}.
	
	Next we prove $\normsquare{(\R_\V)_{\bar{i}} } \leqslant C$. By definition of $\R_\V$, we have
	\begin{flalign*}
	\normsquare{(\R_\V)_{\bar{i}} }= &\left\Vert{\frac{(\mathbf{V}_{2 \sim K}\Ob_\V)_{\bar{i}}}{C_V \V_1(i)}}\right\Vert^2  \numleqslant{i}  \frac{C\frac {\deltab (i)^2}{\norm {\deltab}^2}}{|C_V\V_1(i)|^2} \\
	&\numleqslant{ii} \frac{C\frac {\deltab (i)^2}{\norm {\deltab}^2}}{\frac {\deltab (i)^2}{\norm {\deltab^{(c_i)}}^2}|H_1(c_i)|^2}  \numleqslant{iii} C , \numberthis    \label{R_V_constant}
	\end{flalign*}
	where (i)  follows from \cref{bound_for_V2KOV},
	(ii) follows from \cref{row represent U} which implies $V_{1}(i)=\frac {\deltab (i)}{\norm {\deltab^{(c_i)}}} \Hb_{1}(c_i)$, and (iii) follows from \cref{eq:4.4} and \Cref{constant lemma}, which implies $\Hb_1(c_i) \geqslant C > 0$.
	
	In order to prove $\norml{\R_{\hat{\V}}  - \R_{\V}}_F^2\leqslant    C\frac{ T_n^2 \log (n) Z}{\deltab_{min}^2 \norml{\thetab}^2 }$, we divide the sum into following two parts:
	\begin{flalign*}
	\norml{\R_{\hat{\V}}  - \R_{\V}}_F^2 = \sum_{i \in (\mathcal{V} \backslash \hat{S}_V)} \norml{(\R_{\hat{\V}} )_{\bar{i}} - (\R_{\V})_{\bar{i}}}^2
	+\sum_{i \in \hat{S}_V} \norml{(\R_{\hat{\V}} )_{\bar{i}} - (\R_{\V})_{\bar{i}}}^2 . \numberthis \label{new_edit_0}
	\end{flalign*}
	For the first term, we have
	\begin{flalign*} \label{sum_ill}
	\sum_{i \in (\mathcal{V} \backslash \hat{S}_V)} \norml{(\R_{\hat{\V}} )_{\bar{i}} - (\R_{\V})_{\bar{i}}}^2    &\leqslant  C \sum_{i \in (\mathcal{V} \backslash \hat{S}_V)} (\norml{(\R_{\hat{\V}} )_{\bar{i}} }^2 + \norml{(\R_{\V})_{\bar{i}}}^2)\\
	&\numleqslant{i} C \sum_{i \in (\mathcal{V} \backslash \hat{S}_V)} ( KT_n^2 + C) \\
	&\numleqslant{ii} C |\mathcal{V} \backslash \hat{S}_V| T_n^2 \\
	&\numleqslant{iii}  \frac{C T_n^2 \log (n) Z}{\norml{\thetab}^2 \delta_{\min}^2},  \numberthis
	\end{flalign*}
	where (i) follows from  \cref{ratio_matrix,R_V_constant}, (ii) follows from the fact that $T_n$  scales with $n$, and thus $T_n$ dominates $C$ for sufficient large $n$, and (iii) follows from \Cref{ill behaviour}.
	
	For the second term in \cref{new_edit_0}, we have
	\begin{flalign*} \label{sum_well}
	\sum_{i \in \hat{S}_V} &\normlarge{(\R_{\hat{\V}} )_{\bar{i}} - (\R_{\V})_{\bar{i}}}^2 \\
	& \overset{(\text{i})}{\leqslant}  C \sum_{i \in  \hat{S}_V} \normlarge{\frac{(\hat{\V}_{2 \sim K})_{\bar{i}}}{\hat{\V}_1(i)} - \frac{(\mathbf{V}_{2 \sim K}\Ob_\V)_{\bar{i}}}{C_V\V_1(i)}}^2\\
	&\overset{(\text{ii})}{\leqslant}  C\sum_{i \in \hat{S}_V} \left(\frac{1}{(\hat{\V}_1(i))^2}\norml{(\hat{\V}_{2 \sim K})_{\bar{i}} - (\mathbf{V}_{2 \sim K}\Ob_\V)_{\bar{i}}}^2 + \frac{(C_V\V_1(i) -\hat{\V}_1(i) )^2}{(\hat{\V}_1(i) C_V\V_1(i))^2}\norml{(\mathbf{V}_{2 \sim K}\Ob_\V)_{\bar{i}}}^2\right)\\
	&\overset{(\text{iii})}{\leqslant}  C\sum_{i \in \hat{S}_V} \left( \frac{\norml{\deltab}^2}{\deltab{(i)}^2}\norml{(\hat{\V}_{2 \sim K})_{\bar{i}} - (\mathbf{V}_{2 \sim K}\Ob_\V)_{\bar{i}}}^2 + \frac{\norml{\deltab}^2}{\deltab{(i)}^2}(C_V\V_1(i) -\hat{\V}_1(i) )^2 \right) \\
	&\leqslant  C\frac{\norml{\deltab}^2}{\delta_{\min}^2} \sum_{i \in \hat{S}_V} \left( \norml{(\hat{\V}_{2 \sim K})_{\bar{i}} - (\mathbf{V}_{2 \sim K}\Ob_\V)_{\bar{i}}}^2 +  \sum_{i \in \hat{S}_V}(C_V\V_1(i) -\hat{\V}_1(i) )^2 \right)\\
	&\leqslant  C\frac{\norml{\deltab}^2}{\delta_{\min}^2} \left( \norml{\hat{\V}_{2 \sim K} -\mathbf{V}_{2 \sim K}\Ob_\V }_F^2 + \norml{\hat{\V}_{1} -\V_{1}C_V }^2 \right)\\
	&\overset{(\text{iv})}{\leqslant} \frac{C\log (n) Z}{\norml{\thetab}^2 \delta_{\min}^2},  \numberthis
	\end{flalign*}
	where (i) follows from the fact that $|(\R_{\V})_{\bar{i}}| \leqslant C$ (see \cref{R_V_constant}), and $T_n$ scales with $n$, which implies $T_n \geqslant C \geqslant |(\R_{\V})_{\bar{i}}|$ for $n$ large enough. Thus, although \cref{ratio_matrix} shows that $(\R_{\hat{\V}} )_{\bar{i}}$ is truncated by $T_n$, we still have $\norml{(\R_{\hat{\V}} )_{\bar{i}} - (\R_{\V})_{\bar{i}}}^2 \leqslant \normlarge{\frac{(\hat{\V}_{2 \sim K})_{\bar{i}}}{\hat{\V}_1(i)} - \frac{(\mathbf{V}_{2 \sim K}\Ob_\V)_{\bar{i}}}{C_V\V_1(i)}}^2$ for large $n$,  (ii) follows from \Cref{SCORE_inequality},  (iii) follows from \Cref{ill behaviour}  and \cref{bound_for_V2KOV}, and (iv) follows from \Cref{distance between singular vector V}.

	Combining \cref{sum_ill,sum_well}, we obtain $\norml{\R_{\hat{\V}}  - \R_{\V}}_F^2\leqslant    C\frac{ T_n^2 \log (n) Z}{\deltab_{min}^2 \norml{\thetab}^2 }$.  
\end{proof}

\subsection{Proof of \Cref{M* R gap}}\label{proof_proposition5}
\label{prove_of_lemma_3_1}
\begin{proof}
	Recall that $\M^*$ is defined as
	\begin{flalign*}
	\M^*= \argmin \limits_{\M \in \boldsymbol{M}_{n,2K-2,K}} \norm{\M- \hat \R}_F^2,
	\end{flalign*}
	where $\boldsymbol{M}_{n,2K-2,K} $ denotes the set of $n\times{(2K-2)}$ matrices with only $K$ different rows. Note that $R$ is also in  $ \boldsymbol{M}_{n,2K-2,K}$. Thus,
	\begin{align} \label{inner_use}
	\norml{\M^* - \hat{\R}} \leqslant \norml{\R - \hat{\R}}.
	\end{align}
	Then, we obtain
	\begin{flalign*}
	\norml{\M^{*} -\ R}_F^2 &\leqslant \norml{\M^* - \hat{\R} + \hat{\R} - \R}_F^2 \\
	&\leqslant C\norml{\M^* - \hat{\R}}_F^2 + C\norml{\hat{\R} - \R}_F^2\\
	&\numleqslant{i} C\norml{\R - \hat{\R}}_F^2 + C\norml{\hat{\R} - \R}_F^2\\
	&\leqslant C\norml{\R - \hat{\R}}_F^2 \\
	&\numleqslant{ii} CT_n^2\log(n)err_n,
	\end{flalign*}
	where (i) follows from \cref{inner_use},  and (ii) follows from \Cref{R gap}.
	
\end{proof}

%
%
\subsection{Proof of \Cref{mis clustering nodes}}\label{prove_of_theorem_1}
\begin{proof}
	First, if nodes $i$, $j$ in set $\mathcal{W}$  are in  different communities, then
	\begin{flalign*}
	\norml{\M^*_{\bar{i}} - \M_{\bar{j}}^*} &= \norml{\M_{\bar{i}}^* - \Rb_{\bar{i}} +  \Rb_{\bar{i}} - \Rb_{\bar{j}} + \Rb_{\bar{j}} -\M_{\bar{j}}^*} \\
	&\geqslant \norml{\Rb_{\bar{i}} - \Rb_{\bar{j}}} - \norml{\M_{\bar{i}}^* - \Rb_{\bar{i}} + \Rb_{\bar{j}} -\M_{\bar{j}}^*} \\
	&\geqslant  \norml{\Rb_{\bar{i}} - \Rb_{\bar{j}}} - \norml{\M_{\bar{i}}^* - \Rb_{\bar{i}}} - \norml{ \M_{\bar{j}}^* - \Rb_{\bar{j}} } .
	\end{flalign*}
	By \Cref{k mean gap}, i.e.,  $\norml{\Rb_{\bar{i}} - \Rb_{\bar{j}} } \geqslant 2$, and the definition of set $\mathcal{W}$ in \Cref{mis clustering nodes}.   We obtain that for  $i,j \in \mathcal{W}$,
	\begin{flalign*}
	\norml{\M_{\bar{i}}^* - \M_{\bar{j}}^*} &\geqslant (2 - 1) = 1.
	\end{flalign*}
	Thus, if nodes $i$, $j$ are in  different communities, then their corresponding rows in $\M^*$ are sufficiently different. By the assumption $|\mathcal{V} \backslash \mathcal{W}| < \min{\{n_1, n_2 \cdots n_K \}}$,  $\mathcal{W}$ contains at least one node in each community. Combining there two facts and  the  definition that $\M^*$ has only $K$ different rows,  we conclude that the corresponding rows in $\M^*$ of nodes in the same community are same. In conclusion, if two nodes in $\mathcal{W}$ are in the same community, then their corresponding rows in $\M^*$ are the same.  Otherwise, their  corresponding rows in $\M^*$ are sufficiently different. Thus, nodes in $\mathcal{W}$ are correctly clustered. Then, the definition of $\mathcal{W}$ and \Cref{M* R gap}  directly imply
	\begin{equation}
	|\mathcal{V} \backslash \mathcal{W}| \leqslant CT_n^2\log(n) err_n.   \label{final_answer}
	\end{equation}
	
	
	
\end{proof}

\subsection{Proof of Lemmas for D-SCORE} \label{proof_of_lemma_of_DSCORE}
\subsubsection{Proof of \Cref{eigenvalue relationship}} \label{proof_eigenvalue relationship}
\begin{proof}
	Following \cref{eq:5.1}, we obtain that $\sigma_i(\Omegab) =  \norml{\thetab} \norml{\deltab} \sigma_i(\Sb)$, for $1 \leqslant i \leqslant K$, it is sufficient to show $0 < C_1 \leqslant \sigma_{K}(\Sb) \leqslant  \sigma_1 (\Sb) \leqslant C_2$.
	
	Recall $\Sb=\Psib_{\thetab}\B\Psib_{\deltab}^T$, where $\Psib_{\thetab}$ and $\Psib_{\deltab}$ are diagonal matrices, and $\norml{\Psib_{\thetab}}$ and  thus $\norml{\Psib_{\thetab}}_{\min}$  correspond to the largest and smallest absolute value of the diagonal entries of $\Psib_{\thetab}$, respectively. Following from \cref{eq:4.4}, we have
	\begin{flalign*}
	\norml{\Psib_{\thetab}} &= \max_{i} \Psib_{\thetab}(i,i) = \max_{i} \frac{\norml{\thetab^{(i)}}}{\norml{\thetab}} \leqslant C, \\
	\norml{\Psib_{\thetab}}_{\min} &= \min_{i} \Psib_{\thetab}(i,i) = \min_{i} \frac{\norml{\thetab^{(i)}}}{\norml{\thetab}} \geqslant C.
	\end{flalign*}
	Therefore,  there exist two constants $C_m > 0$ and $C_M > 0$ such that 
	\begin{flalign*}
	C_m \leqslant \norml{\Psib_{\thetab}}_{\min} \leqslant	\norml{\Psib_{\thetab}} \leqslant C_M.  \numberthis \label{eig_constant_heterogeneous_spectral_norm}
	\end{flalign*} 
	It can be similarly shown that
	\begin{flalign*}
	C_m \leqslant \norml{\Psib_{\deltab}}_{\min} \leqslant	\norml{\Psib_{\deltab}} \leqslant C_M .
	\end{flalign*}
	
	By the definition of $\Sb \equiv \Psib_{\thetab}\B\Psib_{\deltab}^T$, we have  
	\begin{flalign*}
	\sigma_{1}(\Sb) = \norml{\Sb} &= \norml{\Psib_{\thetab}\B\Psib_{\deltab}^T}  \leqslant \norml{\Psib_{\thetab}} \norml{\B} \norml{\Psib_{\deltab}^T}  \numleqslant{i} C,  \numberthis \label{eig_constant_big}
	\end{flalign*}
	where (i) follows from   \cref{eig_constant_heterogeneous_spectral_norm} and because $\B$ is a constant matrix, i.e., $\B$ does change with $n$,  so that there exists a constant $C$, such that $\norml{\B} \leqslant C$. On the other hand,
	\begin{flalign*}
	\sigma_{K}(\Sb) = \norml{\Sb}_{\min} &= \norml{\Psib_{\thetab}\B\Psib_{\deltab}^T}_{\min}  \geqslant \norml{\Psib_{\thetab}}_{\min} \norml{\B}_{\min} \norml{\Psib_{\deltab}^T}_{\min}  \numgeq{i} C, \numberthis \label{eig_constant_small}
	\end{flalign*}
	where (i) follows from  \cref{eig_constant_heterogeneous_spectral_norm} and the inequality $\norml{\A\B}_{\min} \geqslant \norml{\A}_{\min} \norml{\B}_{\min}$. Also, since  $\B$ is a constant matrix which does change with $n$, and  $B$ is non-singular (see \cref{eq:4.2}),  there exists a constant $C$, such that $\norml{\B}_{\min} \geqslant C >0$.
	
	Combining \cref{eig_constant_big,eig_constant_small}, we obtain $\sigma_i(\Sb) \asymp C$. Then, by \cref{eq:5.1}, we have $\sigma_i^2({\Omegab})  = \sigma_i^2(\Sb) \norml{\thetab}^2 \norml{\deltab}^2$. Hence, for $ 1 \leqslant i \leqslant K$
	\begin{flalign*}
	\lambda_i (\Omegab^T \Omegab )  = \sigma_i^2({\Omegab}) = \sigma_i^2(\Sb) \norml{\thetab}^2 \norml{\deltab}^2  \asymp \norml{\thetab}^2 \norml{\deltab}^2  .
	\end{flalign*}
\end{proof}

\subsubsection{Proof of \Cref{spectral norm gap}}\label{proof_spectral norm gap}
\begin{proof}
	Define $\e_i$ as an $n \times 1$\ vector, where $\e_i(i) = 1$ and $0$ elsewhere. Thus, we can write $\W$ as $\W = \sum_{i,j = 1}^{n} \W(i,j)\e_i\e_j^T$. By the definition that $\W \equiv \A-\Omegab = \A - E  [\A] $, the entry $\W(i,j)$ is an independent centered Bernoulli random  variable. Thus  $\W(i,j)\e_i\e_j^T$ is an independent centered Bernoulli random matrix with the dimension $n \times n$. In fact,  $\W(i,j)\e_i\e_j^T$  is a  matrix with only one nonzero entry $\W(i,j) = \A(i,j) - \thetab(i)\B(c_i,c_j)\deltab(j)$ at the location $(i,j)$.
	
	In order to apply matrix Bernstein inequality, we need to bound the spectral norm  of each summation matrix, and the variance of the entire summation. By the definition of the matrix spectral norm, for $1 \leqslant i,j \leqslant n $, we have
	\begin{flalign*}
	\quad \norml{\W(i,j)\e_i\e_j^T} &=\abs{ \W(i,j)} \norml{\e_i\e_j^T}  = \abs{ \W(i,j)} \sqrt{   \norml{\e_i\e_j^T (\e_i\e_j^T)^T }}  \\
	&=  \abs{ \W(i,j)} \sqrt{\norml{\e_i\e_i^T}} \numequ{i} \abs{ \W(i,j)}    \\
	&\numequ{ii} \abs{\A(i,j)-\Omegab(i,j)} \\
	& \numleqslant{iii} \max \left( \ \abs{0 - \thetab(i) \B(c_i,c_j) \deltab(j)},\abs{1-\thetab(i) \B(c_i,c_j) \deltab(j)} \right)\\
	& \numleqslant{iv} 1,  \numberthis \label{5.21}
	\end{flalign*}
	where (i) follows because $\e_i\e_i^T$ is a diagonal matrix with only one non-zero entry $1$ at location $(i,i)$, thus $\norml{\e_i\e_i^T}=1$, (ii) follows because that $\W = \A -\Omegab$,
	(iii) follows because $\A(i,j)$ is a Bernoulli random variable that it takes the values  $0$ or $1$, and (iv) follows because  $0 < \Omegab(i,j) = \thetab(i) \B(c_i,c_j) \deltab(j) \leqslant 1$.
	
	Next we consider the variance of the random matrix  $V(\W) \equiv \max{ \left( \norml{E(\W\W^T)},\norml{E(\W^T\W)}  \right)}$. We first bound $\norml{E(\W\W^T)}$, and then bound $\norml{E(\W^T\W)} $. Note that
	\begin{flalign*}
	E(\W\W^T) & = E[(\sum_{i,j = 1}^{n} \W(i,j) \e_i\e_j^T) (\sum_{k,l = 1}^{n}  \W(k,l) \e_k\e_l^T)^T] \\
	& = E[\sum_{i,j,k,l =1}^{n}  \W(i,j)\W(k,l)\e_i\e_j^T\e_l\e_k^T] \\
	& \numequ{i} \sum_{i,j,k=1}^{n} E[ \W(i,j)\W(k,j)\e_i\e_k^T]\\
	& \numequ{ii} \sum_{i,j=1}^{n} E[\W^2(i,j)]\e_i\e_i^T,   \numberthis \label{eq:5.22}
	\end{flalign*}
	where (i) follows from the fact that $\e_j^T\e_l = 1$ if $j = l$ and $0$ otherwise,
	and (ii) follows from the fact that  if $i \neq k$, $\W(i,j)$ and $\W(k,j)$ are independent random Bernoulli random variables with the expected value $0$, i.e., $E[\W(i,j)\W(k,j)\e_i\e_k^T)]= E[\W(i,j)]E[\W(k,j)]\e_i\e_k^T=0\times0=0 $. Thus, we only need to consider the case with $i = k$. Observing that  $\Omegab(i,j)=E[\X(i,j)]$ and let $Var(\X(i,j))$ denote the variance of Bernoulli random variable $\X(i,j)$. Then, we obtain
	\begin{flalign*}
	E[\W^2(i,j)] &= E[(\X(i,j)- \Omegab(i,j))^2]\\
	&= E[(\X(i,j)- E[\X(i,j)])^2]	 \\
	&= \text{Var} (\X(i,j))       \\
	&= \thetab(i) \B(c_i,c_j) \deltab(j)  [1 - \thetab(i) \B(c_i,c_j) \deltab(j) ]     \\
	&\leqslant \thetab(i) \B(c_i,c_j) \deltab(j)  \\
	&\leqslant \thetab(i) \deltab(j). \numberthis \label{eq:5.23}
	\end{flalign*}
	By \cref{eq:5.22}, we have
	\begin{flalign*}
	\norml{E[\W\W^T]} &= \normlarge{ \sum_{i,j=1}^{n}  E[\W^2(i,j)]\e_i\e_i^T} \numequ{i} \max_{1 \leqslant i \leqslant n } \abs{\sum_{j=1}^{n}  E[\W^2(i,j)]}    \\
	&\numleqslant{ii} \max_{1 \leqslant i \leqslant n }    \sum_{j=1}^{n} |\thetab(i)\deltab(j)|   \leqslant  \max_{1 \leqslant i \leqslant n } \thetab(i) \norml{\deltab}_1 \\
	&\leqslant   \thetab_{\max}  \norml{\deltab}_1,       \numberthis
	\end{flalign*}
	where (i) follows because $E[\W\W^T] = \sum_{i=1}^{m} \left(\sum_{j=1}^{n} E[\W^2(i,j)]\right) \e_i\e_i^T$ is a diagonal matrix (the spectral norm  of a diagonal matrix is  the maximum absolute value of its diagonal entries), and (ii) follows from \cref{eq:5.23}. 	Following the similar proof procedure, we obtain $ \norml{E[\W^T\W]} \leqslant \deltab_{\max} \norml{\thetab}_1$. Thus, we have
	\begin{flalign*}
	V(\W) &= \max \left( \norml{E[\W\W^T]},\norml{E[\W^T\W]} \right) \\
	&\leqslant \max \left(\thetab_{\max}  \norml{\deltab}_1 ,\deltab_{\max} \norml{\thetab}_1\right)	\\
	&\leqslant \Z.     \numberthis     \label{variance}
	\end{flalign*}
	
	Note that $Z \equiv \Z$, by \cref{variance}, we have $V(\W) \leqslant Z$. Since \cref{5.21} implies that $\norml{\W(i,j)\e_i\e_j^T} $ is  bounded by $1$, and these are also independent centered random matrices,   we apply the asymmetric version of the matrix version of Bernstein inequality (Theorem 1.6.2 in \cite{Tropp2015})  with $t= 6 \sqrt{ \log(n)Z}$ and   $V(\W) \leqslant Z$, and obtain
	\begin{flalign*}
	P(\norml{\W} \geqslant t) &\leqslant 2n \exp (\frac{-\frac{t^2}{2}}{V(\W)+\frac{t}{3}})  \\
	&\leqslant 2n \exp \left(\frac{-18 \log (n) Z}{Z + 2\sqrt{ \log(n)Z}} \right) \\
	&\leqslant 2n \exp \left(\frac{-18 \log (n) } {1 +  2 \sqrt{\frac{ \log(n)}{Z}}  }\right)  \\
	&\leqslant \frac{1}{n^4},
	\end{flalign*}
	where the last inequality follows from \cref{eq:4.7}, which implies that $\sqrt{\frac{\log (n)}{Z}} \leqslant 1$ for sufficiently large $n$.
\end{proof}

\subsubsection{Proof of \Cref{constant lemma}} \label{proof_constant lemma}
\begin{proof}
	We first introduce the following useful lemma,
	\begin{Lemma}[\cite{A.Horn1985}, Theorem 8.4.4] \label{Horn_theorem}
		For every $K \times K$ irreducible, nonnegative, and positive semidefinite  matrix $\mathbf{M}$, let $\V_1$ denote  the eigenvector corresponding to the largest eigenvalue. Then, the following facts hold:
		\begin{itemize}
			\item[(i)] $\V_1$ can be a positive vector.
			\item[(ii)] The largest eigenvalue is an algebraically simple eigenvalue.
		\end{itemize} 
	\end{Lemma}
	
	We note that it is sufficient to prove that $\Sb^T\Sb$ and $\Sb\Sb^T$ are irreducible and nonnegative, and such properties do not change with $n$. Once these facts hold, $(i)$ in \Cref{Horn_theorem} implies $\Hb_1(i) \geqslant C >0$ and $\Y_1(i) \geqslant C >0$, where $\Hb_1$ and $\Y_1$ are the eigenvectors corresponding to the largest eigenvalues of  $\Sb^T\Sb$ and $\Sb\Sb^T$, respectively.  Furthermore,  $\Hb_1(i) \geqslant C >0$ implies $\V_1(i) > 0$ due to   $\V_{\bar{i}}=\frac {\deltab (i)}{\norm {\deltab^{(c_i)}}} \Hb_{\bar{c_i}}$ (see \cref{row represent U}). Similarly, we obtain that $\U_1(i) > 0$ for $1 \leqslant i \leqslant n$. Moreover,   $(ii)$ in  \Cref{Horn_theorem} implies $\lambda_{1}(\Sb^T\Sb)-\lambda_2(\Sb^T\Sb) \geqslant C$.  Then, we complete the proof of \Cref{constant lemma}.
	
	Thus, we next prove that  $\Sb^T\Sb$ and $\Sb\Sb^T$ are irreducible and nonnegative, and such properties do not change with $n$. Recall $\Sb=\s$. By \cref{eig_constant_heterogeneous_spectral_norm} and the definition of the diagonal matrices $\Psib_{\deltab}$ and $ \Psib_{\thetab} $ (\cref{definition_diagonal_Psi}), it is clear that there exist $C_1$ and $C_2$ such that $0 < C_1  \leqslant \Psib_{\thetab}(i,i) \leqslant C_2 $ and $ 0 <  C_1  \leqslant \Psib_{\deltab}(i,i) \leqslant C_2 $. Thus, we have
	\begin{flalign}
	C_1^2 \B(i,j) \leqslant \Sb(i,j) \leqslant C_2^2 \B(i,j), \quad \text{for} \quad 1 \leqslant i,j \leqslant K.
	\end{flalign} 
	Then, for $1 \leqslant i,j \leqslant K$, we obtain, 
	\begin{align*}
	(\Sb^T\Sb)(i,j) &= \sum_{k=1}^{K} \Sb^T(i,k)\Sb(k,j) \leqslant \sum_{k=1}^{K} \Sb^T(i,k)\Sb(k,j) \\
	&\leqslant  C_2^4 \sum_{k=1}^{K} \B^T(i,k)\B(k,j) = C ( \B^T\B)(i,j). \numberthis \label{new_edd_1}
	\end{align*}
	Similarly, for $1 \leqslant i,j \leqslant K$, we obtain 
	\begin{align}
	C (\B^T\B)(i,j) 	\leqslant  (\Sb^T\Sb)(i,j). \label{new_edd_2}
	\end{align}
	Combining \cref{new_edd_1,new_edd_2}, for $1 \leqslant i,j \leqslant K$, we obtain $ C  (\B^T\B)(i,j) 	\leqslant  (\Sb^T\Sb)(i,j) \leqslant  C ( \B^T\B)(i,j)$.
	We further note that $B$ is a constant matrix with positive entries, and $\B^T\B$ is irreducible and nonnegative by \Cref{Assumption}.  Thus, we conclude that $\Sb^T\Sb$ is nonnegative and irreducible, and  these properties do not change with $n$. Similarly, we obtain that $\Sb\Sb^T$   is also nonnegative and irreducible, and  these properties do not change with $n$.  This completes the proof.
\end{proof}

\subsubsection{Proof of \Cref{ill behaviour}} \label{proof_ill behaviour}

We  first prove $ \abs{C_V \V_1(i)} \asymp\abs {\frac {\deltab (i)}{\norm {\deltab}}}$ for  $1 \leqslant i\leqslant n$. By \cref{row represent U},  $\V_{\bar{i}}=\frac {\deltab (i)}{\norm {\deltab^{(c_i)}}} \Hb_{\bar{c_i}} $, and thus $C_V \V_1(i) = C_V \frac {\deltab (i)}{\norm {\deltab^{(c_i)}}}  \Hb_{1}(c_i) $, where $\abs{C_V} = 1$ by \Cref{distance between singular vector V}. Following from \cref{constant H,eq:4.4}, we have
\begin{flalign} \label{first_singularvalue_constant}
\abs{C_V \V_1(i)} \asymp \abs{C_V \frac {\deltab (i)}{\norm {\deltab^{(c_i)}}}  \Hb_{1}(c_i)} \asymp \abs {\frac {\deltab (i)}{\norm {\deltab}}},  \quad \text{for } 1 \leqslant i \leqslant n.
\end{flalign}

Then, by \cref{ill-node}, nodes in the set   $\hat{S}_V$ satisfies $\abs{\frac{\hat{\V}_1(i)}{C_V\V_1(i)} -  1 } \leqslant C <1$, and hence $\abs{\hat{\V}_1(i)} \asymp \abs{C_V \V_1(i)}$, where $\abs{C_V} = 1$ by \Cref{distance between singular vector V}. Then, by \cref{first_singularvalue_constant},   we have for $i \in \hat{S}_V$,
\begin{flalign} \label{asymp_V}
\abs{\hat{\V}_1(i)} \asymp \abs{C_V \V_1(i)}  \asymp \abs {\frac {\deltab (i)}{\norm {\deltab}}}.
\end{flalign}
Similarly, $\abs{\hat{\U}_1(i)} \asymp \abs{C_U \U_1(i)} \asymp \abs {\frac {\thetab (i)}{\norm {\thetab}}}$ for $i \in \hat{S}_U$.

Next, by \cref{first_singularvalue_constant}, $	\abs{C_V \V_1(i)} \asymp \abs {\frac {\deltab (i)}{\norm {\deltab}}} > 0$, and thus having $C_V \V_1(i)$ as denominator for all $1 \leqslant i \leqslant n$ is valid. We further derive
\begin{flalign*}
\sum_{i \in (\mathcal{V} \backslash \hat{S}_V)}    \left(\frac{\hat{\V}_1(i)}{C_V\V_1(i)} -  1 \right)^2  &= \sum_{i \in (\mathcal{V} \backslash \hat{S}_V)}  \left(\frac{1}{C_V\V_1(i)}\right) ^2 (\hat{\V}_1(i) -  C_V\V_1(i))^2 \\
&\numleqslant{i} \sum_{i \in (\mathcal{V} \backslash \hat{S}_V)}    \frac{\norml{\deltab}^2}{\delta_{\min}^2}    (\hat{\V}_1(i) -  C_V\V_1(i))^2 \\
&\leqslant \sum_{i=1}^{n} \frac{\norml{\deltab}^2}{\delta_{\min}^2}    (\hat{\V}_1(i) -  C_V\V_1(i))^2 \\
&\leqslant \frac{\norml{\deltab}^2}{\delta_{\min}^2} \norml{\hat{V}_1 -  \V_1C_V}^2 \\
&\numleqslant{ii} C \frac{ \log (n) Z}{\norml{\thetab}^2 \delta_{\min}^2 }, \numberthis \label{new_edd_4}
\end{flalign*}
where (i) follows from \cref{first_singularvalue_constant}, and (ii) follows from  \Cref{distance between singular vector V}. Since nodes in the set   $ \mathcal{V} \backslash \hat{S}_V$ satisfy $(\frac{\hat{\V}_1(i)}{C_V\V_1(i)} -  1)^2 > C_0^2$, we have 
\begin{align*}
|\mathcal{V} \backslash \hat{S}_V| = \sum_{i \in \mathcal{V} \backslash \hat{S}_V} 1 \leqslant \sum_{i \in \mathcal{V} \backslash \hat{S}_V} \frac{1}{C_0^2} \bigg(\frac{\hat{\V}_1(i)}{C_V\V_1(i)} -  1\bigg)^2 \numleqslant{i} C \frac{ \log (n) Z}{\norml{\thetab}^2 \delta_{\min}^2 }, 
\end{align*}
where (i) follows from \cref{new_edd_4}. Similarly, we can show that
$\abs{\mathcal{V} \backslash \hat{S}_U} \leqslant \frac{C \log (n) Z}{\norml{\deltab}^2 \theta_{\min}^2}$.
\subsubsection{Proof of \Cref{SCORE_inequality}} \label{proof_SCORE_inequality}
\begin{proof}
	We derive the following bound:
	\begin{flalign*}
	\normlarge{\frac{\vb}{a} - \frac{\ub}{b}}^2 &= \normlarge{\frac{b\vb - a\ub}{ab}}^2 \\
	&= \frac{1}{(ab)^2} \norml{b\vb - b\ub +b\ub -a\ub}^2 \\
	&\leqslant \frac{2}{(ab)^2}  \left( \norml{b\vb - b\ub}^2 + \norml{b\ub -a\ub}^2 \right)  \\
	&\leqslant \frac{2}{a^2} \norml{\vb - \ub}^2 + \frac{2(b-a)^2}{(ab)^2}\norml{\ub}^2 \\
	&=2 \left( \frac{1}{a^2} \norml{\vb - \ub}^2 + \frac{(b-a)^2}{(ab)^2}\norml{\ub}^2 \right).
	\end{flalign*}
\end{proof}

%
%
%

%
%
%
\section{Proof of \Cref{rown_mis clustering nodes} (Convergence of DSCORE$_q$)} \label{app:d-scoreq}

To establish the performance guarantee for D-SCORE$_q$, the general idea is similar to that of D-SCORE, but there are technical differences. Hence, the proof here focuses only on these differences. As in \Cref{app:d-score}, we first prove a few propositions, which then lead to the proof of \Cref{rown_mis clustering nodes}. 

We first state the following two lemmas, which are useful in our proof.
\begin{Lemma} \label{inequality}
	For $\x ,\y \in \mathcal{R}^d$ where $d$ is finite, the following inequality holds,
	\begin{flalign*}
	\normlarge{\frac{\x}{\normlq{\x}} - \frac{\y}{\normlq{\y}}} \leqslant C\frac{\norml{\x -\y}}{\min  \left( \normlq{\x}, \normlq{\y}\right) }.
	\end{flalign*}
\end{Lemma}
\begin{proof}
	The proof can be found in  \Cref{proof_inequality}.
\end{proof}
\begin{Lemma}           \label{matrix_analysis_mutual_bound}
	For any vector norm $\norml{\cdot}$ in the finite dimensional space, it can be bounded by its $l_2$-norm, i.e., there exists two constants $0 < C_1 \leqslant C_2$, such that for all $x$ in the finite dimensional space, we have
	\begin{flalign}
	C_1 \norml{\x}_2	\leqslant \norml{\x} \leqslant C_2 \norml{\x}_2.
	\end{flalign}
\end{Lemma}
\begin{proof}
	The proof follows directly from Corollary 5.4.5 in \cite{A.Horn1985}.
\end{proof}


We also note that \Cref{lemma5.1} on the property of the expected adjacency matrix $\Omega$ also holds here and is very useful for the analysis of D-SCORE$_q$.

\subsection{\Cref{rown_distance_of_singular_vector} and its Proof}

In parallel to  \Cref{distance between singular vector V} for D-SCORE, we bound the distance between the singular vector matrices $\hat{\Ub}$ and $\hat{\Vb}$ of $\A$ and the singular vector matrices $\Ub$ and $\Vb$ of $\Omegab$. However, for D-SCORE, we need to develop the bound for the first singular vectors and the $2$nd to $K$th singular vectors separately, whereas for D-SCORE$_q$ we need only to develop the bound for the entire singular vector matrices. In the following proposition, we adapt the same notation for the singular vector matrices of $\Omegab$ and $\A$ as in \Cref{distance between singular vector V}.

\begin{Proposition}  \label{rown_distance_of_singular_vector}
	There exist two orthogonal matrices $\Ob_\Vb$ and $\Ob_\Ub$, such that for $n$ large enough, with probability at least $1 - o(n^{-4})$ ,
	\begin{flalign}
	\norml{\hat{\Vb} - \Vb\Ob_\Vb}_F &\leqslant  C \frac{\sqrt{\log (n) Z}}{\norml{\thetab} \norml{\deltab}},  \qquad
	\norml{\hat{\Ub} - \Vb\Ob_\Ub}_F \leqslant  C \frac{\sqrt{\log (n) Z}}{\norml{\thetab} \norml{\deltab}}.
	\end{flalign}
\end{Proposition}
\begin{proof}
	The proof follows in the same manner as that of \Cref{distance between singular vector V} for D-SCORE, based on the direct application of Davis-Kehan inequality. 
\end{proof}

\subsection{\Cref{Row_Normalization_K_different_row} and its Proof}\label{proof_of_proposition_6}

The central difference between D-SCORE$_q$ and D-SCORE lies in the way that they eliminate the heterogeneous parameters before clustering. D-SCORE divides each row of the singular vector matrices by its first entry to eliminate the heterogeneous parameters, whereas D-SCORE$_q$ divides each row by its corresponding $\ell_q$ norm. Then, in parallel to  \Cref{k mean gap} for D-SCORE,  we provide \Cref{Row_Normalization_K_different_row} as follows, which characterizes the properties of the ratio matrix $\Rb \equiv [\Rb_\Vb,\Rb_\Ub]$.
\begin{Proposition}\label{Row_Normalization_K_different_row}
	For the ratio matrix  $\Rb = [\Rb_\Vb, \Rb_\Ub]$ generated by the singular vectors of the matrix $\Omegab$, and for $ 1 \leqslant i \leqslant n$  and  $1 \leqslant j \leqslant n$, the following inequalities hold:
	\begin{flalign*}
	\norml{\Rb_{\bar{i}} - \Rb_{\bar{j}}}^2 &= 0 \quad \text{if} \quad c_i = c_j, \quad \text{and} \quad
	\norml{\Rb_{\bar{i}} - \Rb_{\bar{j}}}^2 \geqslant C > 0 \quad \text{if} \quad c_i \neq c_j.
	\end{flalign*}
\end{Proposition}

\Cref{Row_Normalization_K_different_row} states that if nodes $i$ and $j$ are in the same community,  i.e., $c_i = c_j$, then their corresponding rows in the ratio matrix $\Rb$ are same; otherwise their corresponding rows in $\Rb$ are different.  This  property justifies why $\Rb$ is used for clustering.
\begin{proof}
	First, we have
	\begin{flalign*}
	\norml{\R_{\bar{i}} - \R_{\bar{j}}}^2  = \norml{(\R_\V)_{\bar{i}} - (\R_\V)_{\bar{j}}}^2 + \norml{(\R_\U)_{\bar{i}} - (\R_\U)_{\bar{j}}}^2 .
	\end{flalign*}
	For the first term $\norml{(\R_\V)_{\bar{i}} - (\R_\V)_{\bar{j}}}^2$, by \cref{meaning_of_Row_Normalization} which shows $(\R_\V)_{\bar{i}} = \frac{ \Hb_{\bar{c_i}}\Ob_\V}{\normlq{\Hb_{\bar{c_i}}\Ob_\V}} $, the following equation holds,
	\begin{flalign*}
	\norml{(\R_\V)_{\bar{i}} - (\R_\V)_{\bar{j}}}^2  = \normlarge{\frac{ \Hb_{\bar{c_i}}\Ob_\V}{\normlq{\Hb_{\bar{c_i}}\Ob_\V}} -  \frac{ \Hb_{\bar{c_j}}\Ob_\V}{\normlq{\Hb_{\bar{c_j}}\Ob_\V}}}^2	.	
	\end{flalign*}
	If $c_i = c_j$, i.e., node $i$ and $j$ are in the same community, and then
	\begin{flalign}
	\norml{(\R_\V)_{\bar{i}} - (\R_\V)_{\bar{j}}}^2  = 0. \label{distance_1}
	\end{flalign}
	Otherwise, if $c_i \neq c_j$,  we have
	\begin{flalign*}
	\norml{(\R_\V)_{\bar{i}} - (\R_\V)_{\bar{j}}}^2  &= \normlarge{\frac{ \Hb_{\bar{c_i}}\Ob_\V}{\normlq{\Hb_{\bar{c_i}}\Ob_\V}} -  \frac{ \Hb_{\bar{c_j}}\Ob_\V}{\normlq{\Hb_{\bar{c_j}}\Ob_\V}}}^2	\\
	&= \normlarge{\frac{ \Hb_{\bar{c_i}}\Ob_\V}{\normlq{\Hb_{\bar{c_i}}\Ob_\V}}}^2 + \normlarge{\frac{ \Hb_{\bar{c_j}}\Ob_\V}{\normlq{\Hb_{\bar{c_j}}\Ob_\V}}}^2 - 2\left\langle \frac{ \Hb_{\bar{c_i}}\Ob_\V}{\normlq{\Hb_{\bar{c_i}}\Ob_\V}}, \frac{ \Hb_{\bar{c_j}}\Ob_\V}{\normlq{\Hb_{\bar{c_j}}\Ob_\V}}\right\rangle \\
	&\numequ{i}\normlarge{\frac{ \Hb_{\bar{c_i}}\Ob_\V}{\normlq{\Hb_{\bar{c_i}}\Ob_\V}}}^2 + \normlarge{\frac{ \Hb_{\bar{c_j}}\Ob_\V}{\normlq{\Hb_{\bar{c_j}}\Ob_\V}}}^2 \\
	&\overset{(ii)}{\geqslant} C > 0, \numberthis \label{distance_2}
	\end{flalign*}
	where(i) follows from \Cref{lemma5.1}, where $\mathbf{H}$ is an orthogonal matrix so that $\langle \Hb_{\bar{c_i}}\Ob_\V,\Hb_{\bar{c_j}}\Ob_\V \rangle = 0$,
	and (ii) follows from \Cref{matrix_analysis_mutual_bound} so that $\normlarge{\frac{ \Hb_{\bar{c_i}}\Ob_\V}{\normlq{\Hb_{\bar{c_i}}\Ob_\V}}}^2 = \frac{ \norml{\Hb_{\bar{c_i}}\Ob_\V}^2}{\normlq{\Hb_{\bar{c_i}}\Ob_\V}^2} \geqslant C >0$.
	
	Following the similar proof procedure, we obtain
	\begin{flalign*}
	\norml{(\R_\U)_{\bar{i}} - (\R_\U)_{\bar{j}}}^2 &= 0 \quad \text{if} \quad c_i = c_j,\\
	\norml{(\R_\U)_{\bar{i}} - (\R_\U)_{\bar{j}}}^2 &\geqslant C > 0 \quad \text{if} \quad c_i \neq c_j. \numberthis \label{distance_3}
	\end{flalign*}
	Combining \cref{distance_1,distance_2,distance_3}, we have 
	\begin{flalign}
	\norml{\R_{\bar{i}} - \R_{\bar{j}}}^2 &= 0 \quad \text{if} \quad c_i = c_j,\\
	\norml{\R_{\bar{i}} - \R_{\bar{j}}}^2 &\geqslant C > 0 \quad \text{if} \quad c_i \neq c_j. \numberthis
	\end{flalign}
	Thus, if nodes in the same community, they share the same row in $R = [R_V,R_U]$, and if they are in different communities, their corresponding rows in $R$ are  sufficiently difference. Since there are $K$ communities, there are exactly $K$ different rows in $R$.
\end{proof}

\subsection{Proof of \Cref{Proposition_row_normalization_ration_matrix}} \label{proof_of_proposition_7}

In this section, we develop a bound on the difference between the ratio matrix $\hat{\Rb}$ generated by the singular vectors of  $\A$ and the ratio matrix $\Rb$ generated by the singular vectors of $\Omegab$, which is in parallel to \Cref{R gap} for D-SCORE.
\begin{Proposition} \label{Proposition_row_normalization_ration_matrix}
	For $\Rb = [\Rb_{\V},\Rb_{\U}], \hat{\R}=[\R_{\hat{\V}}, \R_{\hat{\U}}]$, and n large enough, with probability at least $1 - O(n^{-4})$,  we have
	\begin{flalign} \label{rown_distance_singular_vector}
	\norml{\hat{\R} - \R}_F^2 \leqslant    C T_n^2 \log (n) err_n .
	\end{flalign}
\end{Proposition}
\begin{proof}
	We define  the sets $\hat{S}_V$ and $\hat{S}_U$ as follows:
	\begin{flalign*}
	\hat{S}_V &= \left(
	1 \leqslant i \leqslant n;\  \abs { \frac{\normlq{\hat{\V}_{\bar{i}}}}{\normlq{(\V\Ob_\V)_{\bar{i}}}} - 1} \leqslant C_0, 0 < C_0 < 1 \right), \\
	\hat{S}_U &= \left(
	1 \leqslant i \leqslant n;\  \abs { \frac{\normlq{\hat{\U}_{\bar{i}}}}{\normlq{(\U\Ob_\U)_{\bar{i}}}} - 1} \leqslant C_0, 0 < C_0 < 1 \right)  \numberthis \label{define_rown_ill_behaviour}.
	\end{flalign*}	
	Then, we have the following bounds for these sets.	
	\begin{Lemma}     \label{rown_ill_behaved_node}
		For nodes in $ \hat{S}_V$ or $ \hat{S}_U$, the following inequalities hold
		\begin{flalign*}
		\norml{\hat{\V}_{\bar{i}}} \asymp \norml{(\V\Ob_\V)_{\bar{i}}} \asymp \frac {\deltab (i)}{\norm {\deltab^{(c_i)}}} \quad \text{for} \quad  i \in \hat{S}_V , \\
		\norml{\hat{\U}_{\bar{i}}} \asymp \norml{(\U\Ob_\U)_{\bar{i}}} \asymp \frac {\deltab (i)}{\norm {\deltab^{(c_i)}}} \quad \text{for} \quad  i \in \hat{S}_U . \numberthis \label{Well_behaviour_second_Part}
		\end{flalign*}
		For $n$ large enough, with probability at least $1 - O(n^{-4})$, the following inequalities hold
		\begin{flalign*}
		\abs{\mathcal{V} \backslash \hat{S}_V} \leqslant \frac{C \log (n) Z}{\norml{\thetab}^2 \delta_{\min}^2} \quad \text{ and } \quad
		\abs{\mathcal{V} \backslash \hat{S}_U} \leqslant \frac{C \log (n) Z}{\norml{\deltab}^2 \theta_{\min}^2}.    \numberthis \label{Well_behaviour_bound}
		\end{flalign*}
	\end{Lemma}
	\begin{proof}
		The proof can be found in  \Cref{proof_rown_ill_behaved_node}.
	\end{proof}
	
	We are now ready to prove the proposition.
	By \cref{expection_matrix_row_ratio} and \Cref{matrix_analysis_mutual_bound}, we have 
	\begin{flalign*}
	\norml{(\R_{\V})_{\bar{i}}} &= \normlarge{\frac{ (\V\Ob_\V)_{\bar{i}}}{\normlq{(\V\Ob_\V)_{\bar{i}}}} } \asymp \norml{\frac{ (\V\Ob_\V)_{\bar{i}}}{\norml{(\V\Ob_\V)_{\bar{i}}}} }  = 1. \numberthis \label{constant_asymptoic}
	\end{flalign*}
	
	Note that
	\begin{flalign*}
	\norml{\R - \R}_F^2 = \norml{\R_{\hat{\V}}  - \R_\V}_F^2 + \norml{ \R_{\hat{\U}} - \R_\U}_F^2.
	\end{flalign*}
	
	
	We first divide $\norml{\R_{\hat{\V}}  - \R_{\V}}_F^2$ into the following two parts:
	\begin{flalign*}
	\norml{\R_{\hat{\V}}  - \R_{\V}}_F^2 = \sum_{i \in (\mathcal{V} \backslash \hat{S}_V)} \norml{(\R_{\hat{\V}} )_{\bar{i}} - (\R_{\V})_{\bar{i}}}^2
	+\sum_{i \in \hat{S}_V} \norml{(\R_{\hat{\V}}  )_{\bar{i}} - (\R_{\V})_{\bar{i}}}^2 .
	\end{flalign*}
	For  the first term, i.e., $i \in (\mathcal{V} \backslash \hat{S}_V)$,
	\begin{flalign*}   \label{row_first_term}
	\sum_{i \in (\mathcal{V} \backslash \hat{S}_V)} \norml{(\R_{\hat{\V}} )_{\bar{i}} - (\R_{\V})_{\bar{i}}}^2    &\leqslant  C \sum_{i \in (\mathcal{V} \backslash \hat{S}_V)} (\norml{(\R_{\hat{\V}} )_{\bar{i}} }^2 + \norml{(\R_{\V})_{\bar{i}}}^2)\\
	&\numleqslant{i} C \sum_{i \in (\mathcal{V} \backslash \hat{S}_V)} ( KT_n^2 + C) \\
	&\leqslant C |\mathcal{V} \backslash \hat{S}_V| T_n^2 \\
	&\numleqslant{ii}  \frac{C T_n^2 \log (n) Z}{\norml{\thetab}^2 \delta_{\min}^2},  \numberthis
	\end{flalign*}
	where (i) follows from  \cref{row_normalization_ratio_matrix}, which shows us that the term is  truncated  by $T_n$, and \cref{constant_asymptoic}, and (ii) follows from \Cref{rown_ill_behaved_node}.  
	
	For the second term, i.e., $i \in \hat{S}_V$, we have

	\begin{flalign*} \label{row_second_term}
	\sum_{i \in \hat{S}_V} \norml{(\R_{\hat{\V}} )_{\bar{i}} - (\R_{\V})_{\bar{i}}}^2  &\numleqslant{i} \sum_{i \in \hat{S}_V} \normlarge{ \frac{\hat{\V}_{\bar{i}}}{ \normlq{\hat{\V}_{\bar{i}}} }  -   \frac{(\V\Ob_\V)_{\bar{i}}}{ \normlq{(\V\Ob_\V)_{\bar{i}}}  }}^2 \\
	& \numleqslant{ii} \sum_{i \in \hat{S}_V} \frac{ \norml{\hat{\V}_{\bar{i}} - (\V\Ob_\V)_{\bar{i}} }^2} {\min  \left( \normlq{\hat{\V}_{\bar{i}}}^2 , \normlq{(\V\Ob_\V)_{\bar{i}}}^2 \right) } \\
	& \numleqslant{iii} C  \frac{\norml{\deltab}^2}{ \delta_{min}^2} \sum_{i \in \hat{S}_V} \norml{\hat{\V}_{\bar{i}} - (\V\Ob_\V)_{\bar{i}} }^2 \\
	& \numleqslant{iv} \frac{C  \log (n) Z}{\norml{\thetab}^2 \delta_{\min}^2}, \numberthis
	\end{flalign*}
	where (i) follows from  \cref{constant_asymptoic}, which implies $\norml{(\R_{\V})_{\bar{i}}} = \normlarge{ \frac{(\V\Ob_\V)_{\bar{i}}}{ \normlq{(\V\Ob_\V)_{\bar{i}}}  }}  \leqslant C$,  and hence  $T_n \geqslant C \geqslant \norml{(\R_{\V})_{\bar{i}}} $ for  large $n$, so that, $\norml{(\R_{\hat{\V}} )_{\bar{i}} - (\R_{\V})_{\bar{i}}}^2 \leqslant  \normlarge{ \frac{\hat{\V}_{\bar{i}}}{ \normlq{\hat{\V}_{\bar{i}}} }  -   \frac{(\V\Ob_\V)_{\bar{i}}}{ \normlq{(\V\Ob_\V)_{\bar{i}}}  }}^2 $,
	(ii) follows from \Cref{inequality},
	(iii) follows from \Cref{rown_ill_behaved_node},
	and (iv) follows from \Cref{rown_distance_of_singular_vector}.
	
	Combining \cref{row_first_term,row_second_term},  we obtain $\norml{\R_{\hat{\V}}  - \R_\V}_F^2  \leqslant \frac{C T_n^2 \log (n) Z}{\norml{\thetab}^2 \delta_{\min}^2}$. Similarly, we obtain $\norml{ \R_{\hat{\U}} - \R_\U}_F^2  \leqslant \frac{C T_n^2 \log (n) Z}{\norml{\deltab}^2 \theta_{\min}^2}$. Therefore,  \Cref{Proposition_row_normalization_ration_matrix} follows by combining these two inequalities together.
\end{proof}

\subsection{\Cref{lemma_3_2_DSCOREq} and its Proof}
The following \Cref{lemma_3_2_DSCOREq} is in parallel to \Cref{M* R gap} for D-SCORE.
\begin{Proposition} \label{lemma_3_2_DSCOREq}
	For n large enough, with probability at least $1 - O(n^{-4})$, we have
	\begin{flalign*}
	\norml{\M^* - \R}_F^2 \leqslant C T_n^2 \log(n) err_n.
	\end{flalign*}
\end{Proposition}
\begin{proof}
	The proof follows in a similar manner to that for \Cref{M* R gap} for D-SCORE.
\end{proof}

\subsection{Proof of \Cref{rown_mis clustering nodes}}
\begin{proof}
	The proof follows in a similar manner to that for \Cref{mis clustering nodes} for D-SCORE. Note that the constant $C$ in this theorem can be chosen based on \Cref{Row_Normalization_K_different_row}, where $\norml{\R_{\bar{i}} - \R_{\bar{j}}}^2 \geqslant C > 0 $ if $c_i \neq c_j$.
\end{proof}

{\boldmath
	\subsection{Proof of Lemmas for D-SCORE$_q$}}
\subsubsection{Proof of \Cref{rown_ill_behaved_node}} \label{proof_rown_ill_behaved_node}
\begin{proof}
	By \cref{row represent U}, we obtain $	\V_{\bar{i}}=\frac {\deltab (i)}{\norml {\deltab^{(c_i)}}} \Hb_{\bar{c_i}}$, and $H$ is an orthogonal matrix. Thus,
	\begin{flalign}
	\norml{(\V\Ob_\V)_{\bar{i}}} = \norml{\V_{\bar{i}}\Ob_\V} = \norml{\V_{\bar{i}}} = \normlarge{\frac {\deltab (i)}{\norml {\deltab^{(c_i)}}} \Hb_{\bar{c_i}}} = \frac {\deltab (i)}{\norml {\deltab^{(c_i)}}}.  \label{row_value}
	\end{flalign}
	With  \cref{eq:4.4},  we have
	\begin{flalign}
	\norml{(\V\Ob_\V)_{\bar{i}}} = \frac {\deltab (i)}{\norml {\deltab^{(c_i)}}} \asymp \frac {\deltab (i)}{\norml {\deltab}} \label{row_norm}.
	\end{flalign}
	Combining \Cref{matrix_analysis_mutual_bound} with \cref{row_norm}, we have
	\begin{flalign} \label{equivalent_step_3}
	\normlq{(\V\Ob_\V)_{\bar{i}}}  \asymp \norml{(\V\Ob_\V)_{\bar{i}}} \asymp \frac {\deltab (i)}{\norml {\deltab}}.
	\end{flalign}
	By definition of $\hat{S}_V$ (\cref{define_rown_ill_behaviour}), for $i \in \hat{S}_V$, we have $1 - C_0 \leqslant \frac{\normlq{\hat{\V}_{\bar{i}}}}{\normlq{(\V\Ob_\V)_{\bar{i}}}} \leqslant 1+C_0$. Thus
	\begin{flalign} \label{qnorm_Equivalent_in_well_behaviour_node_set}
	\normlq{\hat{\V}_{\bar{i}}} \asymp \normlq{(\V\Ob_\V)_{\bar{i}}}.
	\end{flalign}
	Combining \cref{qnorm_Equivalent_in_well_behaviour_node_set,equivalent_step_3}, we conclude that for $i \in \hat{S}_V$
	\begin{flalign}
	\normlq{\hat{\V}_{\bar{i}}} \asymp  \frac {\deltab (i)}{\norml {\deltab}}.
	\end{flalign}
	Similarly, we obtain  $\norml{\hat{\U}_{\bar{i}}} \asymp \norml{(\U\Ob_\U)_{\bar{i}}} \asymp \frac {\thetab (i)}{\norm {\thetab^{(c_i)}}}$, for $i \in \hat{S}_U$, which completes the proof for \cref{Well_behaviour_second_Part}.
	
	To prove \cref{Well_behaviour_bound}, we first drive
	\begin{flalign*}
	\sum_{i \in (\mathcal{V} \backslash \hat{S}_V)}  \left(\frac{ \normlq{\hat{\V}_{\bar{i}}} } { \normlq{(\V\Ob_\V)_{\bar{i}}} } - 1\right) ^2
	&=  \sum_{i \in (\mathcal{V} \backslash \hat{S}_V)} \frac{1}{\normlq{ (\V\Ob_\V)_{\bar{i}}}^2} (\normlq{\hat{\V}_{\bar{i}}} - \normlq{ (\V\Ob_\V)_{\bar{i}} })^2\\
	&\numleqslant{i} \frac{\norml{\deltab}^2}{\delta_{\min}^2} \sum_{i \in (\mathcal{V} \backslash \hat{S}_V)} \norml{\hat{\V}_{\bar{i}} - (\V\Ob_\V)_{\bar{i}} }^2\\
	& \leqslant \frac{\norml{\deltab}^2}{\delta_{\min}^2} \sum_{i \in \mathcal{V} } \norml{\hat{\V}_{\bar{i}} - (\V\Ob_\V)_{\bar{i}} }^2 \\
	& =\frac{\norml{\deltab}^2}{\delta_{\min}^2} \norml{\hat{V} - \V\Ob_\V}_F^2\\
	& \numleqslant{ii} C  \frac{ \log (n) Z}{\norml{\thetab}^2  \delta_{\min} ^2},
	\end{flalign*}
	where (i) follows because $\abs{\norml{v} - \norml{u} } \leqslant \norml{v - u}$ and \cref{equivalent_step_3}, and (ii) follows from \Cref{rown_distance_of_singular_vector}.
	
	Thus, $\abs{\mathcal{V} \backslash \hat{S}_V} \leqslant C  \frac{ \log (n) Z}{\norml{\thetab}^2  \delta_{\min} ^2} $. Similar steps can show that $\abs{\mathcal{V} \backslash \hat{S}_U} \leqslant C  \frac{ \log (n) Z}{\norml{\deltab}^2  \theta_{\min} ^2}$.
\end{proof}

\subsubsection{Proof of \Cref{inequality}} \label{proof_inequality}
\begin{proof}
	Without loss of generality, we assume $\normlq{\x} \leqslant \normlq{\y}$, and   only need to show
	\begin{flalign*}
	\normlarge{\frac{\x}{\normlq{\x}} - \frac{\y}{\normlq{\y}}}^2 \leqslant  C\frac{\norml{\x -\y}^2}{ \normlq{\x}^2}.
	\end{flalign*}
	We derive
	\begin{flalign*}
	\normlarge{\frac{\x}{\normlq{\x}} - \frac{\y}{\normlq{\y}}}^2 &= \normlarge{\frac{\x}{\normlq{\x}} - \frac{\y}{\normlq{\x}} + \frac{\y}{\normlq{\x}} - \frac{\y}{\normlq{\y}}}^2 \\
	&\numleqslant{i} 2\left(\normlarge{\frac{\x}{\normlq{\x}} - \frac{\y}{\normlq{\x}} }^2 +  \normlarge{\frac{\y}{\normlq{\x}} - \frac{\y}{\normlq{\y}}}^2 \right) \\
	&\leqslant 2\frac{\norml{\x-\y}^2}{\normlq{\x}^2} + 2\norml{\y}^2 \abs{\frac{1}{\normlq{\x}} - \frac{1}{\normlq{\y}}}^2 \\
	&\leqslant 2\frac{\norml{\x-\y}^2}{\normlq{\x}^2} + 2\norml{\y}^2 \abs{ \frac{\normlq{\y}  - \normlq{\x}}{\normlq{\x}\normlq{\y}}}^2  \\
	&\numleqslant{ii} 2\frac{\norml{\x-\y}^2}{\normlq{\x}^2} + 2 \frac{\norml{\y}^2 }{\normlq{\x}^2\normlq{\y}^2}\normlq{\x - \y}^2 \\
	&\numleqslant{iii} 2\frac{\norml{\x-\y}^2}{\normlq{\x}^2} + C \frac{\normlq{\x - \y}^2 }{\normlq{\x}^2} \\
	&\numleqslant{iv} C \frac{\norml{\x-\y}^2}{\normlq{\x}^2},
	\end{flalign*}
	where (i) follows because $\norml{\x + \y}^2 \leqslant 2\big( \norml{\x}^2 + \norml{\y}^2 \big)$, and (ii) follows because  $\abs{\normlq{\x} - \normlq{\y}} \leqslant \normlq{\x - \y}$,  (iii) follows from \Cref{matrix_analysis_mutual_bound} such that $\norml{\y} \asymp \normlq{\y}$, and  (iv) follows from \Cref{matrix_analysis_mutual_bound}, which implies  $\norml{\x-\y} \asymp \normlq{\x-\y}$.
\end{proof}


\bibliography{D_SCORE}

\end{document}